\documentclass[a4paper,11pt]{article}
\usepackage{latexsym,amssymb,enumerate,amsmath,epsfig,amsthm,authblk,dsfont,bm}
\usepackage[margin=1in]{geometry}
\usepackage{setspace,color}
\usepackage[dvipsnames]{xcolor}
\usepackage{tikz}
\usepackage{floatrow}
\usepackage{graphicx,subfig}
\usepackage[ruled]{algorithm2e}
\usepackage{algorithmic}
\usepackage{epstopdf}
\usepackage{grffile}
\usepackage{caption}
\usepackage{multirow}
\usepackage{hyperref}
\hypersetup{
	colorlinks=true,
	linkcolor=blue,
	citecolor=blue,
	filecolor=magenta,
	urlcolor=blue,
}
\numberwithin{equation}{section}

\newcommand{\bp}{\mathbf{p}}
\newcommand{\be}{\mathbf{e}}
\newcommand{\bb}{\mathbf{b}}
\newcommand{\RR}{\mathds{R}}

\newcommand{\bx}{\mathbf{x}}
\newcommand{\bq}{\mathbf{q}}
\newcommand{\bn}{\mathbf{n}}

\newcommand{\bG}{\mathbf{G}}
\newcommand{\bM}{\mathbf{M}}
\newcommand{\bD}{\mathbf{D}}
\newcommand{\bH}{\mathbf{H}}
\newcommand{\bB}{\mathbf{B}}
\newcommand{\bF}{\mathbf{F}}
\newcommand{\cS}{\mathcal{S}}
\newcommand{\cI}{\mathcal{I}}
\newcommand{\cF}{\mathcal{F}}

\newcommand{\cH}{\mathcal{H}}
\newcommand{\cL}{\mathcal{L}}

\newcommand{\real}{\mathrm{Real}}

\newcommand{\diver}{\mathrm{div}}

\newcommand{\vshrink}{\vspace{0cm}}

\DeclareMathOperator*{\argmin}{arg\,min}
\newtheorem{theorem}{Theorem}[section]
\newtheorem{remark}{Remark}[section]
\begin{document}
\title{An Operator-Splitting Method for the Gaussian Curvature Regularization Model with Applications to Surface Smoothing and Imaging}
\date{\emph{In memory of Roland Glowinski--a dear friend, mentor, colleague and great leader.}}
\author{Hao Liu\thanks{Department of Mathematics, Hong Kong Baptist University, Kowloon Tong, Hong Kong. Email: haoliu@hkbu.edu.hk.}
	, Xue-Cheng Tai\thanks{Department of Mathematics, Hong Kong Baptist University, Kowloon Tong, Hong Kong. Email: xuechengtai@hkbu.edu.hk.}
	, Roland Glowinski\thanks{Department of Mathematics, University of Houston, Honston, TX 77204, USA, and Department of Mathematics, Hong Kong Baptist University, Kowloon Tong, Hong Kong. Email: roland@math.uh.edu.}
}
\maketitle

\begin{abstract}
	Gaussian curvature is an important geometric property of surfaces, which has been used broadly in mathematical modeling. Due to the full nonlinearity of the Gaussian curvature, efficient numerical methods for models based on it are uncommon in literature. In this article, we propose an operator-splitting method for a general Gaussian curvature model. In our method, we decouple the full nonlinearity of Gaussian curvature from differential operators by introducing two matrix- and vector-valued functions. The optimization problem is then converted into the search for the steady state solution of a time dependent PDE system. The above PDE system is well-suited to time discretization by operator splitting, the sub-problems encountered at each fractional step having either a closed form solution or being solvable by efficient algorithms. The proposed method is not sensitive to the choice of parameters, its efficiency and performances being demonstrated via systematic experiments on surface smoothing and image denoising.
\end{abstract}

\section{Introduction}
Gaussian curvature is a most important geometric property finding applications in many scientific areas, such as biology \cite{dharmavaram2019gaussian, bade2018gaussian, bozelli2018membrane}, physics \cite{germani2020nonlinear}, graph regularization \cite{elghawalby2009graph}, image processing and surface fairing \cite{zhao2006triangular}. For example: (i) Gaussian curvature is used in \cite{dharmavaram2019gaussian} to explain the budding process of enveloped viruses.
(ii) One studies in \cite{germani2020nonlinear} primordial black holes from Gaussian curvature perturbations. (iii) One uses in \cite{elghawalby2009graph} Gaussian curvature to regularize triangulated graphs. (iv) Gaussian curvature based models have also been proposed for image regularization \cite{gong2017curvature,zhong2020image} and surface fairing \cite{elsey2009analogue, brito2013fast}.

Consider a two-dimensional surface $S$. The Gaussian curvature of $S$ at $\bx$ is the product of its principal curvatures at $\bx$ \cite{do2016differential}. The Gaussian curvature is an intrinsic quantity since it does not depend on how $S$ is embedded in the space. Another desirable property of Gaussian curvature is its relation to the developability of $S$. A surface with zero Gaussian curvature can be isometrically mapped onto a plane without distortion; it is then called developable. Many simple surfaces are developable, such as cylinders and cones. The property of being an intrinsic quantity and the relation to developability of surfaces make Gaussian curvature a natural regularizer which has been used widely in mathematical modeling \cite{elsey2009analogue,brito2013fast,gong2017curvature}.

Despite of the rich applications of Gaussian curvature, only few publications dedicated to numerical methods for Gaussian curvature models can be found in the literature. Gaussian curvature driven flows for image smoothing are proposed in \cite{lee2005noise,lu2011high} in which the flow PDE is numerically solved by the forward Euler scheme. Another Gaussian curvature flow is proposed in \cite{firsov2006domain}. The evolution PDE is solved by a Crank-Nicholson scheme together with a domain decomposition technique. In \cite{gong2013local}, one proposes a weighted Gaussian curvature model in which the weight of the Gaussian curvature term is designed such that the model simplifies to a quadratic form leading to an explicit formula for the problem solution.
Recently, the authors of \cite{gong2021discrete} proposed a robust discrete scheme to compute the weighted Gaussian curvature. The authors of \cite{elsey2009analogue} propose a Gaussian curvature based model for surface fairing in which the surface is represented by a triangulation. The proposed model is discretized using a dedicated scheme introduced in \cite{banchoff1997tight} and optimized by gradient descent. The augmented Lagrangian method (ALM) has demonstrated superior performance in image processing \cite{yashtini2019efficient, chan2016plug, yuan2014spatially, lanza2016convex} and has been applied to optimize Gaussian curvature based models for image denoising \cite{brito2016image,ren2018optimization}, image registration \cite{begum2021two,ibrahim2014novel}, and image inpainting \cite{zhong2020image}. Although the ALM may converge very quickly, its performances are sensitive to the choice of the augmentation parameters. Indeed, finding the optimal parameters is tricky and maybe time consuming. A two-step method has been applied to optimize Gaussian curvature based models for image denoising \cite{brito2016image} and surface fairing \cite{brito2013fast}. In each step of the two-step method, the authors solve an optimization problem using gradient descent. As shown in \cite{brito2016image}, the two-stage method is less efficient than the ALM. Numerical algorithms for other curvature based models are developed for the total curvature \cite{zhong2020minimizing}, mean curvature \cite{zhu2012image,ma2018image} and Euler's elastica model \cite{yashtini2016fast,zhu2013image,duan2014two,el2010fast}.

Actually, the ALM is a special operator-splitting method which has a long history on providing efficient numerical solvers for various PDE related problems \cite{burger2016first, glowinski2019fast,glowinski2017splitting}. When solving a complicated time-dependent PDE by an operator-splitting method, one decomposes the PDE into several sub-PDEs such that each sub-PDE problem can be solved efficiently. For each time step, instead of solving the original PDE, one solves these sub-PDEs sequentially \cite{mclachlan2002splitting,liu2019finite}. When applying the above splitting strategy on optimization problems, one first derives the related Euler-Lagrange equation and associates with it an initial value problem such that the steady state solution of this initial value problem solves the optimization problem. Then the initial value problem is solved by the splitting strategy mentioned above. Such a kind of operator-splitting method  has been proposed for image regularization \cite{deng2019new,liu2020color}, and surface reconstruction \cite{he2020curvature,he2019fast}. The performances of these methods have a low sensitivity to the choice of parameters. In \cite{deng2019new}, the authors focus on Euler's elastica model for image denoising. The proposed operator-splitting method is more efficient than the ALM proposed in \cite{tai2011fast}. The authors of \cite{he2020curvature} proposed an operator-splitting method and an ALM to reconstruct a surface from a point cloud. The numerical experiments reported in \cite{he2020curvature} show that the operator-splitting method is more robust than the ALM.

In this article, we propose an operator-splitting method for a two-dimensional Gaussian curvature based model. We consider a general model consisting of a fidelity term and of two regularization terms: a Gaussian curvature term and a total variation one. To decouple the nonlinearities of the model, two matrix- and vector-valued functions are introduced with some constraints. To derive the optimality condition of the new problem, the constraints are enforced by utilizing indicator functionals. We then associate with the optimality conditions an initial-value problem, a time-dependent PDE system, which is time discretized  by an operator-splitting method. In our splitting scheme, each sub-problem has either a closed-form solution or can be solved efficiently. The efficiency of the proposed method is demonstrated on surface smoothing and image denoising examples. Our method can optimize the Gaussian curvature based model efficiently and is not sensitive to the choice of parameters.

The remaining of this paper is organized as follows: We introduce the Gaussian curvature based model in Section \ref{sec.model}. The proposed operator-splitting method and solvers for each sub-problem are presented in Section \ref{sec.splitting}. The proposed method is space discretized in Section \ref{sec.discretization}. In Section \ref{sec.experiments}, we present the results of numerical experiments, where the method we propose is applied to surface smoothing and image denoising problems. We conclude this article in Section \ref{sec.conclusion}.

\section{The Gaussian curvature model}\label{sec.model}
Let $\Omega\subset \RR^2$ be a rectangular domain and $f$ be a noisy function of two variables. The function $f$ is not a surface, but its graph is one. In image processing, one can take $f$ as a noisy image whose function values are pixel values.
%In image processing, one can take $f$ as the surface of the noisy image.
We consider regularizing $f$ by the following Gaussian curvature-TV model
\begin{align}
	\min_{v\in \cH^2(\Omega)} \left[\int_{\Omega} \frac{|\det\bD^2 v|}{(1+|\nabla v|^2)^2}ds+\alpha\int_{\Omega}|\nabla v|d\bx +\frac{1}{2\beta}\int_{\Omega}|f-v|^2d\bx\right],
	\label{eq.model0}
\end{align}
where $\cH^2(\Omega)$ is the Sobolev space defined by
\begin{align*}
	&\cH^2(\Omega)=\left\{ v|v\in \cL^2(\Omega),\ \nabla v\in (\cL^2(\Omega))^2,\ \bD^2v\in(\cL^2(\Omega))^{2\times 2}\right\},\\
	&\mbox{with }\cL^2(\Omega)=\left\{v|\int_{\Omega} v^2 d\bx< +\infty\right\},
\end{align*}
the derivatives being in the sense of distributions.
%, defined by:
%\begin{align*}
%	\cH^2(\Omega)=&\big\{v\in \cH^2(\Omega): v(0,:)=v(L_1,:), v(:,0)=v(:,L_2), \\
%	&\quad \nabla v(0,:)= \nabla v(L_1,:), \nabla v(:,0)= \nabla v(:,L_2)\big\}.
%\end{align*}
Above, $d\bx=dx_1dx_2$, $s$ denotes the surface area, $\alpha\geq0, \beta>0$ are weighting parameters balancing these terms, and $\bD^2 v$ is the Hessian of $v$ given by
\begin{align}
	\bD^2 v=\begin{pmatrix}
		\frac{\partial^2 v}{\partial x_1^2} & \frac{\partial^2 v}{\partial x_1\partial x_2}\\
		\frac{\partial^2 v}{\partial x_1\partial x_2} & \frac{\partial^2 v}{\partial x_2^2}
	\end{pmatrix}.
\end{align}
In (\ref{eq.model0}), the first two terms are regularization terms: the Gaussian curvature of $v$ \cite{gauss2005general} and the total variation of $v$, respectively. Since the Gaussian curvature is an intrinsic geometric quantity of a surface, we integrate it with respect to the surface area. The third term in (\ref{eq.model0}) is a fidelity term.

Substituting $ds$ by $\sqrt{1+|\nabla v|^2}d\bx$, we get
\begin{align}
	\min_{v\in \cH^2(\Omega)} \left[\int_{\Omega} \frac{|\det\bD^2 v|}{(1+|\nabla v|^2)^{3/2}}d\bx+\alpha\int_{\Omega}|\nabla v|d\bx +\frac{1}{2\beta}\int_{\Omega}|f-v|^2d\bx\right].
	\label{eq.model}
\end{align}
The full nonlinearity and the non-smoothness of the  Gaussian curvature term make solving (\ref{eq.model}) a challenging problem. To overcome this difficulty, we introduce two matrix- and vector-valued functions to decouple the nonlinearities from the differential operators.

%Define
%\begin{align*}
%	&\cH_P^1(\Omega)=\{f\in \cH^1(\Omega): v(0,:)=v(L_1,:), v(:,0)=v(:,L_2),\\
%	&\cL_P^{\infty}(\Omega)=\{v\in \cL^{\infty}(\Omega): v(0,:)=v(L_1,:), v(:,0)=v(:,L_2).
%\end{align*}
Let
$$
\bq=\begin{bmatrix}
	q_1 \\ q_2
\end{bmatrix}\in (\cH^1(\Omega))^2, \mbox{ and }
\bG=\begin{pmatrix}
	G_{11} & G_{12}\\ G_{21} & G_{22}
\end{pmatrix}\in (\cL^{2}(\Omega))^{2\times 2}.
$$
If $u$ is a solution to (\ref{eq.model}), then  $(u,\bp,\bH)$ solves
\begin{equation}
	\begin{cases}
		\min\limits_{\substack{v\in\cH^1(\Omega),\ \bq\in (\cH^1(\Omega))^2, \\ \bG\in (\cL^{2}(\Omega))^{2\times 2}}} \left[\displaystyle\int_{\Omega} \frac{|\det\bG|}{(1+|\bq|^2)^{3/2}}d\bx+\alpha\displaystyle\int_{\Omega}|\bq|d\bx +\frac{1}{2\beta}\displaystyle\int_{\Omega}|f-v|^2d\bx\right],\\
		\bq=\nabla v,\\
		\bG=\nabla \bq
	\end{cases}
	\label{eq.model.1}
\end{equation}
with $\bp=\nabla u,\bH=\nabla\bp$. Therefore (\ref{eq.model}) is converted to the constrained optimization problem (\ref{eq.model.1}). Next, we relax the constraints by utilizing indicator functionals.

Define the sets
\begin{align*}
	&\Sigma=\left\{\bq| \bq\in (\cL^2(\Omega))^2, \exists v\in \cH^1(\Omega) \mbox{ such that } \bq=\nabla v \mbox{ and } \int_{\Omega} (f-v)d\bx=0\right\},\\
	&S=\left\{ (\bq,\bG)| (\bq,\bG)\in \left(\cH^1(\Omega)\right)^2\times \left(\cL^2(\Omega)\right)^{2\times 2} \mbox{ such that } \bG=\nabla \bq\right\},
\end{align*}
and their indicator functionals
\begin{align*}
	I_{\Sigma}(\bq)=\begin{cases}
		0 & \mbox{ if } \bq\in \Sigma,\\
		+\infty & \mbox{ otherwise},
	\end{cases}\quad
	I_{S}(\bq,\bG)=\begin{cases}
		0 &\mbox{ if } (\bq,\bG)\in S,\\
		+\infty &\mbox{ otherwise}.
	\end{cases}
\end{align*}

We have that $(\bp,\bH)$ is the solution to
\begin{align}
	\min\limits_{\substack{\bq\in (\cH^1(\Omega))^2, \\ \bG\in (\cL^{2}(\Omega))^{2\times 2}}} \left[\displaystyle\int_{\Omega} \frac{|\det\bG|}{(1+|\bq|^2)^{3/2}}d\bx+\alpha\displaystyle\int_{\Omega}|\bq|d\bx +\frac{1}{2\beta}\displaystyle\int_{\Omega}|f-v_{\bq}|^2d\bx + I_{\Sigma}(\bq) +I_{S}(\bq,\bG)\right]
	\label{eq.model.2}
\end{align}
where  $v_{\bq}$ is the unique solution to
\begin{align}
	\begin{cases}
		\nabla^2v_{\bq}=\nabla\cdot \bq \mbox{ in } \Omega,\\
		(\nabla v_{\bq}-\bq)\cdot \bn=0 \mbox{ on } \partial\Omega,\\
		\displaystyle\int_{\Omega}(f-v_{\bq})d\bx=0.
	\end{cases}
	\label{eq.vq}
\end{align}
In (\ref{eq.vq}), {$\nabla^2$ represents the Laplace operator, }$\bn$ is the unit outward normal vector at the boundary. Compared to (\ref{eq.model.1}), in (\ref{eq.model.2}) one relaxes the constraints by introducing the two indicator functionals $I_{\Sigma}$ and $I_{S}$. Taking advantage of (\ref{eq.vq}), one can uniquely determine $v$ in (\ref{eq.model.1}) using $\bq$. Therefore the triple $(u,\bp,\bH)$ in (\ref{eq.model.1}) is reduced to $(\bp,\bH)$ in (\ref{eq.model.2}), which is an unconstrained optimization problem.
{
	\begin{remark}
		For any given $\bq$, problem (\ref{eq.vq}) is a standard Poisson--Neumann problem. On rectangular domains, there are many efficient solvers for problem (\ref{eq.vq}), such as sparse Cholesky, conjugate gradient,
		cyclic reduction, etc. In particular, when replacing in (\ref{eq.vq}) the Neumann boundary conditions by periodic ones, (\ref{eq.vq}) can be solved efficiently by FFT, see Section \ref{sec.periodic} and \ref{sec.p.4} for details.
	\end{remark}
}

\section{An operator splitting method to solve problem (\ref{eq.model.2})} \label{sec.splitting}
Operator-splitting methods solve complicated problems by solving a sequence of simpler sub-problems. They have been successfully used for the numerical solutions of PDEs \cite{glowinski2019finite}, inverse problems \cite{glowinski2015penalization}, fluid-structure interactions \cite{bukavc2013fluid} and problems in image processing \cite{deng2019new,liu2020color,glowinski2019fast}. We refer the readers to \cite{glowinski2017splitting} for a detailed discussion of operator-splitting methods.
In this section, we propose an operator splitting method to find the minimizers of (\ref{eq.model.2}).
\subsection{The optimality condition associated with (\ref{eq.model.2})}
The functional in (\ref{eq.model.2}) can be written as
\begin{align*}
	J_1+J_2+J_3
\end{align*}
where
\begin{align}
	&J_1(\bq,\bG)=\int_{\Omega} \frac{|\det\bG|}{(1+|\bq|^2)^{3/2}}d\bx,\\
	&J_2(\bq)=\alpha\int_{\Omega} |\bq|d\bx,\\
	&J_3(\bq)=\frac{1}{2\beta}\int_{\Omega} |f-v_{\bq}|^2d\bx.
\end{align}
The Euler-Lagrange equation for (\ref{eq.model.2}) reads as
\begin{align}
	\begin{cases}
		D_{\bq} J_1(\bp,\bH)+\partial_{\bq} J_2(\bp) + D_{\bq} J_3(\bp) + \partial_{\bq} I_S(\bp,\bH) +\partial_{\bq}I_{\Sigma}(\bp)\ni \mathbf{0},\\
		\partial_{\bG} J_1(\bp,\bH) + \partial_{\bG} I_S(\bp,\bH) \ni \mathbf{0},
	\end{cases}
	\label{eq.optimality}
\end{align}
where $D_{\bq}$ (resp. $\partial_{\bq}$) denotes the partial derivative (resp. subdifferential) of a differentiable functional (resp. a non-smooth functional) with respect to $\bq$. Operator $\partial_{\bG}$ is defined similarly.

With the optimality system (\ref{eq.optimality}), we associate the following initial value problem (dynamical flow):
\begin{align}
	\begin{cases}
		\gamma\frac{\partial\bp}{\partial t}+D_{\bq} J_1(\bp,\bH)+\partial_{\bq} J_2(\bp) + D_{\bq} J_3(\bp) + \partial_{\bq} I_S(\bp,\bH) +\partial_{\bq}I_{\Sigma}(\bp)\ni \mathbf{0},\\
		\frac{\partial\bH}{\partial t}+\partial_{\bG} J_1(\bp,\bH) + \partial_{\bG} I_S(\bp,\bH) \ni \mathbf{0},\\
		\bp(0)=\bp_0,\ \bH(0)=\bH_0
	\end{cases}
	\label{eq.dynamicflow}
\end{align}
with $\gamma>0$.
In (\ref{eq.dynamicflow}), $(\bp_0,\bH_0)$ is the initial condition of the flow. The choice of $(\bp_0,\bH_0)$ will be discussed in Section \ref{sec.initial}. Note that the steady state solution of (\ref{eq.dynamicflow}) solves (\ref{eq.optimality}). In the next subsection, we propose an operator-splitting method to time-discretize (\ref{eq.dynamicflow}) and to compute the steady state solution.

\subsection{An operator-splitting method for the dynamical-flow system}
We use the Lie scheme (see \cite{glowinski2017splitting,glowinski2016some} and the references therein) to time-discretize (\ref{eq.dynamicflow}). Denote by $\tau(>0)$ a time discretization step and by $n$ the step number. Let $t^n=n\tau$. We use $(\bp^n,\bG^n)$ to denote an approximate solution at time $t^n$. Given an initial condition $(\bp_0,\bH_0)$, we update $(\bp^n,\bH^n)$ via the following four steps:\\
\emph{\underline{Initialization}}:
\begin{align}
	(\bp^0,\bH^0)=(\bp_0,\bH_0).
	\label{eq.split.1}
\end{align}
\emph{\underline{Fractional Step 1}}:
\begin{align}
	\begin{cases}
		\begin{cases}
			\gamma\frac{\partial \bp}{\partial t}+D_{\bq}J_1(\bp,\bH)= \mathbf{0},\\
			\frac{\partial \bH}{\partial t}+\partial_{\bG}J_1(\bp,\bH)\ni \mathbf{0},
		\end{cases} \mbox{ in } \Omega\times(t^n,t^{n+1}),\\
		(\bp(t^n),\bH(t^n))=(\bp^n,\bH^n),
	\end{cases}
	\label{eq.split.2}
\end{align}
and set
\begin{align}
	(\bp^{n+1/4},\bH^{n+1/4})=(\bp(t^{n+1}),\bH(t^{n+1})).
	\label{eq.split.3}
\end{align}
\emph{\underline{Fractional Step 2}}:
\begin{align}
	\begin{cases}
		\begin{cases}
			\gamma\frac{\partial \bp}{\partial t}+\partial_{\bq}J_2(\bp)\ni \mathbf{0},\\
			\frac{\partial \bH}{\partial t}= \mathbf{0},
		\end{cases} \mbox{ in } \Omega\times(t^n,t^{n+1}),\\
		(\bp(t^n),\bH(t^n))=(\bp^{n+1/4},\bH^{n+1/4}),
	\end{cases}
	\label{eq.split.4}
\end{align}
and set
\begin{align}
	(\bp^{n+2/4},\bH^{n+2/4})=(\bp(t^{n+1}),\bH(t^{n+1})).
	\label{eq.split.5}
\end{align}
\emph{\underline{Fractional Step 3}}:
\begin{align}
	\begin{cases}
		\begin{cases}
			\gamma\frac{\partial \bp}{\partial t}+\partial_{\bq}I_S(\bp,\bH)\ni \mathbf{0},\\
			\frac{\partial \bH}{\partial t}+\partial_{\bG}I_S(\bp,\bH)\ni \mathbf{0},
		\end{cases} \mbox{ in } \Omega\times(t^n,t^{n+1}),\\
		(\bp(t^n),\bH(t^n))=(\bp^{n+2/4},\bH^{n+2/4}),
	\end{cases}
	\label{eq.split.6}
\end{align}
and set
\begin{align}
	(\bp^{n+3/4},\bH^{n+3/4})=(\bp(t^{n+1}),\bH(t^{n+1})).
	\label{eq.split.7}
\end{align}
\emph{\underline{Fractional Step 4}}:
\begin{align}
	\begin{cases}
		\begin{cases}
			\gamma\frac{\partial \bp}{\partial t}+D_{\bq} J_3(\bp)+ \partial_{\bq}I_{\Sigma}(\bp)= \mathbf{0},\\
			\frac{\partial \bH}{\partial t}= \mathbf{0},
		\end{cases} \mbox{ in } \Omega\times(t^n,t^{n+1}),\\
		(\bp(t^n),\bH(t^n))=(\bp^{n+3/4},\bH^{n+3/4}),
	\end{cases}
	\label{eq.split.8}
\end{align}
and set
\begin{align}
	(\bp^{n+1},\bH^{n+1})=(\bp(t^{n+1}),\bH(t^{n+1})).
	\label{eq.split.9}
\end{align}
In scheme (\ref{eq.split.1})-(\ref{eq.split.9}), the positive constant $\gamma>0$ controls the evolution speed of $\bp$. Scheme (\ref{eq.split.1})-(\ref{eq.split.9}) is only semidiscrete. One still needs to solve the subproblems (\ref{eq.split.2}), (\ref{eq.split.4}), (\ref{eq.split.6}) and (\ref{eq.split.8}). Here we advocate a Marchuk-Yanenko type scheme (see \cite{glowinski2017splitting} for more information on the Marchuk-Yanenko scheme) to time-discretize (\ref{eq.split.1})-(\ref{eq.split.9}), that is:\\
Set
\begin{align}
	(\bp^0,\bH^0)=(\bp_0,\bH_0).
	\label{eq.split.1.dis}
\end{align}
For $n\geq0$, $(\bp^n,\bH^n)\rightarrow (\bp^{n+1/4},\bH^{n+1/4}) \rightarrow (\bp^{n+2/4},\bH^{n+2/4}) \rightarrow (\bp^{n+3/4},\bH^{n+3/4}) \rightarrow (\bp^{n+1},\bH^{n+1})$ as follows:
\begin{align}
	&\begin{cases}
		\gamma\frac{\bp^{n+1/4}-\bp^n}{\tau}+D_{\bq}J_1(\bp^{n+1/4},\bH^n)= \mathbf{0},\\
		\frac{\bH^{n+1/4}-\bH^{n}}{\tau}+\partial_{\bG}J_1(\bp^{n+1/4},\bH^{n+1/4})\ni \mathbf{0},
	\end{cases} \label{eq.split.2.dis}\\
	&\begin{cases}
		\gamma\frac{\bp^{n+2/4}-\bp^{n+1/4}}{\tau}+\partial_{\bq}J_2(\bp^{n+2/4})\ni \mathbf{0},\\
		\frac{\bH^{n+2/4}-\bH^{n+1/4}}{\tau}= \mathbf{0},
	\end{cases}
	\label{eq.split.3.dis}\\
	& \begin{cases}
		\gamma\frac{\bp^{n+3/4}-\bp^{n+2/4}}{\tau}+\partial_{\bq}I_S(\bp^{n+3/4},\bH^{n+3/4})\ni \mathbf{0},\\
		\frac{\bH^{n+3/4}-\bH^{n+2/4}}{\tau}+\partial_{\bG}I_S(\bp^{n+3/4},\bH^{n+3/4})\ni \mathbf{0},
	\end{cases}
	\label{eq.split.4.dis}\\
	&\begin{cases}
		\gamma\frac{\bp^{n+1}-\bp^{n+3/4}}{\tau}+D_{\bq} J_3(\bp^{n+1})+ \partial_{\bq}I_{\Sigma}(\bp^{n+1})\ni \mathbf{0},\\
		\frac{\bH^{n+1}-\bH^{n+3/4}}{\tau}= \mathbf{0}.
	\end{cases}
	\label{eq.split.5.dis}
\end{align}
{
	Problem  (\ref{eq.split.2.dis}) is a time-discrete variant of (\ref{eq.split.2}). Given $\{\bp^n,\bH^n\}$, it is difficult to solve (\ref{eq.split.2}) for $\{\bp^n,\bH^n\}$ directly by an implicit scheme. Therefore, we split this complicated problem into two substeps in (\ref{eq.split.2}) by decoupling variables $\bp$ and $\bH$. Problem  (\ref{eq.split.2.dis}) consists of two substeps: In the first substep, we fix $\bH=\bH^n$ and compute for $\bp^{n+1}$ implicitly. In the second substep, we fix $\bp=\bp^{n+1/4}$ and update $\bH^{n+1/4}$ implicitly. Details on each substep can be found in Section \ref{sec.step2}. Such a splitting strategy is known as the Marchuk-Yanenko type scheme.  The convergence of this scheme is verified by our numerical experiments in Section \ref{sec.experiments}.
}
In the remaining part of this section, we discuss solutions to subproblems (\ref{eq.split.2.dis})-(\ref{eq.split.5.dis}).

{
	\begin{remark}
		Our operator--splitting method is an approximation of the gradient flow of the functional (\ref{eq.model}). The convergence of the proposed method closely relates to that of the gradient flow. When there is only one variable and the operator in each subproblem is smooth enough, the approximation error is of $O(\tau)$ (see \cite{chorin1978product} and \cite[Chapter 6]{glowinski2003finite}). In our problem, since $J_1,J_2, I_S$ and $I_{\Sigma}$ are not smooth, the approximation error of the proposed method requires a separate study. Due to the non-convexity of the functional in (\ref{eq.model}), all we can expect is that the gradient flow and the proposed method converge to a local minimizer. 
	\end{remark}
}
\subsection{On the solution to (\ref{eq.split.2.dis})}\label{sec.step2}
\subsubsection{Computing $\bp^{n+1/4}$}\label{sec.p2}
In (\ref{eq.split.2.dis}), $\bp^{n+1/4}$ solves the following minimization problem
\begin{align}
	\bp^{n+1/4}=\argmin_{\bq\in (\cL^{2}(\Omega))^2} \left[\frac{\gamma}{2}\int_{\Omega} |\bq-\bp^n|^2 d\bx +\tau\int_{\Omega} \frac{|\det \bH^n|}{(1+|\bq|^2)^{3/2}}d\bx\right].
	\label{eq.p.1}
\end{align}
By differentiating the functional in (\ref{eq.p.1}), $\bp^{n+1/4}=\left[p_1^{n+1/4}, p_2^{n+1/4}\right]^{\top}$ satisfies
\begin{align}
	\gamma \bp^{n+1/4}-3\tau|\Delta_1|\frac{\bp^{n+1/4}}{(1+|\bp^{n+1/4}|^2)^{5/2}}=\gamma \bb,
	\label{eq.p.1.optimality}
\end{align}
where $\Delta_1=\det \bH^{n+1/4}$, $\bb=\bp^n$. System (\ref{eq.p.1.optimality}) can be solved by Newton's method or a fixed point method. 

{
	We first discuss Newton's method. Define 
	\begin{align}
		\bF(\bq)=\gamma \bq-3\tau|\Delta_1|\frac{\bq}{(1+|\bq|^2)^{5/2}}-\gamma\bb.
	\end{align}
	for $\bq=[q_1,q_2]^{\top}$. It is easy to derive that
	\begin{align}
		D\bF(\bq)=\begin{pmatrix}
			\frac{\partial F_1}{\partial q_1} & \frac{\partial F_1}{\partial q_2}\\
			\frac{\partial F_2}{\partial q_1} & \frac{\partial F_2}{\partial q_2}
		\end{pmatrix}= \gamma I+\frac{3\tau|\Delta_1|}{(1+|\bq|^2)^{7/2}}
		\begin{pmatrix}
			4q_1^2-q_2^2-1 & 5q_1q_2\\
			5q_1q_2 & 4q_2^2-q_1^2-1
		\end{pmatrix},
	\end{align}
	where $I$ denotes the identity matrix. The Newtons method is conduced as follows:\\
	Set $\bq^0=\bb$. For $k>0$, we update $\bq^k\rightarrow \bq^{k+1}$ as
	\begin{align}
		\bq^{k+1}=\bq-\rho(D\bF(\bq^k))^{-1}F(\bq^k),
		\label{eq.p1.newton}
	\end{align}
	where $\rho\in (0,1]$ is a parameter controlling the updating rate of $\bq$. We update $\bq^k$ until $\|\bq^{k+1}-\bq^{k}\|_{\infty}\leq \xi_1$ for some small $\xi_1$. Denote the converged quantity by $\bp^*$. We set
	\begin{align*}
		\bp^{n+1/4}=\bp^*.
	\end{align*}
}

The formulation of the fixed point method is simpler. First observe that (\ref{eq.p.1.optimality}) can be rewritten as
\begin{align}
	\left(\gamma - \frac{3\tau|\Delta_1|}{(1+|\bp^{n+1/4}|^2)^{5/2}}\right)\bp^{n+1/4}=\gamma \bb.
	\label{eq.p.1.optimality.1}
\end{align}
Set $\bq^0=\bb$. For $k>0$, we update $\bq^k\rightarrow\bq^{k+1}$ as
\begin{align}
	&s^k=\left(\gamma - \frac{3\tau|\Delta_1|}{(1+|\bq^k|^2)^{5/2}}\right),\label{eq.p1.fix1}\\
	&\widetilde{\bq}=\gamma\bb/s^k,\\
	&\bq^{k+1}=(1-\rho_1)\bq^k+\rho_1\widetilde{\bq},\label{eq.p1.fix3}
\end{align}
where $\rho_1\in(0,1]$ is a parameter controlling the updating rate of $\bq$. By using the above algorithm, $\bq^k$ is updated until $\|\bq^{k+1}-\bq^{k}\|_{\infty}\leq \xi_1$ for some small $\xi_1$. Denote the converged quantity by $\bp^*$. We set
\begin{align*}
	\bp^{n+1/4}=\bp^*.
\end{align*}
In our experiments, the fixed point method is more stable and has a faster convergence compared to Newton's method. In all of our experiments reported in this paper, the fixed point method is used.
{
	
	In Newton's method (\ref{eq.p1.newton}) and the fixed point iteration (\ref{eq.p1.fix1})--(\ref{eq.p1.fix3}), initial guess $\bq^0=\bp^n$ is used. In problems (\ref{eq.split.2.dis})-(\ref{eq.split.5.dis}), $\tau$ is the artificial time step which controls the evolution speed of $\bp$ and $\bH$. As long as $\tau$ is small enough, $\{\bp^n,\bH^n\}$ are close to $\{\bp^{n-1},\bH^{n-1}\}$ and $\{\bp^{n-1+1/4},\bH^{n-1+1/4}\}$, where $\bp^{n-1+1/4}$ is the minimizer of (\ref{eq.p.1}) in the previous outer iteration. In addition, the functional in (\ref{eq.p.1}) in the current outer iteration does not change too much from that in the previous outer iteration. It is reasonable to expect the minimizer of (\ref{eq.p.1}) at the current outer iteration is close to $\bp^{n}$ or $\bp^{n-1+1/4}$. Therefore $\bp^n$ is a good initial guess and should converge to the minimizer fast. This is verified in our numerical experiments.
	
	The operator splitting method we used is the Marchuk-Yanenko variant of the Lie scheme. Unlike ADMM type splitting methods, the Lie and Marchuk-Yanenko schemes `enjoy' a systematic splitting error (of order $\tau$, at best typically). In order to have an accurate method one has to use small values of $\tau$, implying many time steps before reaching a steady state solution. This drawback becomes an advantage when using Newton's method initialized with solution at time step $n$ to compute solution at time step $n + 1$, since the small value of $\tau$ one uses implies that both solutions are close to each other, which helps for Newton's method convergence. We expect the same for the fixed point method.
}
\subsubsection{Computing $\bH^{n+1/4}$}\label{sec.H2}
Function $\bH^{n+1/4}$ is the solution to
\begin{align}
	\bH^{n+1/4}=\argmin_{\bG\in (\cL^{2}(\Omega))^{2\times 2}} \left[\frac{1}{2} \int_{\Omega} |\bG-\bH^{n}|^2d\bx +\tau \int_{\Omega} \frac{|\det \bG|}{(1+|\bp^{n+1/4}|^2)^{3/2}}d\bx\right].
	\label{eq.H1}
\end{align}
Problem (\ref{eq.H1}) is of the form
\begin{align}
	\bM=\argmin_{\bG\in (\cL^{2}(\Omega))^{2\times 2}} \left[\frac{1}{2} \int_{\Omega} |\bG-\bB|^2d\bx +\tau \int_{\Omega} \Delta_2|G_{11}G_{22}-G_{12}G_{21}|d\bx\right]
	\label{eq.H1.1}
\end{align}
with $\bB=\bH^n, \Delta_2=(1+|\bp^{n+1/4}|^2)^{-3/2}$. By grouping $(G_{11},G_{12})$ and $(G_{21},G_{22})$, we use the following relaxation method to solve for $\bM$:\\
Set $\bM^0=\bB$, and fix $\rho_2\in(0,1]$. For $k>0$, we update $\bM^k\rightarrow \bM^{k+1}$ in the following two steps:\\
\textbf{Step 1}: Solve
\begin{align}
	(\widetilde{M}_{11},\widetilde{M}_{12})=\argmin_{(z_1,z_2)\in (\cL^{2}(\Omega))^{2}} \left[\frac{1}{2} \int_{\Omega} \left(|z_1-B_{11}|^2+|z_2-B_{12}|^2\right) d\bx +\tau \int_{\Omega} \Delta_2|M_{22}^kz_1-M_{21}^kz_2|d\bx\right]
	\label{eq.alterH.1}
\end{align}
and update
\begin{align}
	M_{11}^{k+1}=(1-\rho_2)M_{11}^k+\rho_2 \widetilde{M}_{11},\\
	M_{12}^{k+1}=(1-\rho_2)M_{12}^k+\rho_2 \widetilde{M}_{12}.
\end{align}
\textbf{Step 2}: Solve
\begin{align}
	(\widetilde{M}_{22},\widetilde{M}_{21})=\argmin_{(z_1,z_2)\in (\cL^{2}(\Omega))^{2}} \left[\frac{1}{2} \int_{\Omega} \left(|z_1-B_{22}|^2+|z_2-B_{21}|^2\right) d\bx +\tau \int_{\Omega} \Delta_2|M_{11}^{k+1}z_1-M_{12}^{k+1}z_2|d\bx\right]
	\label{eq.alterH.2}
\end{align}
and update
\begin{align}
	M_{22}^{k+1}=(1-\rho_2)M_{22}^k+\rho_2 \widetilde{M}_{22},\\
	M_{21}^{k+1}=(1-\rho_2)M_{21}^k+\rho_2 \widetilde{M}_{21}.
	\label{eq.alterH.end}
\end{align}
The above procedure is repeated until
$$
\max\left\{\|M_{11}^{k+1}-M_{11}^k\|_{\infty}, \|M_{22}^{k+1}-M_{22}^k\|_{\infty}, \|M_{12}^{k+1}-M_{12}^k\|_{\infty}, \|M_{21}^{k+1}-M_{21}^k\|_{\infty}\right\}\leq \xi_2
$$
for some small $\xi_2$.

Problems (\ref{eq.alterH.1}) and (\ref{eq.alterH.2}) can be solved pixel-wise. On each pixel, one needs to solve a minimization problem in the form of
\begin{align}
	(v_1,v_2)=\argmin_{(w_1,w_2)\in \RR^2} \left[\frac{1}{2} \left( |w_1-b_1|^2+|w_2-b_2|^2\right) +c|a_1w_1-a_2w_2|\right]
	\label{eq.alterH.general}
\end{align}
for some constants $a_1,a_2,b_1,b_2,c\in \RR$ with $c>0$. The closed-form solution of (\ref{eq.alterH.general}) is summarized in the following theorem.
\begin{theorem}\label{thm.H}
	The closed-form solution of (\ref{eq.alterH.general}) is given in the following five cases:
	\begin{enumerate}[{Case} 1:]
		\item $a_1=0$. The solution is
		\begin{align}
			v_1=b_1,\ v_2=\max\left\{ 0, 1-\frac{c|a_2|}{|b_2|}\right\}b_2.
		\end{align}
		\item $a_2=0$. The solution is
		\begin{align}
			v_1=\max\left\{ 0, 1-\frac{c|a_1|}{|b_1|}\right\}b_1,\ v_2=b_2.
		\end{align}
		\item $a_1\neq0,a_2\neq0$ and $(a_1b_1-a_2b_2)-(a_1^2+a_2^2)c>0.$ The solution is
		\begin{align}
			v_1=b_1-ca_1,\ v_2=b_2+ca_2.
		\end{align}
		\item $a_1\neq0,a_2\neq0$ and $(a_1b_1-a_2b_2)+(a_1^2+a_2^2)c<0.$ The solution is
		\begin{align}
			v_1=b_1+ca_1,\ v_2=b_2-ca_2.
		\end{align}
		\item Other cases. The solution is
		\begin{align}
			v_1=\frac{a_2^2b_1+a_1a_2b_2}{a_1^2+a_2^2}, \ v_2=\frac{a_1a_2b_1+a_1^2b_2}{a_1^2+a_2^2}.
		\end{align}
	\end{enumerate}
\end{theorem}
\begin{proof}
	We derive the closed form solution for each case.
	
	%We divide the 5 cases in Theorem \ref{thm.H} into two groups: the first group contains case 1 and 2, the second group contains case 3-5. We are going to derive the solutions to each group.
	\emph{Cases 1 and 2:} Case 1 and 2 are very similar to each other. We derive the expression of the solution in Case 1. The solution in Case 2 can be derived analogously. When $a_1=0$, (\ref{eq.alterH.general}) reduces to
	\begin{align}
		(v_1,v_2)=&\argmin_{(w_1,w_2)\in \RR^2} \left[\frac{1}{2} (|w_1-b_1|^2+|w_2-b_2|^2) +c|a_2||w_2|\right]\nonumber\\
		=&\argmin_{w_1\in \RR} \frac{1}{2}|w_1-b_1|^2 +\argmin_{w_2\in \RR} \left[\frac{1}{2}|w_2-b_2|^2+c|a_2||w_2|\right].
		\label{eq.case1}
	\end{align}
	In (\ref{eq.case1}), the minimization problem with respect to $w_1$ has solution $v_1=b_1$. The minimization problem with respect to $w_2$ is a common one in image processing; its solution is given via the shrinkage operator \cite{donoho1995noising}
	\begin{align}
		v_2=\max\left\{ 0, 1-\frac{c|a_2|}{|b_2|}\right\}b_2.
	\end{align}
	
	\emph{Case 3-5:} In Case 3-5, $a_1,a_2\neq 0$. When $a_1v_1-a_2v_2>0$, the optimality condition of $(v_1,v_2)$ is
	\begin{align}
		\begin{cases}
			v_1-b_1+ca_1=0,\\
			v_2-b_2-ca_2=0,
		\end{cases}
	\end{align}
	which gives $v_1=b_1-ca_1,\ v_2=b_2+ca_2$. Substituting this expression into the condition $a_1v_1-a_2v_2>0$ yields
	\begin{align}
		(a_1b_1-a_2b_2)-c(a_1^2+a_2^2)>0,
	\end{align}
	which proves Case 3.
	
	When $a_1v_1-a_2v_2<0$, the optimality condition of $(v_1,v_2)$ is
	\begin{align}
		\begin{cases}
			v_1-b_1-ca_1=0,\\
			v_2-b_2+ca_2=0,
		\end{cases}
	\end{align}
	which gives $v_1=b_1+ca_1,\ v_2=b_2-ca_2$. Substituting this expression into the condition $a_1v_1-a_2v_2<0$ yields
	\begin{align}
		(a_1b_1-a_2b_2)+c(a_1^2+a_2^2)<0,
	\end{align}
	which proves Case 4.
	
	For Case 5, the condition is $a_1,a_2\neq0$, $(a_1b_1-a_2b_2)-c(a_1^2+a_2^2)\leq0$ and $(a_1b_1-a_2b_2)+c(a_1^2+a_2^2)\geq0$. Under this condition, the optimality conditions in Cases 3 and 4 can not be satisfied. Therefore we must have $a_1v_1-a_2v_2=0$. Since $a_1,a_2\neq 0$, we can write $v_2=a_1v_1/a_2$. Then (\ref{eq.alterH.general}) reduces to
	\begin{align}
		\begin{cases}
			v_1=\argmin\limits_{w_1\in \RR} \frac{1}{2} \left(|w_1-b_1|^2+\left|\frac{a_1w_1}{a_2}-b_2\right|^2 \right)\\
			v_2=a_1v_1/a_2.
		\end{cases}
		\label{eq.case5}
	\end{align}
	The functional in (\ref{eq.case5}) is a quadratic form of $v_1$, implying that
	\begin{align*}
		v_1=\frac{a_2^2b_1+a_1a_2b_2}{a_1^2+a_2^2}, \ v_2=\frac{a_1a_2b_1+a_1^2b_2}{a_1^2+a_2^2}.
	\end{align*}
\end{proof}
{ Similar to our discusion in Section \ref{sec.step2} on the convergence of the fixed-point iteration for $\bp^{n+1/4}$, $\bH^n$ is a good initial guess of the iteration (\ref{eq.alterH.1})--(\ref{eq.alterH.end}) and the iteration should converge to the minimizer fast. This is verified by our numerical experiments.
}
\subsection{On the solution of (\ref{eq.split.3.dis})}
In (\ref{eq.split.3.dis}), $\bp^{n+2/4}$ solves the following problem
\begin{align}
	\min_{\bq\in (\cL^{2}(\Omega))^2} \left[\frac{\gamma}{2}\int_{\Omega} |\bq-\bp^{n+1/4}|^2d\bx +\tau\alpha\int_{\Omega} |\bq|d\bx\right].
	\label{eq.p.2}
\end{align}
We have $\bp^{n+2/4}$ closed form through the shrinkage operation, namely
\begin{align}
	\bp^{n+2/4}=\max\left\{ 0, 1-\frac{\tau\alpha/\gamma}{|\bp^{n+1/4}|}\right\}\bp^{n+1/4}.
	\label{eq.p3}
\end{align}
Then we set $\bH^{n+2/4}=\bH^{n+1/4}$.

\subsection{On the solution of (\ref{eq.split.4.dis})}
%\subsubsection{Compute $(\bp^{n+3/4},\bH^{n+3/4})$}
In (\ref{eq.split.4.dis}), $(\bp^{n+3/4},\bH^{n+3/4})$ solves
\begin{align}
	\begin{cases}
		\bH^{n+3/4}=\nabla \bp^{n+3/4},\\
		\bp^{n+3/4}=\argmin\limits_{\bq\in (\cH^1(\Omega))^2} \left[\frac{1}{2}\displaystyle\int_{\Omega}\left( \gamma|\bq-\bp^{n+2/4}|^2+|\nabla \bq-\bH^{n+2/4}|^2\right) d\bx\right].
	\end{cases}
	\label{eq.step4}
\end{align}
It follows from (\ref{eq.step4}) that $\bp^{n+3/4}$ is the unique solution to
\begin{align}
	\begin{cases}
		\bp^{n+3/4}\in (\cH^1(\Omega))^2,\\
		\displaystyle\int_{\Omega} \left(\gamma\bp^{n+3/4}\cdot \bq +\nabla \bp^{n+3/4}\colon\nabla\bq\right)d\bx= \displaystyle\int_{\Omega} \left( \gamma \bp^{n+2/4}\cdot \bq +\bH^{n+2/4}\cdot \nabla \bq\right) d\bx,\\
		\forall \bq\in (\cH^1(\Omega))^2,
	\end{cases}
\end{align}
where $\nabla \bp\colon\nabla\bq=\nabla p_1\cdot \nabla q_1+ \nabla p_2\cdot \nabla q_2$. Note that $\bp^{n+3/4}$ is also the unique weak solution of the following linear elliptic problem (a Neumann problem)
\begin{align}
	\begin{cases}
		-\nabla^2 p^{n+3/4}_k+\gamma p^{n+3/4}_k=\gamma p^{n+2/4}_k-\nabla \cdot\bH^{n+2/4}_k &\mbox{ in } \Omega,\\
		\left(\nabla p_k^{n+3/4}-\bH^{n+2/4}_k\right)\cdot \bn=0 &\mbox{ on } \partial \Omega,\\
		\mbox{for } k=1,2,
	\end{cases}
	\label{eq.p4}
\end{align}
where $\bH^{n+2/4}_k=[H_{k1}^{n+2/4}, H_{k2}^{n+2/4}]^{\top}$. %Assuming periodic boundary conditions, (\ref{eq.p4}) can be solved efficiently using FFT, as shown later.
%\subsubsection{On the choice of $\gamma$}
%We would like to choose $\gamma$ so that the two terms in (\ref{eq.step4}) are balanced. It is equivalent to balance the two terms on the left-hand side of (\ref{eq.p4}). (\ref{eq.p4}) suggests to choose $\gamma$ as the smallest eigenvalue of $-\nabla^2$ on $\Omega$. Assume $\Omega$ is a rectangular region $[0,L_1]\times[0,L_2]$ for $L_1,L_2>0$.  We set
%\begin{align}
%  \gamma=\pi^2\left(\frac{1}{L_1^2}+\frac{1}{L_2^2}\right).
%\end{align}
\subsection{On the solution of (\ref{eq.split.5.dis})}
In (\ref{eq.split.5.dis}), $\bH^{n+1}=\bH^{n+3/4}$ and $\bp^{n+1}$ is the solution to
\begin{align}
	\begin{cases}
		u^{n+1}=\argmin\limits_{v\in \cH^1(\Omega)} \left[\frac{1}{2}\displaystyle\int_{\Omega} \gamma|\nabla v-\bp^{n+3/4}|^2d\bx +\frac{\tau}{2\beta}\displaystyle\int_{\Omega} |f-v|^2d\bx\right],\\
		\bp^{n+1}=\nabla u^{n+1}.
	\end{cases}
	\label{eq.p5}
\end{align}
From (\ref{eq.p5}), $u^{n+1}$ is the unique solution to the linear variational problem
\begin{align}
	\begin{cases}
		u^{n+1}\in \cH^1(\Omega),\\
		\displaystyle\int_{\Omega} \gamma\nabla u^{n+1}\cdot \nabla vd\bx + \frac{\tau}{\beta}\displaystyle\int_{\Omega} u^{n+1}vd\bx =\frac{\tau}{\beta}\displaystyle\int_{\Omega} fvd\bx +\gamma \displaystyle\int_{\Omega} \bp^{n+3/4}\cdot \nabla vd\bx,\\
		\forall v\in \cH^1(\Omega).
	\end{cases}
\end{align}
Note that $u^{n+1}\in \cH^1(\Omega)$ is also the weak solution to the following Neumann problem
\begin{align}
	\begin{cases}
		-\gamma\nabla^2 u^{n+1} +\frac{\tau}{\beta} u^{n+1} =\frac{\tau}{\beta} f-\nabla\cdot(\gamma \bp^{n+3/4}) &\mbox{ in } \Omega,\\
		\left(\nabla u^{n+1}-\bp^{n+3/4}\right)\cdot \bn=0 &\mbox{ on } \partial \Omega.
	\end{cases}
	\label{eq.p5.1}
\end{align}
%Assuming periodic boundary conditions, similar to (\ref{eq.p4}), (\ref{eq.p5.1}) can be solved efficiently using FFT, as discussed later.

Our algorithm is summarized in Algorithm \ref{alg.1} below:
\begin{algorithm}
	\caption{\label{alg.1}An operator-splitting method for solving problem (\ref{eq.model.2})}
	\begin{algorithmic}
		\STATE {\bf Input:} The noisy function $f$, parameters $\alpha,\beta,\gamma,\tau$.
		\STATE {\bf Initialization:} $n=0,$ $(\bp^0,\bH^0)=(\bp_0,\bH_0)$.
		\WHILE{not converge}
		\STATE 1. Solve (\ref{eq.split.2.dis}) using (\ref{eq.p1.fix1})-(\ref{eq.p1.fix3}), (\ref{eq.alterH.1})-(\ref{eq.alterH.end}) to obtain $(\bp^{n+1/4}, \bH^{n+1/4})$.
		\STATE 2. Solve (\ref{eq.split.3.dis}) using (\ref{eq.p3}) to obtain $(\bp^{n+2/4}, \bH^{n+2/4})$.
		\STATE 3. Solve (\ref{eq.split.4.dis}) using (\ref{eq.p4}) to obtain $(\bp^{n+3/4}, \bH^{n+3/4})$.
		\STATE 4. Solve (\ref{eq.split.5.dis}) using (\ref{eq.p5.1}) to obtain $(\bp^{n+1}, \bH^{n+1})$.
		\STATE 5. Set $n=n+1$.
		\ENDWHILE
		\STATE Solve (\ref{eq.vq}) using the converged function $\bp^*$ to obtain $u^*$.
		\STATE {\bf Output:} The function $u^*$.
	\end{algorithmic}
\end{algorithm}
\subsection{Initial condition}\label{sec.initial}
Scheme (\ref{eq.split.1.dis})-(\ref{eq.split.5.dis}) requires an initial condition $(\bp_0,\bH_0)$. One simple choice is
\begin{align}
	\bp_0=\nabla f, \ \bH_0=\nabla \bp_0.
	\label{eq.initial1}
\end{align}

A more sophisticated choice is to set $\bp_0$ as the gradient of a smoothed $f$. Let $\varepsilon>0$ be a small constant. We first solve
\begin{align}
	\begin{cases}
		u_0\in \cH^1(\Omega),\\
		\displaystyle\int_{\Omega} u_0vd\bx +\varepsilon\displaystyle\int_{\Omega}\nabla u_0\cdot \nabla vd\bx=\displaystyle\int_{\Omega} fvd\bx,\\
		\forall v\in \cH^1(\Omega).
	\end{cases}
	\label{eq.initial2.1}
\end{align}
Here $u_0$ is the weak solution of
\begin{align}
	\begin{cases}
		u_0-\varepsilon\nabla^2 u_0=f &\mbox{ in } \Omega,\\
		\nabla u_0\cdot  \bn(=\partial u_0/ \partial \bn)=0 &\mbox{ on } \partial \Omega.
	\end{cases}
	\label{eq.initial2.1.pde}
\end{align}
Then we take
\begin{align}
	\bp_0=\nabla u_0,\ \bH_0=\nabla\bp_0.
	\label{eq.initial2.2}
\end{align}

\subsection{On periodic boundary conditions}\label{sec.periodic}
Periodic boundary conditions are commonly used in image processing and enable one to use FFT when solving certain elliptic linear PDEs. { The operator-splitting method and the solvers to each subproblem discussed so far consider Neumann boundary conditions.} In this subsection, we discuss the minimal efforts one needs to modify the aforementioned algorithm and solvers in order to handle periodic boundary conditions.

Assume that our computational domain is $\Omega=[0,L_1]\times [0,L_2]$. The first modification one needs is to replace the functional space $\cH^1(\Omega)$ by $\cH^1_P(\Omega)$ defined as
\begin{align*}
	&\cH_P^1(\Omega)=\left\{v\in \cH^1(\Omega): v(0,:)=v(L_1,:), v(:,0)=v(:,L_2)\right\}.
\end{align*}
Correspondingly, the sets $\Sigma$ and $S$ are replaced by
\begin{align*}
	&\Sigma=\left\{\bq| \bq\in (\cL^2(\Omega))^2, \exists v\in \cH_P^1(\Omega) \mbox{ such that } \bq=\nabla v \mbox{ and } \int_{\Omega} (f-v)d\bx=0\right\},\\
	&S=\left\{ (\bq,\bG)| (\bq,\bG)\in \left(\cH^1_P(\Omega)\right)^2\times \left(\cL^2(\Omega)\right)^{2\times 2} \mbox{ such that } \bG=\nabla \bq\right\}.
\end{align*}
Problems (\ref{eq.model.2}) and (\ref{eq.vq}) are replaced by
\begin{align}
	\min\limits_{\substack{\bq\in \left(\cH_P^1(\Omega)\right)^2, \\ \bG\in \left(\cL^{2}(\Omega)\right)^{2\times 2}}} \left[\displaystyle\int_{\Omega} \frac{|\det\bG|}{(1+|\bq|^2)^{3/2}}d\bx+\alpha\displaystyle\int_{\Omega}|\bq|d\bx +\frac{1}{2\beta}\displaystyle\int_{\Omega}|f-v_{\bq}|^2d\bx + I_{\Sigma}(\bq) +I_{S}(\bq,\bG)\right]
	\label{eq.model.2.periodic}
\end{align}
and
\begin{align}
	\begin{cases}
		\nabla^2v_{\bq}=\nabla\cdot \bq \mbox{ in } \Omega,\\
		v_{\bq} \mbox{ verifies periodic doundary conditions},\\
		(\nabla v_{\bq}-\bq)\cdot \be_j \mbox{ is periodic in the $Ox_j$-direction}, \forall j=1,2,\\
		\displaystyle\int_{\Omega}(f-v_{\bq})d\bx=0,
	\end{cases}
	\label{eq.vq.periodic}
\end{align}
respectively. In (\ref{eq.vq.periodic}), $\be_j$ is the unit vector of the $Ox_j$ direction.

%The following linear variational problem arises  in several locations of our article
%\begin{align}
%	\begin{cases}
%		u\in \cH_P^1(\Omega),\\
%		\displaystyle\int_{\Omega} [uv+c\nabla u\cdot\nabla v]d\bx=\displaystyle\int_{\Omega} fvd\bx, \forall  v\in \cH_P^1(\Omega),
%	\end{cases}
%\label{eq.periodic.var}
%\end{align}
%with $c$ a positive constant and $\Omega=(0,L_1)\times (0,L_2)$.

%One can easily show that the unique solution of (\ref{eq.periodic.var}) is also the solution of
%\begin{align}
%	\begin{cases}
%		u-c\nabla^2 u=f \mbox{ in } \Omega,\\
%		u(0,x_2)=u(L_1,x_2), \ 0< x_2<L_2,\ u(x_1,0)=u(x_1,L_2),\ 0<x_1<L_1,\\
%		\frac{\partial u}{\partial x_1}(0,x_2)=\frac{\partial u}{\partial x_1}(L_1,x_2), \ 0<x_2<L_2\\
%		\quad \quad \Leftrightarrow \frac{\partial u}{\partial \bn}(0,x_2)+\frac{\partial u}{\partial \bn}(L_1,x_2)=0,\ 0<x_2<L_2,\\
%		\frac{\partial u}{\partial x_2}(x_1,0)=\frac{\partial u}{\partial x_2}(x_1,L_2), \ 0<x_1<L_1\\
%		\quad \quad \Leftrightarrow \frac{\partial u}{\partial \bn}(x_1,0)+\frac{\partial u}{\partial \bn}(x_1,L_2)=0,\ 0<x_1<L_1.
%	\end{cases}
%\label{eq.periodic.varPDE}
%\end{align}
%Indeed, we have periodicity of $\nabla u$ but (\ref{eq.periodic.varPDE}) is the right way to reformulate (\ref{eq.periodic.var}) in PDE form.

Finally, we modify the subproblem solvers as follows:\\
For $(\bp^{n+3/4},\bH^{n+3/4})$, replace (\ref{eq.step4}) and (\ref{eq.p4}) by
\begin{align}
	\begin{cases}
		\bH^{n+3/4}=\nabla \bp^{n+3/4},\\
		\bp^{n+3/4}=\argmin\limits_{\bq\in \left(\cH_P^1(\Omega)\right)^2} \left[\frac{1}{2}\displaystyle\int_{\Omega} \left( \gamma|\bq-\bp^{n+2/4}|^2+|\nabla \bq-\bH^{n+2/4}|^2\right)d\bx\right]
	\end{cases}
	\label{eq.step4.periodic}
\end{align}
and
\begin{align}
	\begin{cases}
		-\nabla^2 p^{n+3/4}_k+\gamma p^{n+3/4}_k=\gamma p^{n+2/4}_k-\nabla \cdot\bH^{n+2/4}_k \mbox{ in } \Omega,\\
		p_k^{n+3/4}(0,x_2)=p_k^{n+3/4}(L_1,x_2), \ 0< x_2<L_2,\\ p_k^{n+3/4}(x_1,0)=p_k^{n+3/4}(x_1,L_2),\ 0<x_1<L_1,\\
		\left(\frac{ \partial p_k^{n+3/4}}{\partial x_1}-H^{n+2/4}_{k1}\right)(0,x_2)= \left(\frac{ \partial p_k^{n+3/4}}{\partial x_1}-H^{n+2/4}_{k1}\right)(L_1,x_2), \ 0< x_2<L_2,\\
		\left(\frac{ \partial p_k^{n+3/4}}{\partial x_2}-H^{n+2/4}_{k2}\right)(x_1,0)= \left(\frac{ \partial p_k^{n+3/4}}{\partial x_2}-H^{n+2/4}_{k2}\right)(x_1,L_2), \ 0< x_1<L_1,\\
		\mbox{for } k=1,2,
	\end{cases}
	\label{eq.p4.periodic}
\end{align}
respectively.

For $\bp^{n+1}$, replace (\ref{eq.p5}) and (\ref{eq.p5.1}) by
\begin{align}
	\begin{cases}
		u^{n+1}=\argmin\limits_{v\in \cH^1_P(\Omega)} \left[\frac{1}{2}\displaystyle\int_{\Omega} \gamma|\nabla v-\bp^{n+3/4}|^2d\bx +\frac{\tau}{2\beta}\displaystyle\int_{\Omega} |f-v|^2d\bx\right],\\
		\bp^{n+1}=\nabla u^{n+1},
	\end{cases}
	\label{eq.p5.periodic}
\end{align}
and
\begin{align}
	\begin{cases}
		-\gamma\nabla^2 u^{n+1} +\frac{\tau}{\beta} u^{n+1} =\frac{\tau}{\beta} f-\nabla\cdot(\gamma \bp^{n+3/4}) \mbox{ in } \Omega,\\
		u^{n+1}(0,x_2)=u^{n+1}(L_1,x_2), \ 0< x_2<L_2,\\
		u^{n+1}(x_1,0)=u^{n+1}(x_1,L_2),\ 0<x_1<L_1,\\
		\left(\frac{ \partial u^{n+1}}{\partial x_1}-p^{n+3/4}_1\right)(0,x_2)= \left(\frac{ \partial u^{n+1}}{\partial x_1}-p^{n+3/4}_1\right)(L_1,x_2), \ 0< x_2<L_2,\\
		\left(\frac{ \partial u^{n+1}}{\partial x_2}-p^{n+3/4}_2\right)(x_1,0)= \left(\frac{ \partial u^{n+1}}{\partial x_2}-p^{n+3/4}_2\right)(x_1,L_2), \ 0< x_1<L_1,
	\end{cases}
	\label{eq.p5.1.periodic}
\end{align}
respectively.

We replace the initial conditions (\ref{eq.initial2.1}) and (\ref{eq.initial2.1.pde}) by
\begin{align}
	\begin{cases}
		u_0\in \cH_P^1(\Omega),\\
		\displaystyle\int_{\Omega} u_0vd\bx +\varepsilon\displaystyle\int_{\Omega}\nabla u_0\cdot \nabla vd\bx=\displaystyle\int_{\Omega} fvd\bx,\\
		\forall v\in \cH_P^1(\Omega),
	\end{cases}
	\label{eq.initial2.1.periodic}
\end{align}
and
\begin{align}
	\begin{cases}
		u_0-\varepsilon\nabla^2 u_0=f \mbox{ in } \Omega,\\
		u_0(0,x_2)=u_0(L_1,x_2), \ 0< x_2<L_2,\\
		u_0(x_1,0)=u_0(x_1,L_2),\ 0<x_1<L_1,\\
		\frac{\partial u_0}{\partial x_1}(0,x_2)=\frac{\partial u_0}{\partial x_1}(L_1,x_2), \ 0<x_2<L_2\\
		\frac{\partial u_0}{\partial x_2}(x_1,0)=\frac{\partial u_0}{\partial x_2}(x_1,L_2), \ 0<x_1<L_1,
	\end{cases}
	\label{eq.initial2.1.pde.periodic}
\end{align}
respectively.

Problems (\ref{eq.p4.periodic}), (\ref{eq.p5.1.periodic}) and (\ref{eq.initial2.1.pde.periodic}) are linear elliptic problems with periodic boundary conditions. Their finite difference analogues can be solved efficiently by FFT, as shown in Sections \ref{sec.p.3}, \ref{sec.p.4} and \ref{sec.initial.dis}. In the remainder of this article (Section \ref{sec.discretization} and \ref{sec.experiments}), periodic boundary conditions are used.
\section{Space discretization}\label{sec.discretization}
In this section, we discuss the finite difference analogues of (\ref{eq.split.1.dis})-(\ref{eq.split.5.dis}) with periodic boundary conditions. Let $\Omega=(0,L_1)\times(0,L_2)$ be discretized by $M\times N$ grids with step size $h=L_1/M=L_2/N$. For simplicity, we denote by $v(i,j)$ the approximate value of $v$ at $(ih,jh)$ for any function $v$ defined on $\Omega$. Assume that all of the variables mentioned before satisfy periodic boundary conditions. %We remark that although in (\ref{eq.p4}) and (\ref{eq.p5}), the boundary condition is the Neumann condition. When we solve for $\bp^{n+3/4}$ and $\bp^{n+1}$, we enforce the periodic boundary condition.

We first define the forward $(+)$ and backward $(-)$ finite differences for $1\leq i\leq M,\ 1\leq j\leq N$:
\begin{align*}
	&\partial_1^+ v(i,j)=\left(v(i+1,j)-v(i,j)\right)/h,\\
	&\partial_1^- v(i,j)=\left(v(i,j)-v(i-1,j)\right)/h,\\
	&\partial_2^+ v(i,j)=\left(v(i,j+1)-v(i,j)\right)/h,\\
	&\partial_2^- v(i,j)=\left(v(i,j)-v(i,j-1)\right)/h,
\end{align*}
where $v(M+1,j)=v(1,j), v(-1,j)=v(M,j)$ and $v(i,N+1)=v(i,1), v(i,-1)=v(i,N)$ are used. With the above notation, the forward $(+)$ and backward $(-)$ gradient operators for a scalar-valued function $v$ are defined by
\begin{align*}
	\nabla^{\pm}v(i,j)=(\partial_1^{\pm}v(i,j),\partial_2^{\pm}v(i,j)).
\end{align*}
Correspondingly, the forward $(+)$ and backward $(-)$ divergence and gradient operators for a vector-valued function $\bq$ are defined by
\begin{align*}
	&\diver^{\pm}\bq(i,j)=\partial_1^{\pm}q_1(i,j)+\partial_2^{\pm}q_2(i,j),\ \nabla^{\pm}\bq(i,j)=\begin{pmatrix} \partial_1^{\pm}q_1(i,j) & \partial_2^{\pm}q_1(i,j)\\\partial_2^{\pm}q_2(i,j) & \partial_2^{\pm}q_2(i,j)\end{pmatrix}.
\end{align*}
We define the shifting and identity operator by
\begin{align}
	\cS_1^{\pm}v(i,j)=v(i\pm 1,j), \ \cS_2^{\pm}v(i,j)=v(i,j\pm 1),\ \cI v(i,j)=v(i,j).
\end{align}
Denote the discrete Fourier transform and its inverse by $\cF$ and $\cF^{-1}$, respectively. We have
\begin{align}
	&\cF(\cS_1^{\pm}v)(i,j)=(\cos z_i \pm \sqrt{-1}\sin z_i)\cF(v)(i,j), \\
	&\cF(\cS_2^{\pm}v)(i,j)=(\cos z_j \pm \sqrt{-1}\sin z_j)\cF(v)(i,j),
\end{align}
with
\begin{align}
	z_i=\frac{2\pi}{M}(i-1),\ z_j=\frac{2\pi}{N}(j-1).
	\label{eq.z}
\end{align}

\subsection{Computing the discrete analogue of $\bp^{n+1/4}$ and $\bH^{n+1/4}$}
For the discrete analogue of $\bp^{n+1/4}$, we first compute
\begin{align}
	\Delta_1(i,j)=H_{11}(i,j)H_{22}(i,j)-H_{12}(i,j)H_{21}(i,j), \ |\bq^k(i,j)|^2=(q_1^k(i,j))^2+(q_2^k(i,j))^2.
\end{align}
Then $\bq^{k+1}$ is updated according to (\ref{eq.p1.fix1})-(\ref{eq.p1.fix3}) pixelwisely. After $\bq^{k+1}$ has converged to $\bq^*$, we set $\bp^{n+1/4}=\bq^*$.

For the discrete analogue of $\bH^{n+1/4}$, we compute
\begin{align*}
	\Delta_2(i,j)=\left(1+\left(p_1^{n+1/4}(i,j)\right)^2+\left(p_2^{n+1/4}(i,j)\right)^2\right)^{-3/2}.
\end{align*}
Then $\bM^{k+1}$ is updated pixelwisely according to (\ref{eq.alterH.1})-(\ref{eq.alterH.end}) and Theorem \ref{thm.H}. After $\bM^{k+1}$ has converged to $\bM^*$, we set $\bH^{n+1}=\bM^*$.

\subsection{Computing the discrete analogue of $\bp^{n+2/4}$ and $\bH^{n+2/4}$}
According to (\ref{eq.p3}), we compute
\begin{align}
	\bp^{n+2/4}(i,j)=\max\left\{ 0, 1-\frac{\tau\alpha/\gamma}{\sqrt{\left(p_1^{n+1/4}(i,j)\right)^2 +\left(p_2^{n+1/4}\right)^2}}\right\}\bp^{n+1/4}(i,j).
\end{align}
and set $\bH^{n+2/4}(i,j)=\bH^{n+1/4}(i,j)$.

\subsection{Computing the discrete analogue of $\bp^{n+3/4}$ and $\bH^{n+3/4}$}\label{sec.p.3}
We first compute $\bp^{n+3/4}$ according to (\ref{eq.p4.periodic}). Problem (\ref{eq.p4.periodic}) is discretized by
\begin{align}
	-\diver^+\nabla^- p^{n+3/4}_k+\gamma p^{n+3/4}_k=\gamma p^{n+2/4}_k-\diver^+\bH^{n+2/4}_k \mbox{ in } \Omega,
	\label{eq.p4.dis}
\end{align}
for $k=1,2$. Problem (\ref{eq.p4.dis}) can be solved efficiently by fast Fourier transform (FFT). Note that (\ref{eq.p4.dis}) can be rewritten as
\begin{align}
	\left[\gamma h^2\cI-(\cS_1^+-\cI)(\cI-\cS_1^-)-(\cS_2^+-\cI)(\cI-\cS_2^-)\right]p^{n+3/4}_k=g_k
\end{align}
with $g_k=\gamma h^2 p^{n+2/4}_k-h^2\diver^-\bH^{n+2/4}_k$ for $k=1,2$.
Applying Fourier transform for both sides, we get
\begin{align}
	a\cF(p^{n+3/4}_k)=\cF(g_k)
\end{align}
with
\begin{align*}
	a(i,j)= &\gamma h^2-(\cos z_i+\sqrt{-1}\sin z_i-1)(1-\cos z_i+\sqrt{-1}\sin z_i) \\
	&-(\cos z_j+\sqrt{-1}\sin z_j-1)(1-\cos z_j+\sqrt{-1}\sin z_j)\\
	=&\gamma h^2+4-2\cos z_i-2\cos z_j,
\end{align*}
where $z_i,z_j$ are defined in (\ref{eq.z}). Then $p_k^{n+3/4}$ is computed as
\begin{align}
	p_k^{n+3/4}=\real\left[\cF^{-1}\left(\frac{\cF(g_k)}{a}\right)\right],
\end{align}
where $\real(\cdot)$ denotes the real part of its argument. We then compute
\begin{align*}
	\bH^{n+3/4}=\nabla^-\bp^{n+3/4}.
\end{align*}

\subsection{Computing the discrete analogue of $\bp^{n+1}$ and $\bH^{n+1}$} \label{sec.p.4}
We first compute $u^{n+1}$ by solving (\ref{eq.p5.1.periodic}), which is discretized as
\begin{align}
	%\begin{cases}
	-\gamma\diver^- \nabla^+ u^{n+1} +\frac{\tau}{\beta} u^{n+1} =\frac{\tau}{\beta} f-\gamma \diver^- \bp^{n+3/4} \mbox{ in } \Omega.
	%	u^{n+1} \mbox{ has the periodic boundary condition.}
	%\end{cases}
	\label{eq.p5.dis}
\end{align}
Problem (\ref{eq.p5.dis}) can be rewritten as
\begin{align}
	\left[ \frac{\tau}{ \beta} h^2\cI-\gamma(\cI-\cS_1^-)(\cS_1^+-\cI)-\gamma(\cI-\cS_2^-)(\cS_2^+-\cI)\right] u^{n+1}=g
\end{align}
with $g=\frac{\tau}{\beta} h^2 f-\gamma h^2\diver^- \bp^{n+3/4}$. Taking the Fourier transform on both sides, we get
\begin{align}
	b\cF(u^{n+1})=\cF(g)
\end{align}
with
%\begin{align}
$b=\frac{\tau}{\beta} h^2+4\gamma-2\gamma\cos z_i-2\gamma \cos z_j,$
%\end{align}
where $z_i,z_j$ are defined in (\ref{eq.z}).

We compute
\begin{align}
	u^{n+1}=\real\left[\cF^{-1}\left(\frac{\cF(g)}{b}\right)\right]
\end{align}
and then set $\bp^{n+1}=\nabla^+ u^{n+1}, \bH^{n+1}=\bH^{n+3/4}$.

\subsection{On the discrete analogue of $(\bp_0,\bH_0)$}\label{sec.initial.dis}
For the choice of (\ref{eq.initial1}), we set
\begin{align*}
	\bp_0=\nabla^+ f,\ \bH_0=\nabla^- \bp_0.
	\label{eq.initial1.dis}
\end{align*}

For the choice of (\ref{eq.initial2.1}), one may first follow (\ref{eq.initial2.1.periodic}) to compute $u_0$ in the same way as $u^{n+1}$:
\begin{align}
	u_0=\real\left[\cF^{-1}\left(\frac{\cF(f)}{c}\right)\right]
\end{align}
with
%\begin{align*}
$	c=h^2+4\varepsilon-2\varepsilon\cos z_i-2\varepsilon \cos z_j.$
%\end{align*}
Then $\bp_0,\bH_0$ are set as
\begin{align*}
	\bp_0=\nabla^+ u_0,\ \bH_0=\nabla^- \bp_0.
\end{align*}

\begin{figure}[t!]
	\begin{tabular}{ccc}
		(a) & (b) & (c)\\
		\includegraphics[width=0.28\textwidth]{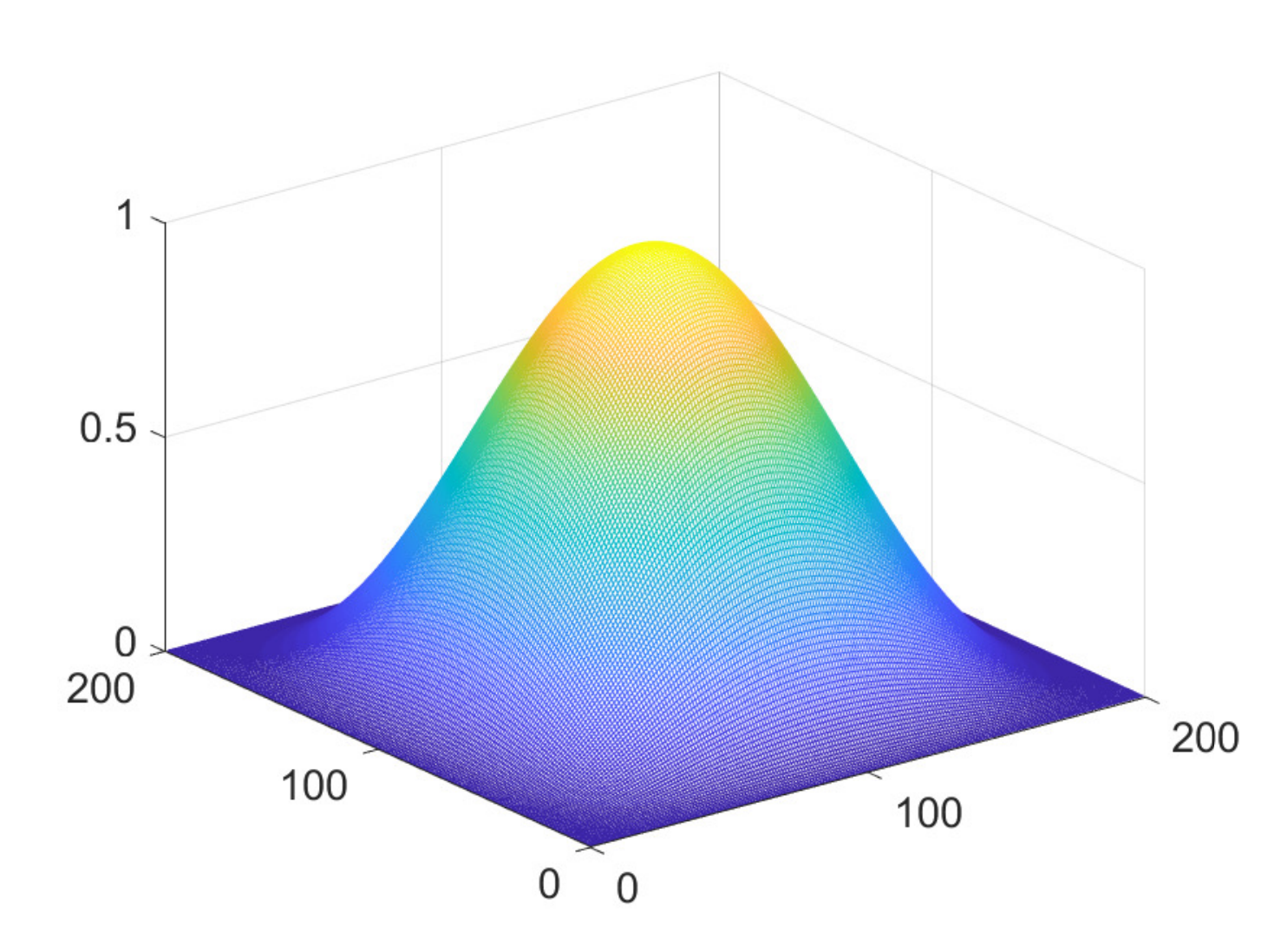}&
		\includegraphics[width=0.28\textwidth]{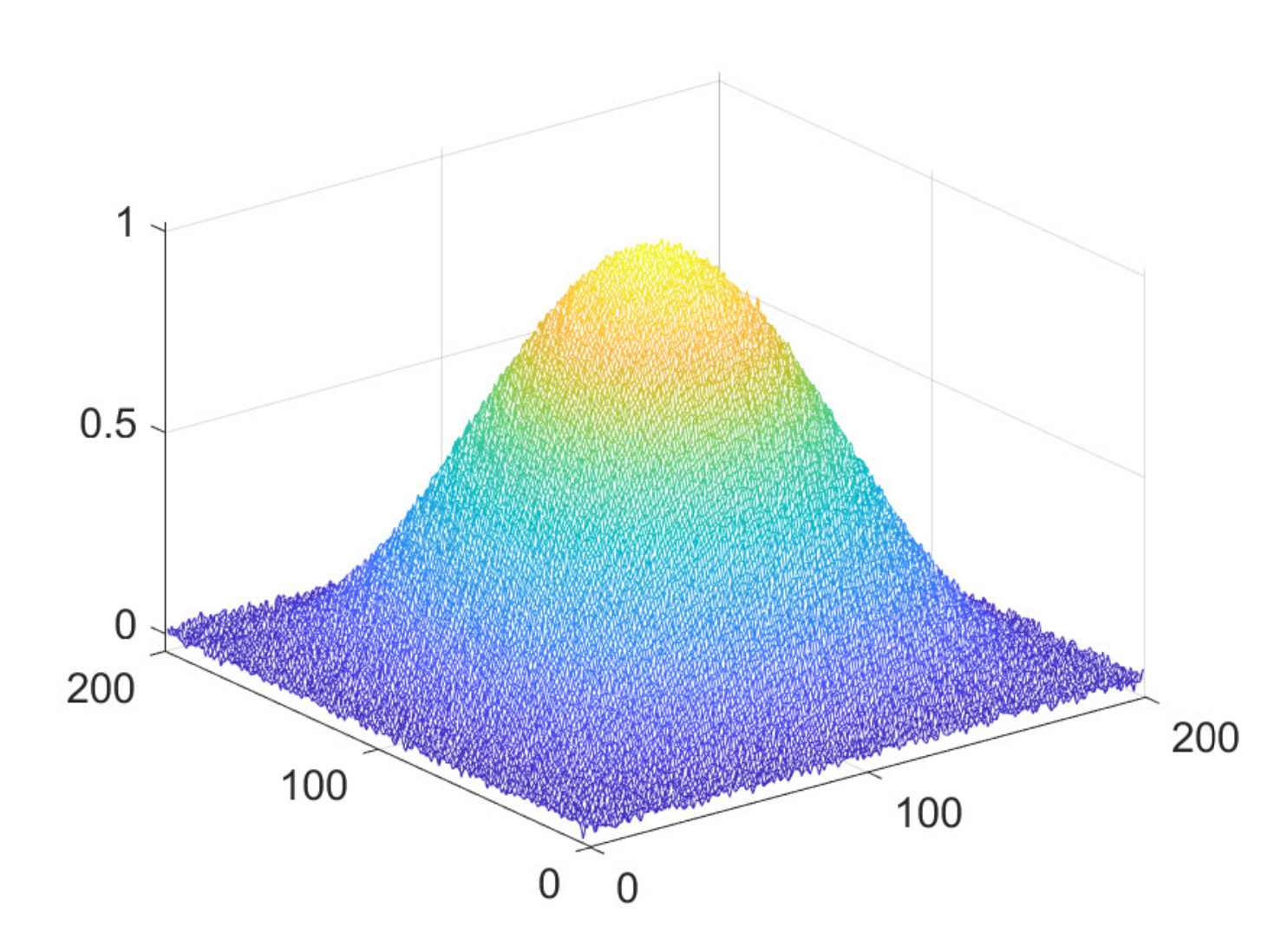}&
		\includegraphics[width=0.28\textwidth]{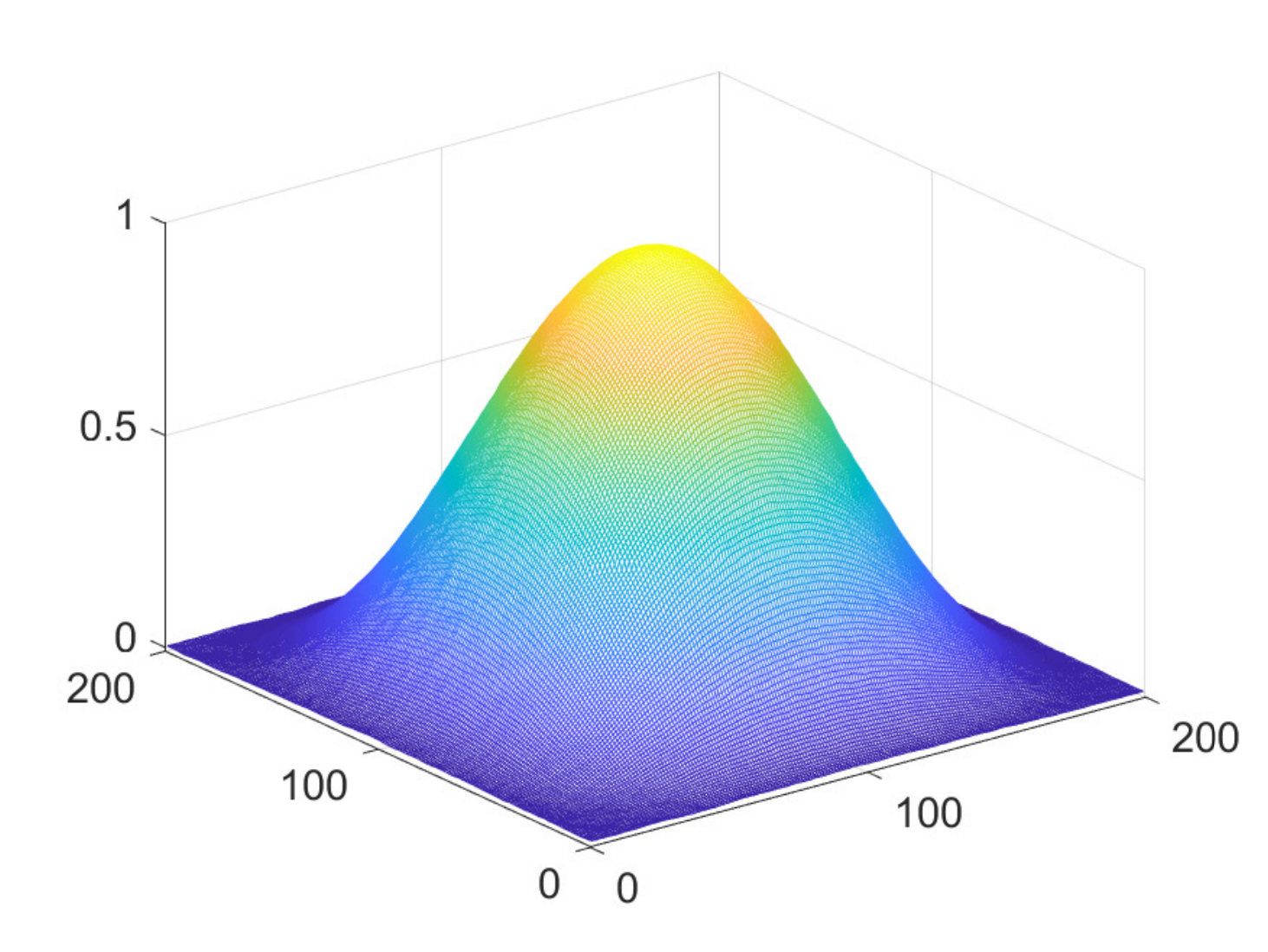}\\
		\includegraphics[width=0.28\textwidth]{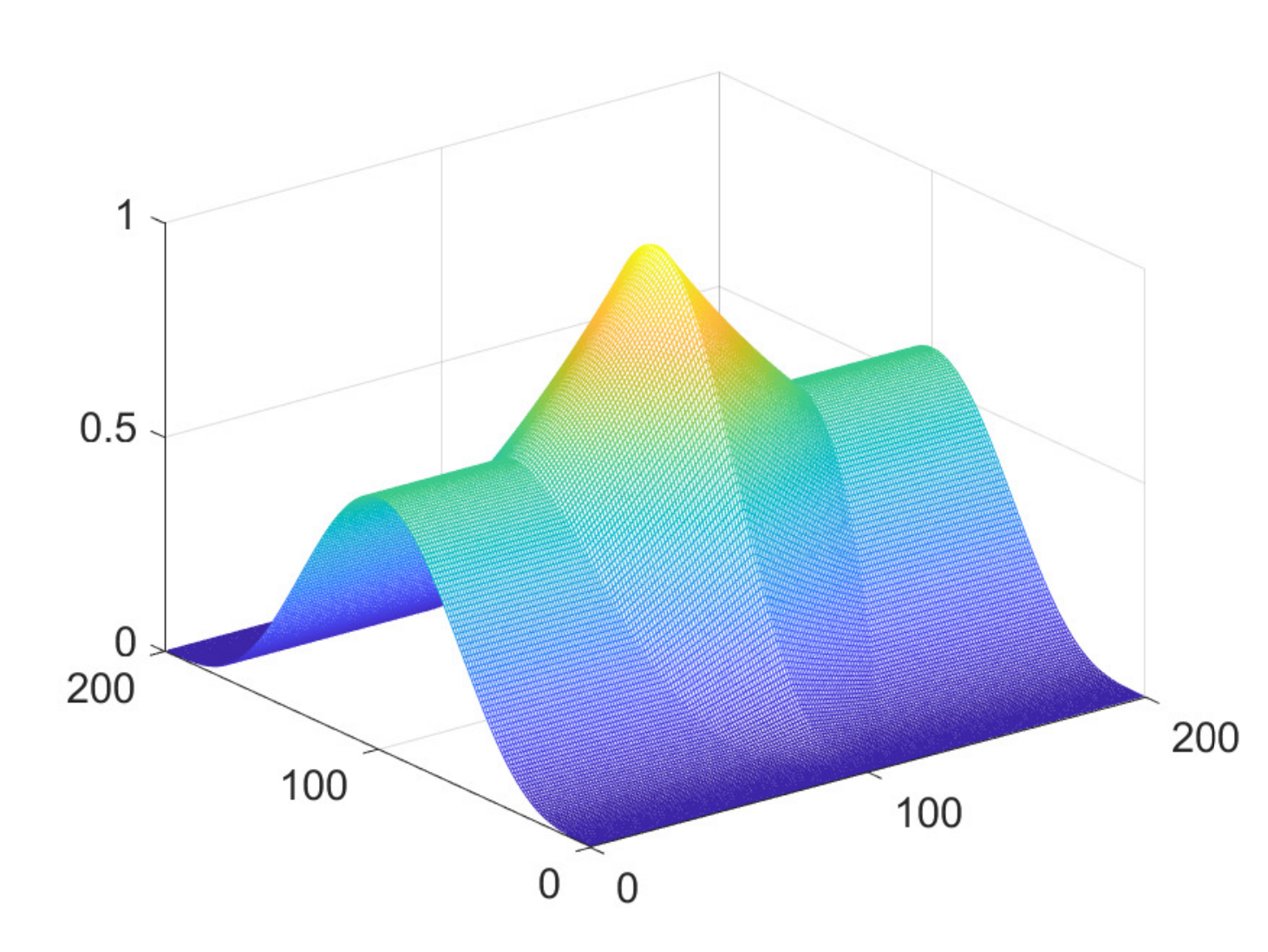}&
		\includegraphics[width=0.28\textwidth]{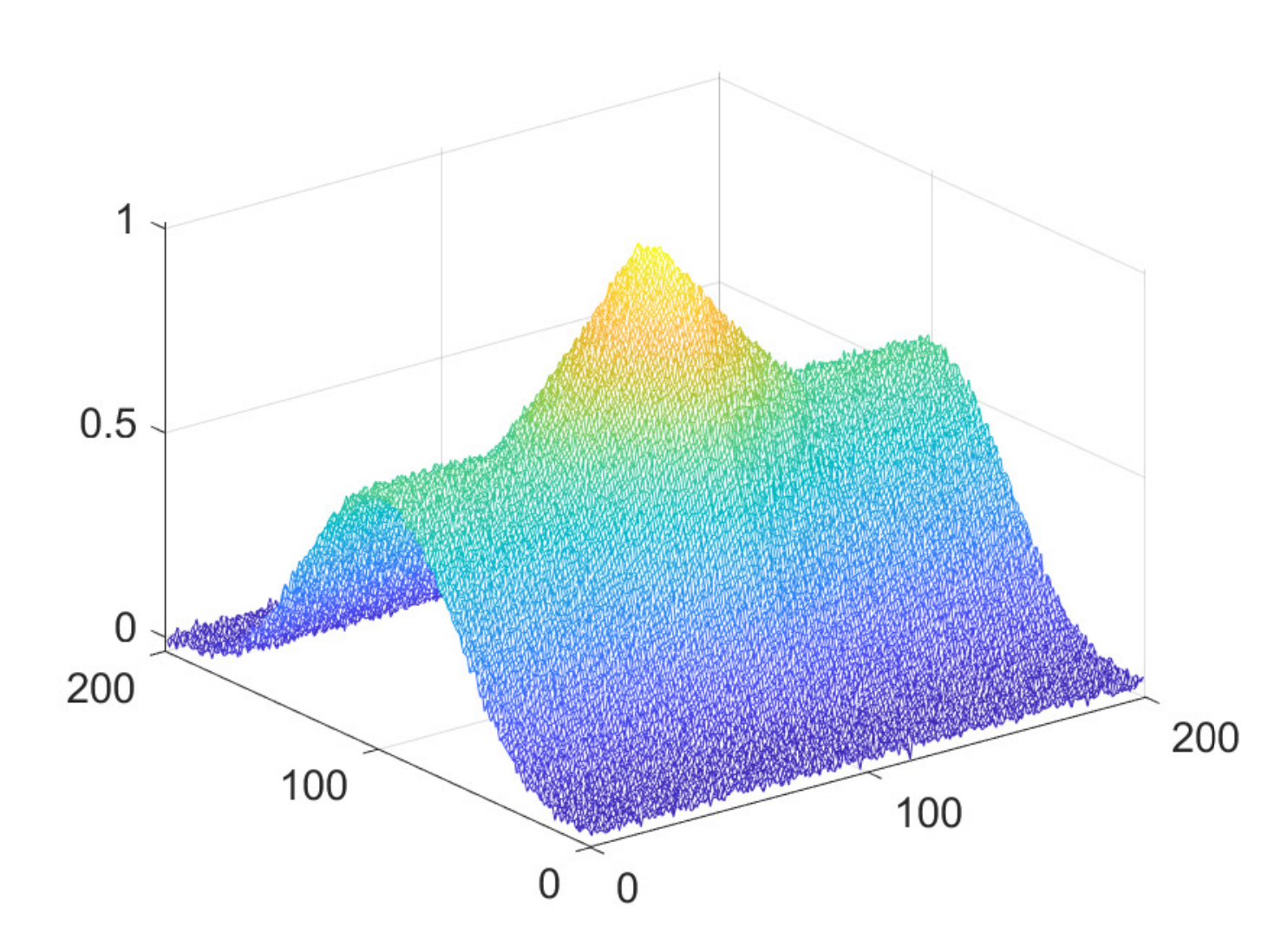}&
		\includegraphics[width=0.28\textwidth]{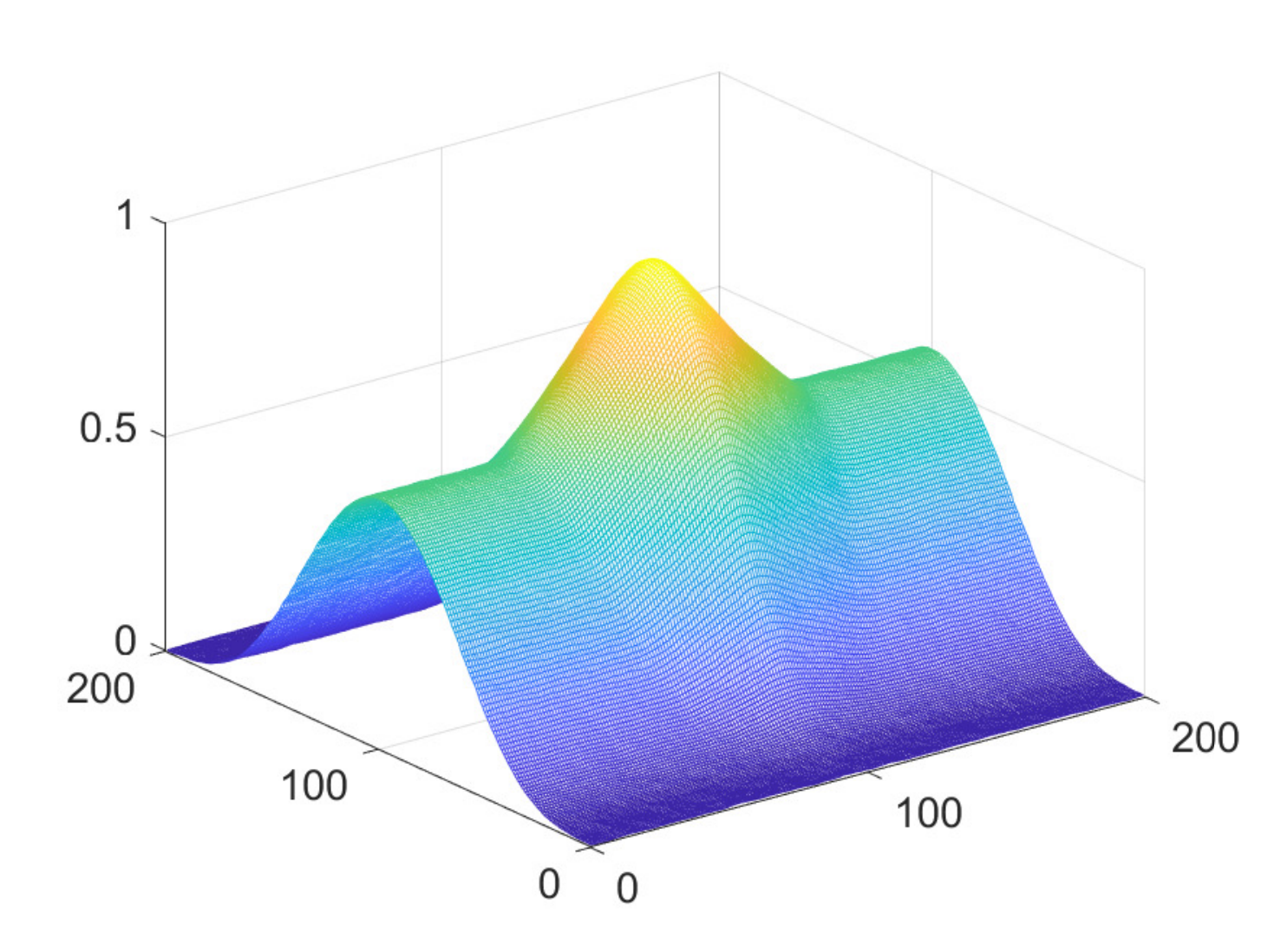}
	\end{tabular}
	\caption{(Surface smoothing.) (a) Clean surfaces. (b) Noisy surfaces with $\sigma=10^{-4}$. (c) Smoothed surfaces by the proposed model with $\alpha=1,\beta=0.1$.\vshrink}
	\label{fig.surf}
\end{figure}

\begin{figure}[t!]
	\begin{tabular}{ccc}
		(a) & (b) & (c)\\
		\includegraphics[width=0.28\textwidth]{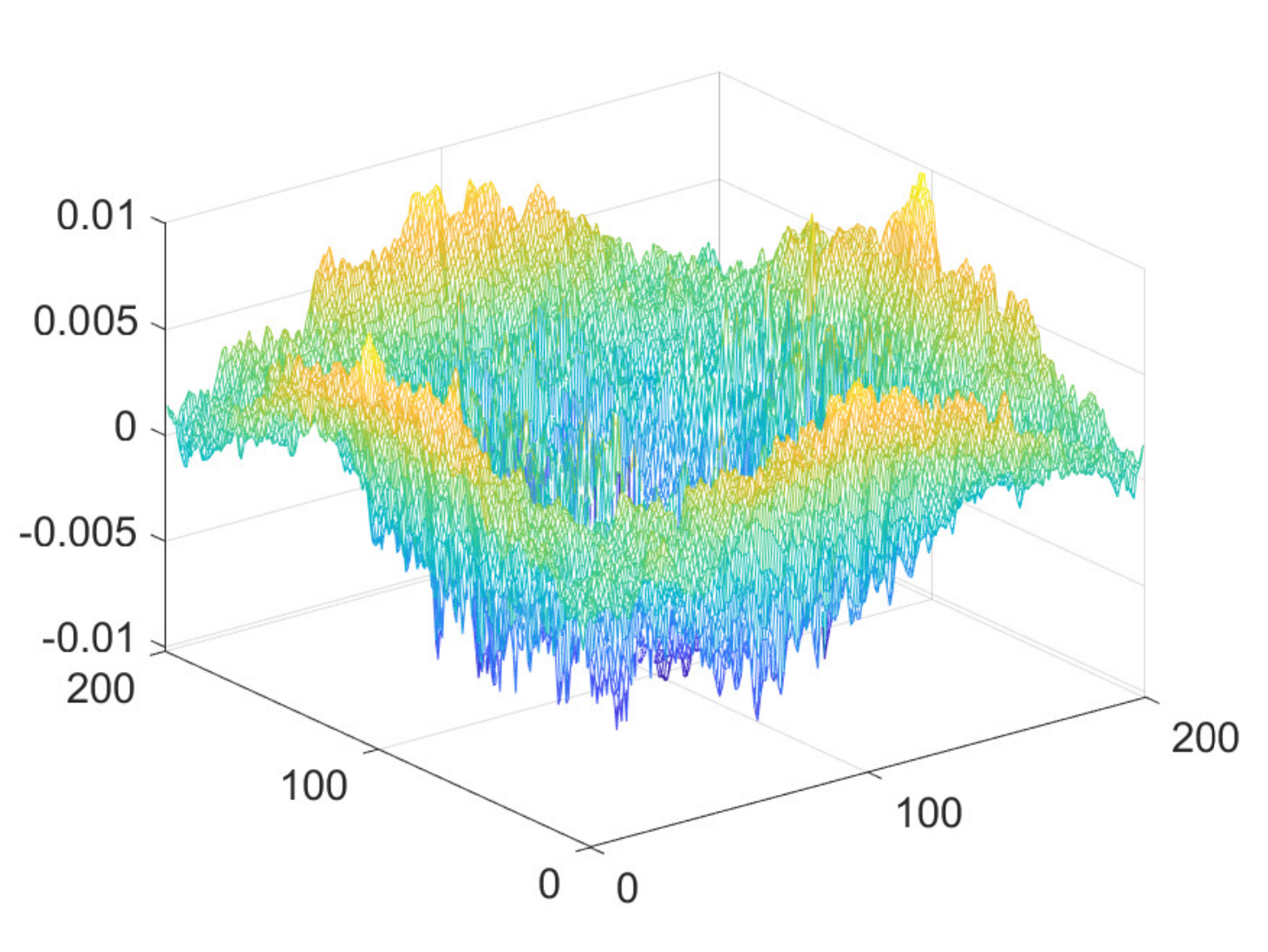}&
		\includegraphics[width=0.28\textwidth]{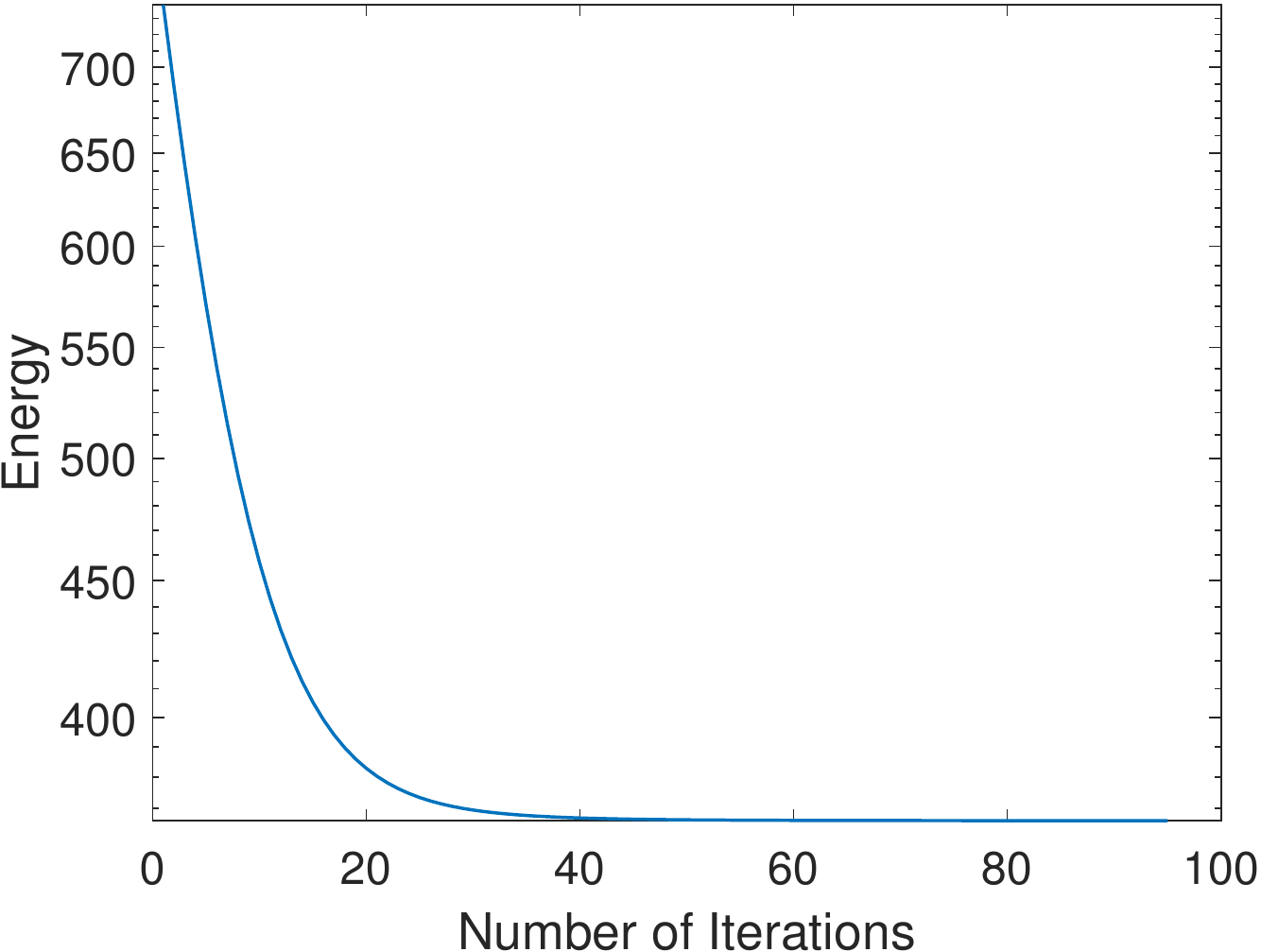}&
		\includegraphics[width=0.28\textwidth]{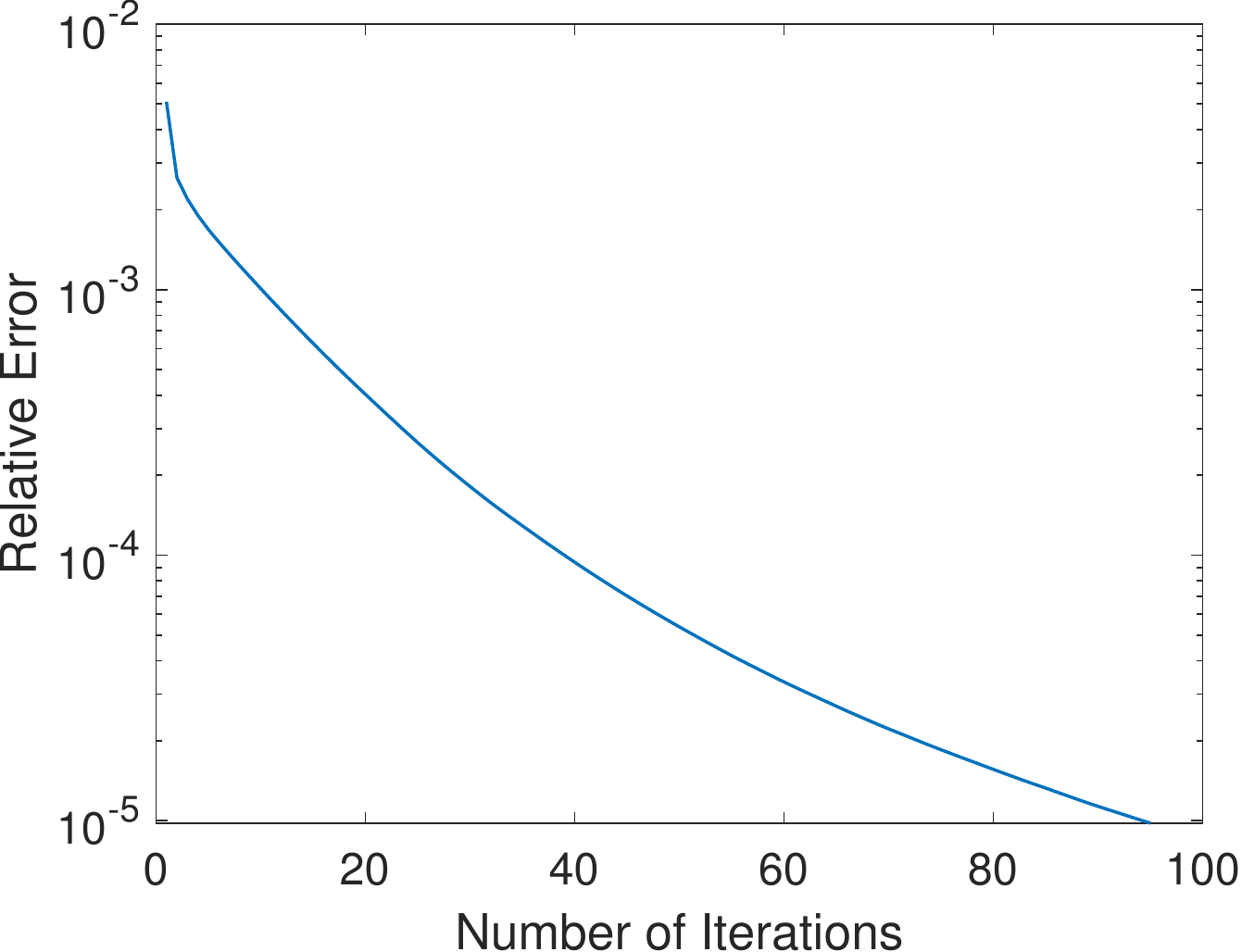}\\
		\includegraphics[width=0.28\textwidth]{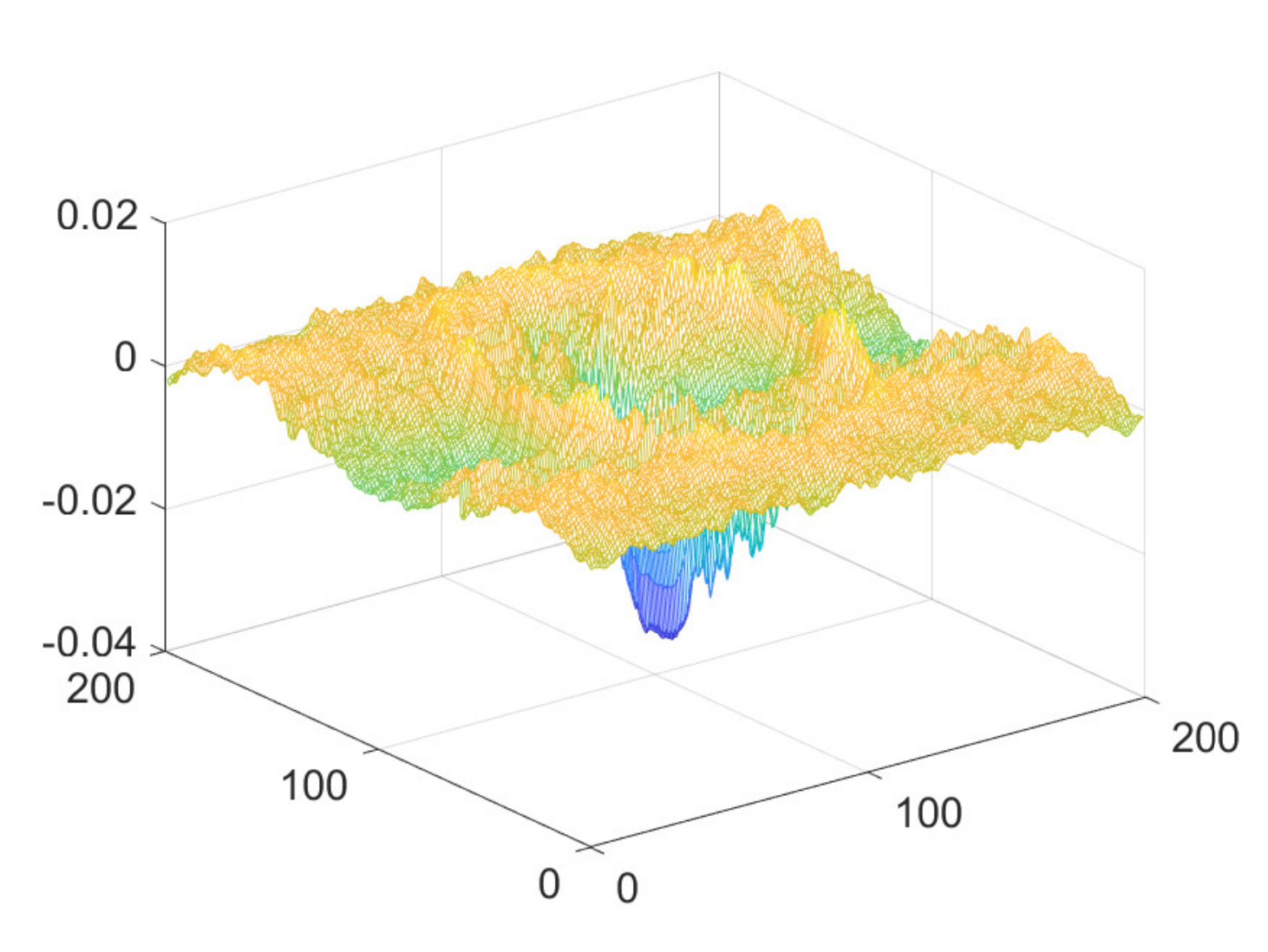}&
		\includegraphics[width=0.28\textwidth]{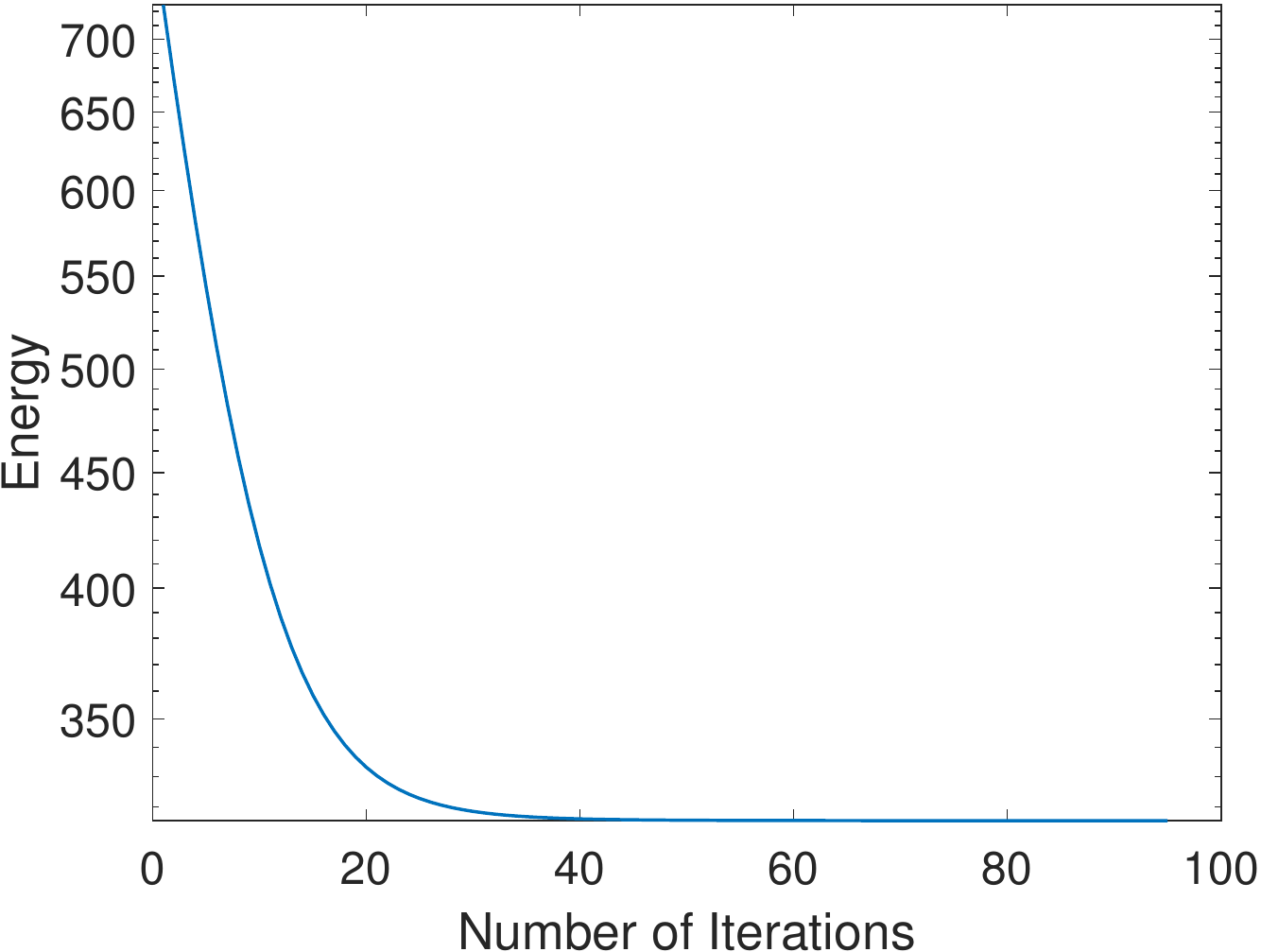}&
		\includegraphics[width=0.28\textwidth]{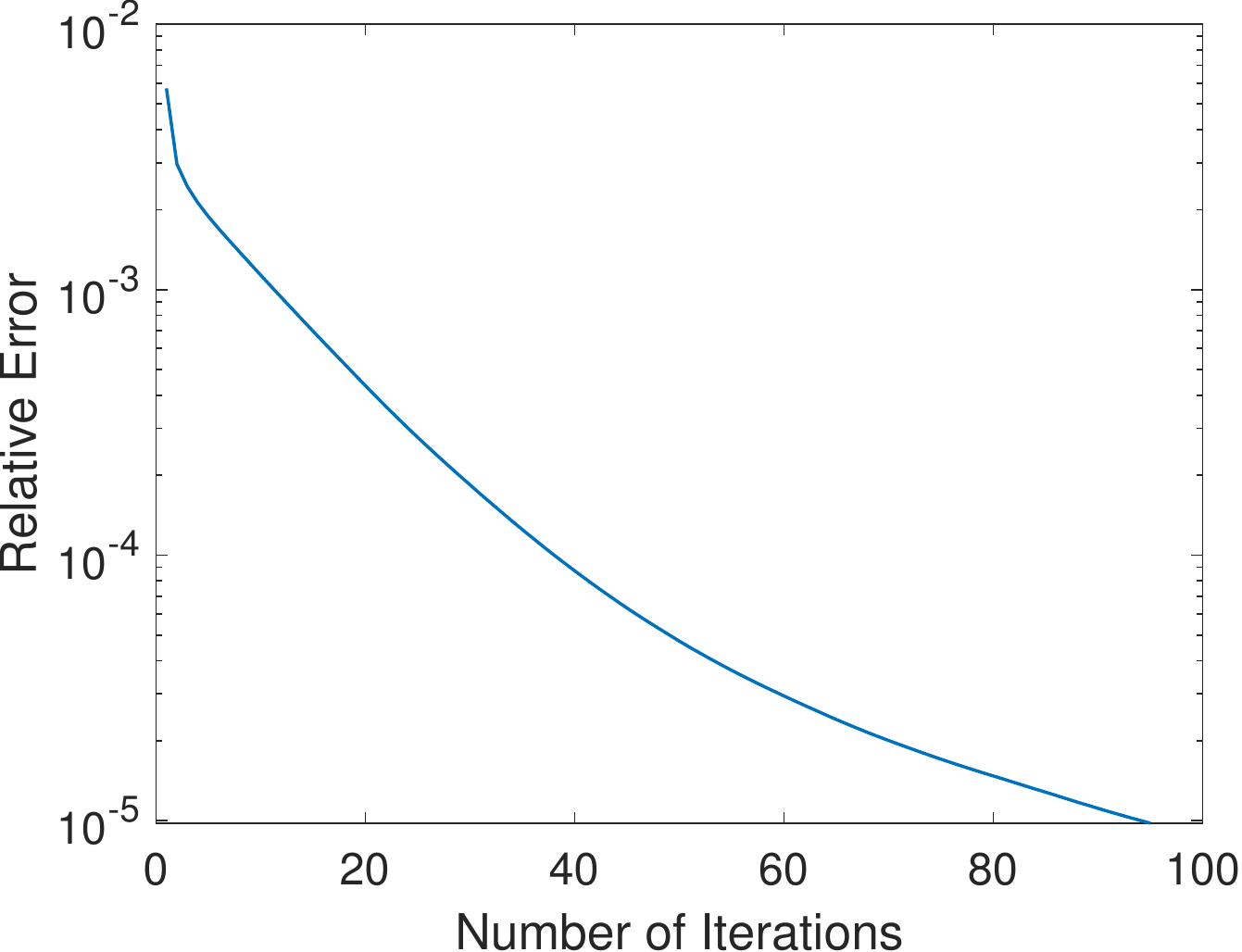}
	\end{tabular}
	\caption{(Surface smoothing.) For results in Figure \ref{fig.surf}, (a) the graph of $u-f^*$, and histories of the (b) energy and (c) relative error w.r.t. the number of iterations. The first (resp. second) row corresponds to the result in the first (resp. second) row of Figure \ref{fig.surf}(c).\vshrink}
	\label{fig.surf.err}
\end{figure}

\begin{figure}[th!]
	\begin{tabular}{cccc}
		(a) & (b) & (c) & (d)\\
		\includegraphics[trim={0.8cm 0 0.3cm 0},clip,width=0.21\textwidth]{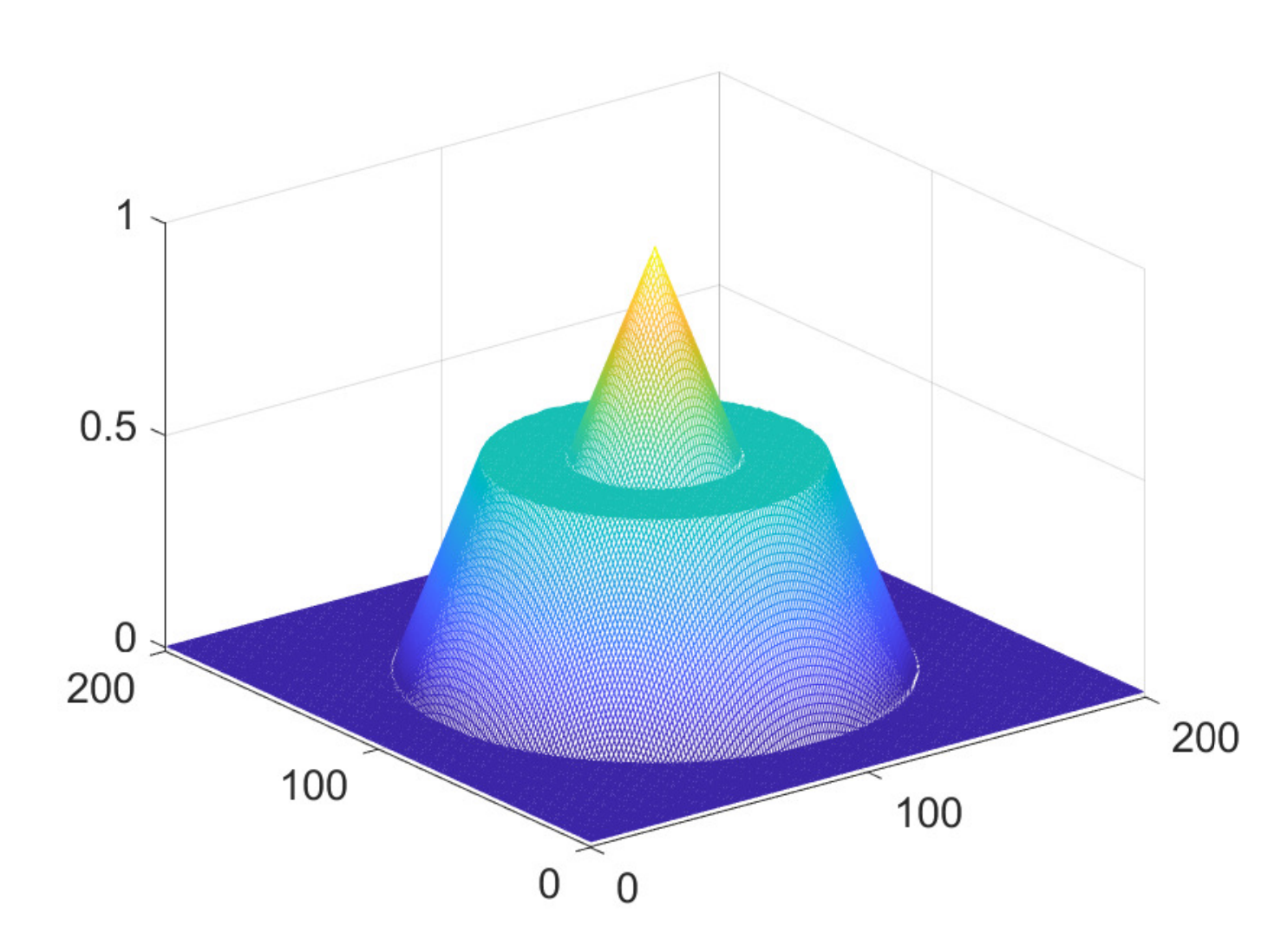}&
		\includegraphics[trim={0.8cm 0 0.3cm 0},clip,width=0.21\textwidth]{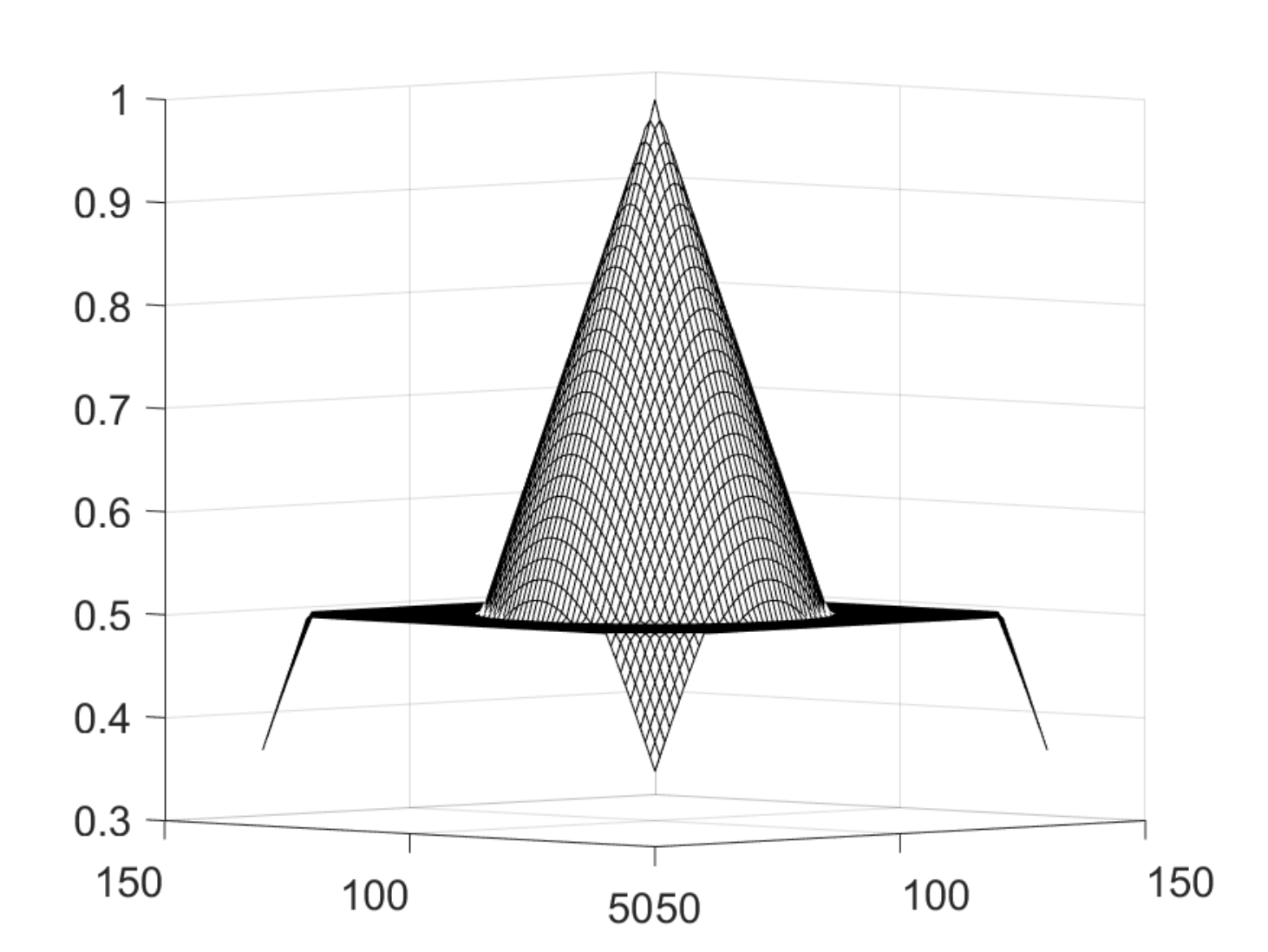}&
		\includegraphics[trim={0.8cm 0 0.3cm 0},clip,width=0.21\textwidth]{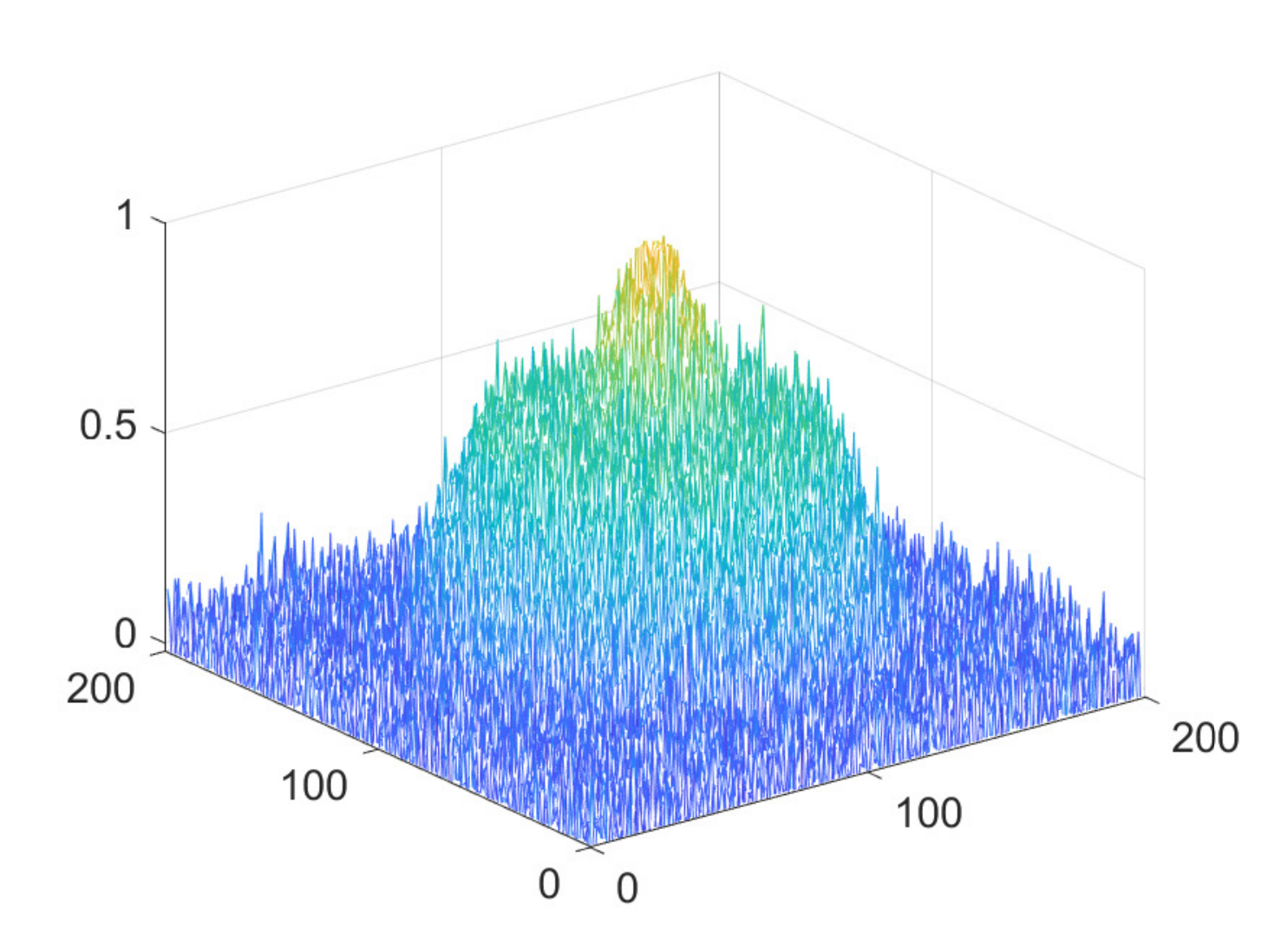}&
		\includegraphics[trim={0.8cm 0 0.3cm 0},clip,width=0.21\textwidth]{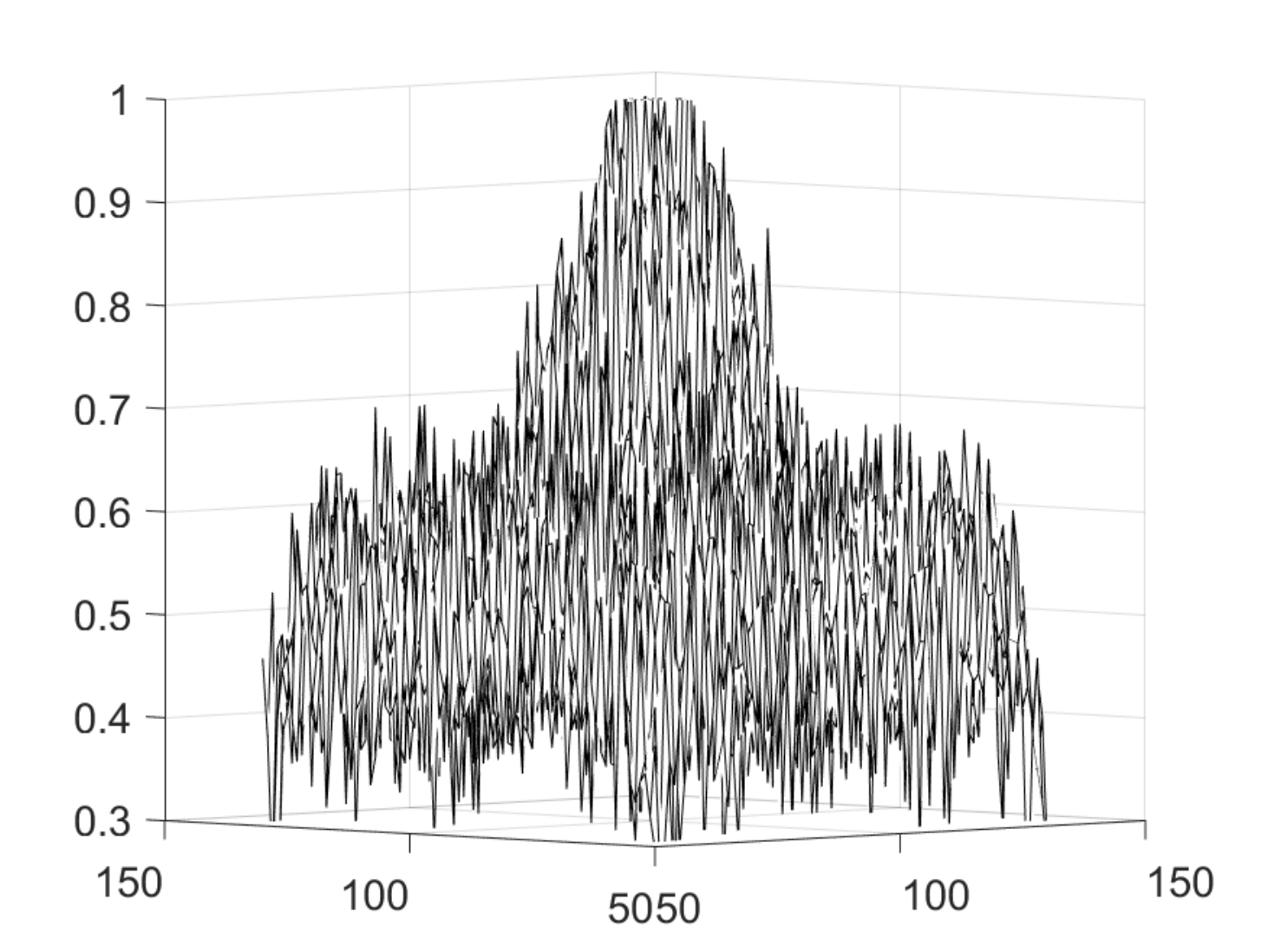}\\
		(e) & (f) & (g) & (h)\\
		\includegraphics[trim={0.8cm 0 0.3cm 0},clip,width=0.21\textwidth]{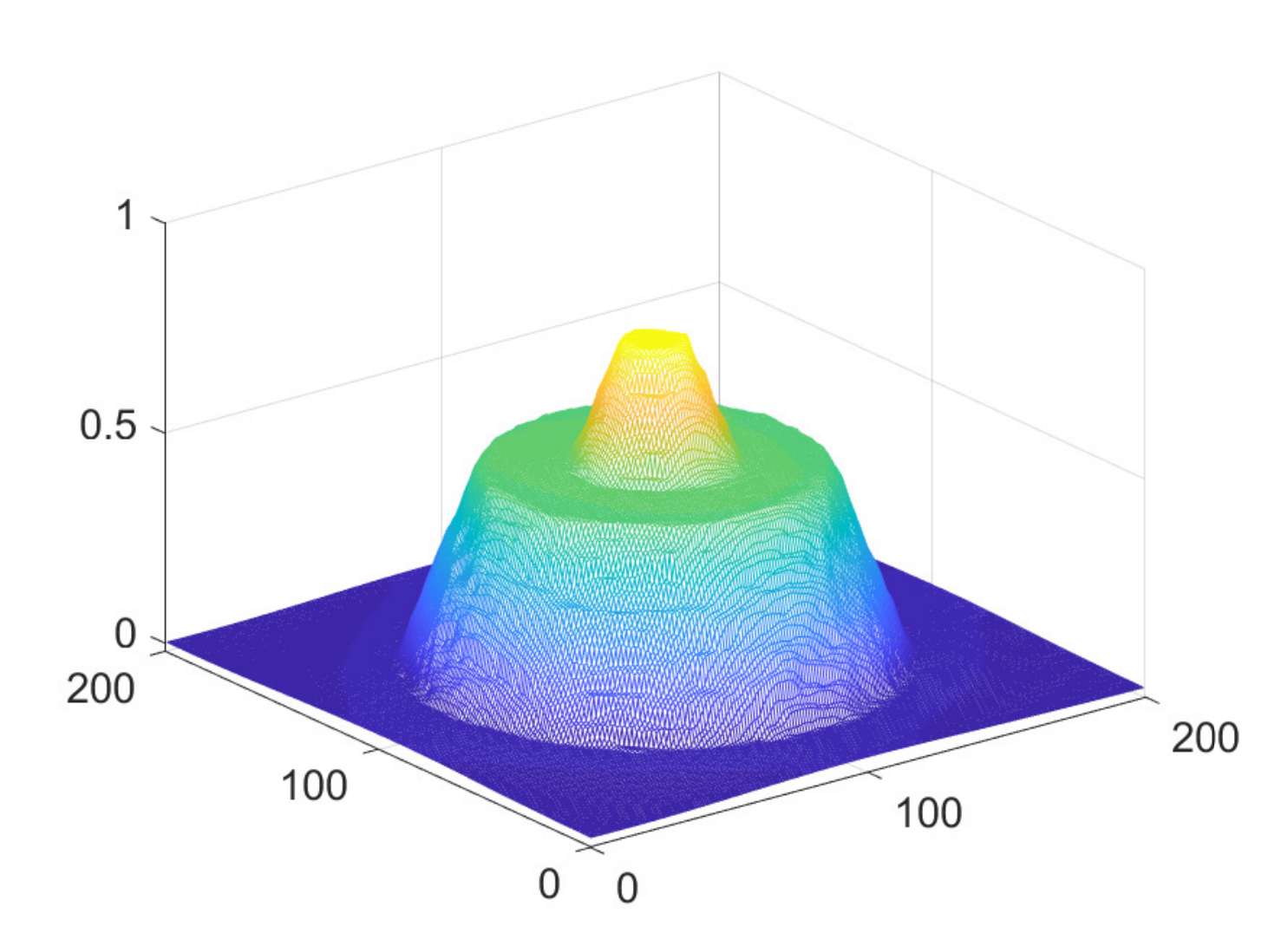}&
		\includegraphics[trim={0.8cm 0 0.3cm 0},clip,width=0.21\textwidth]{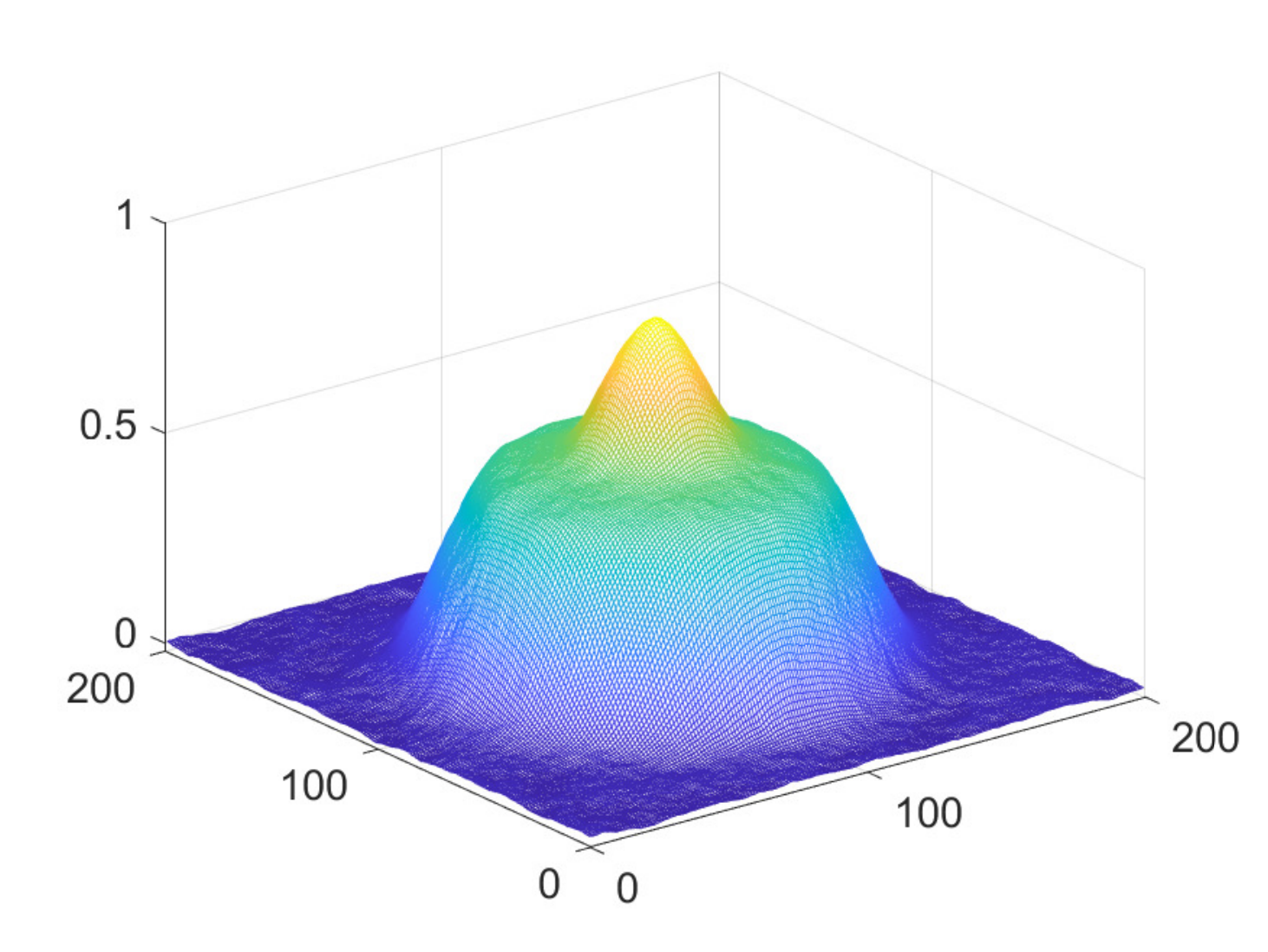}&
		\includegraphics[trim={0.8cm 0 0.3cm 0},clip,width=0.21\textwidth]{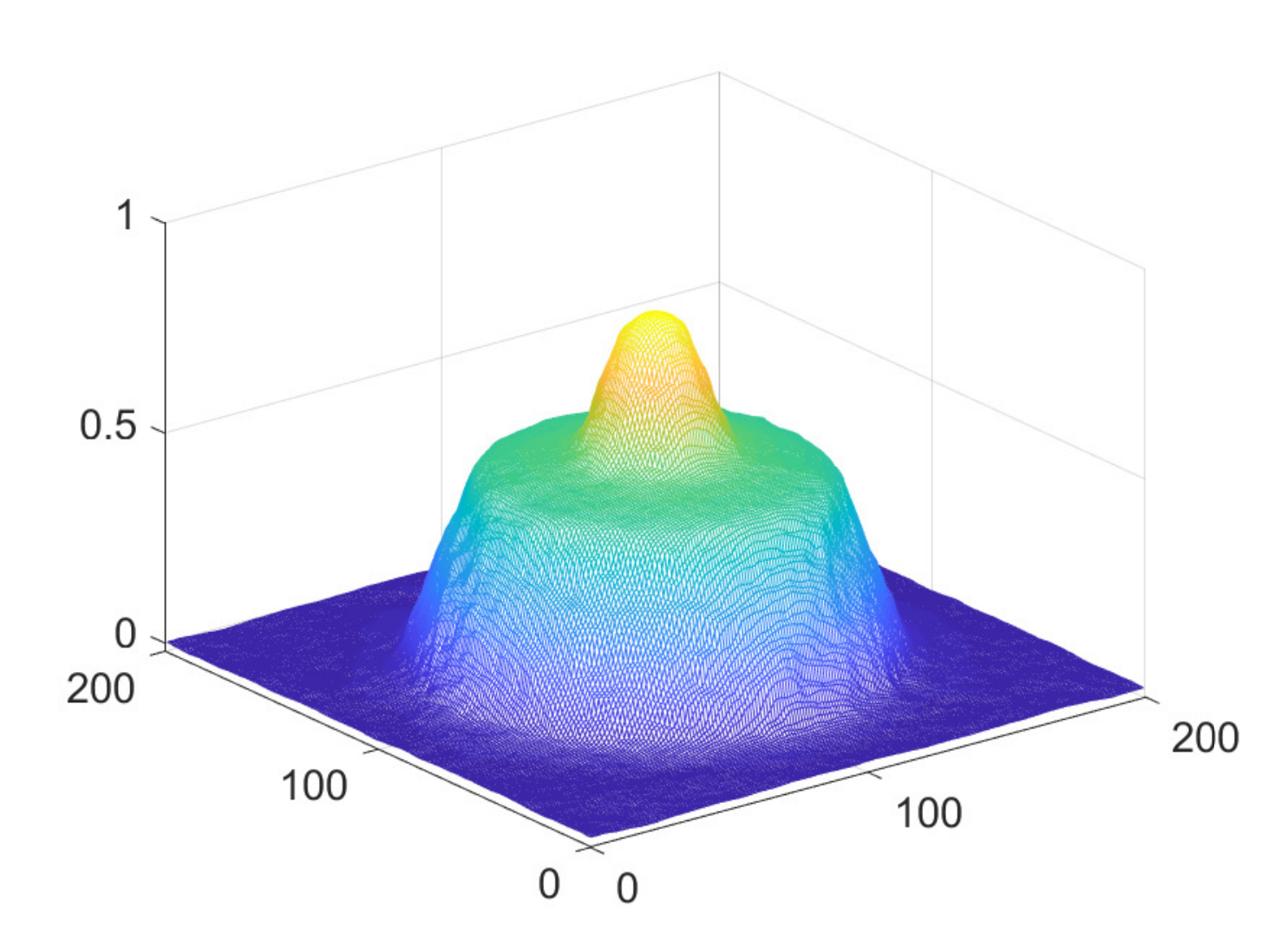}&
		\includegraphics[trim={0.8cm 0 0.3cm 0},clip,width=0.21\textwidth]{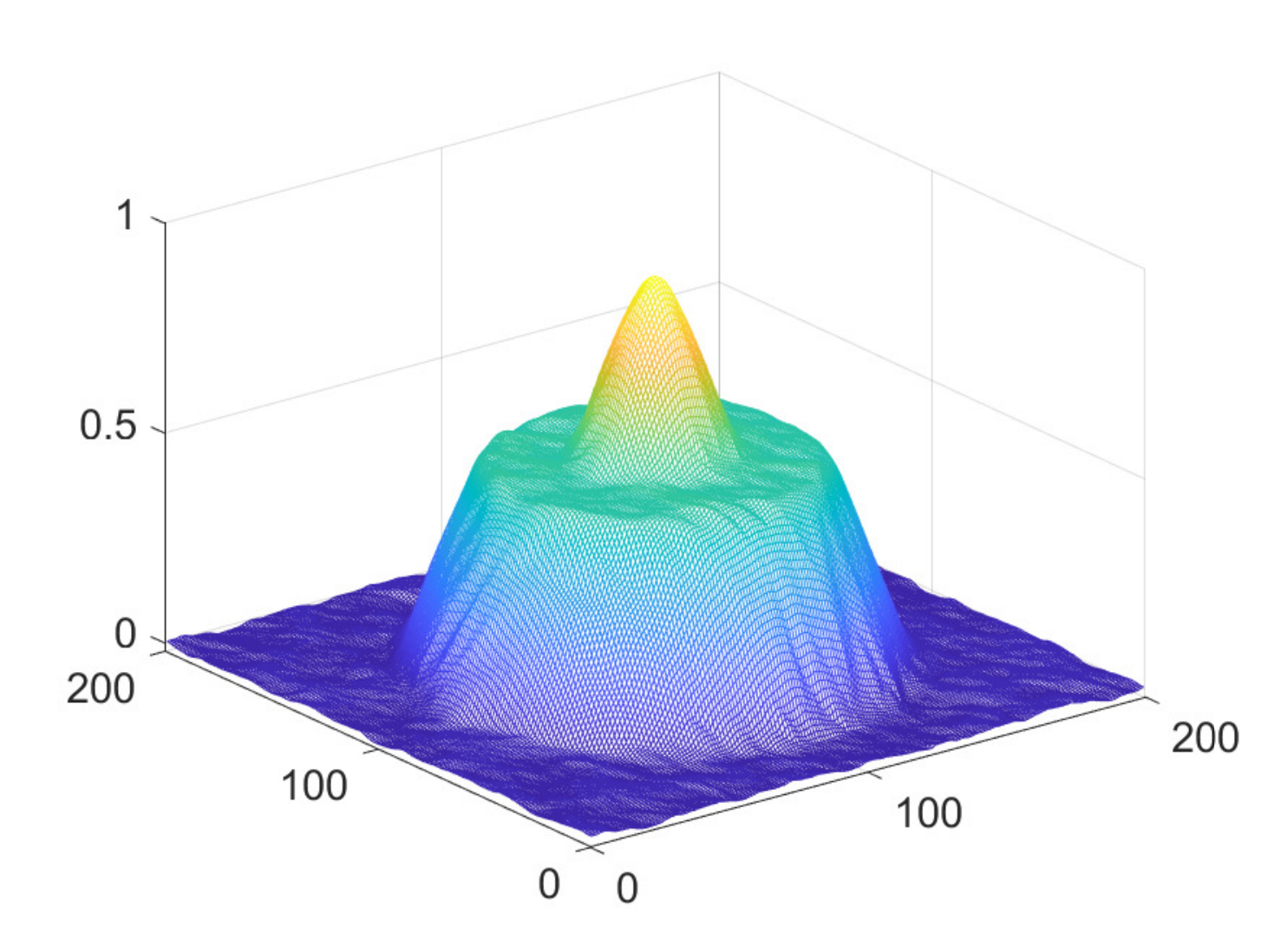}\\
		\includegraphics[trim={0.8cm 0 0.3cm 0},clip,width=0.21\textwidth]{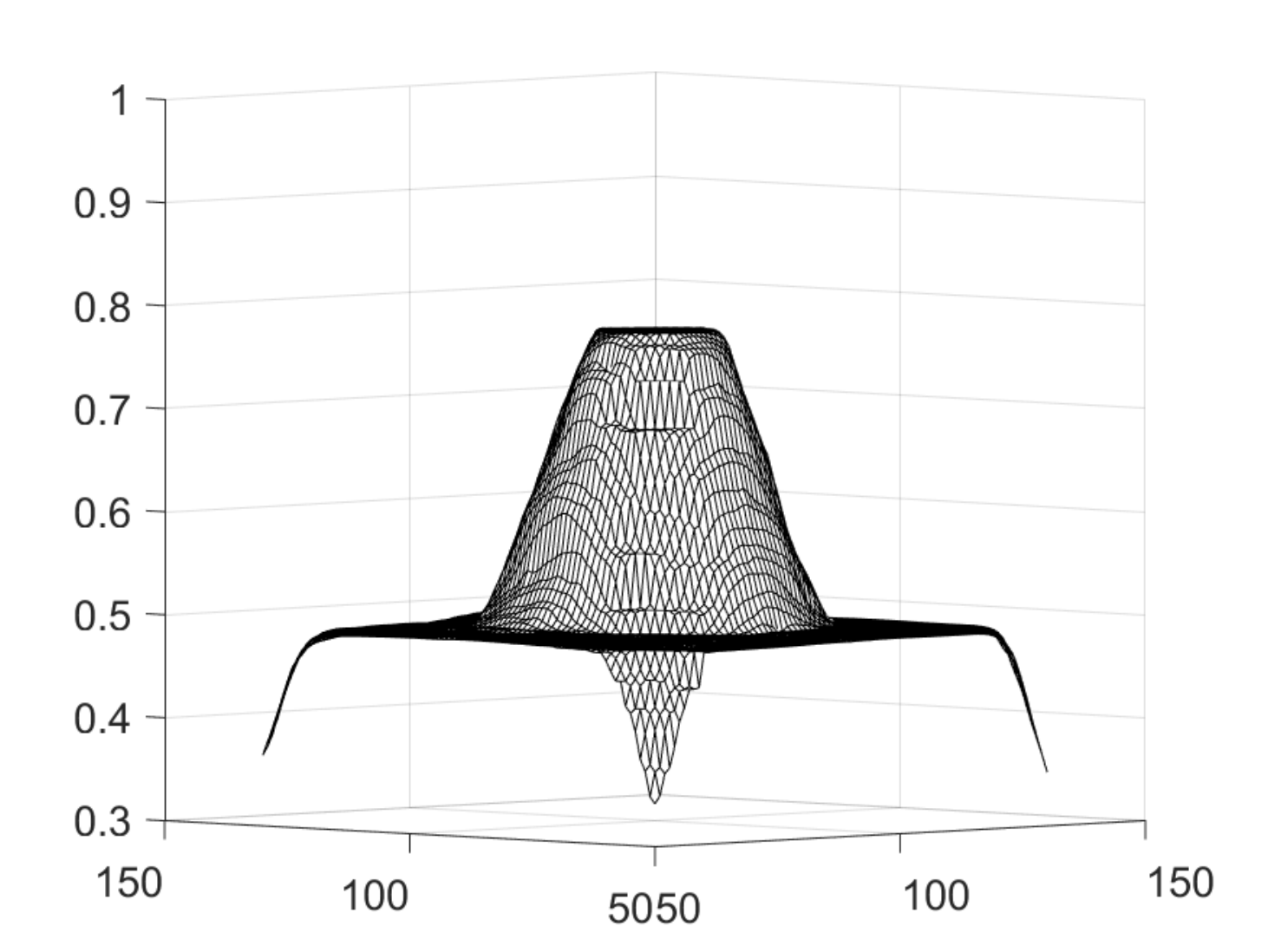}&
		\includegraphics[trim={0.8cm 0 0.3cm 0},clip,width=0.21\textwidth]{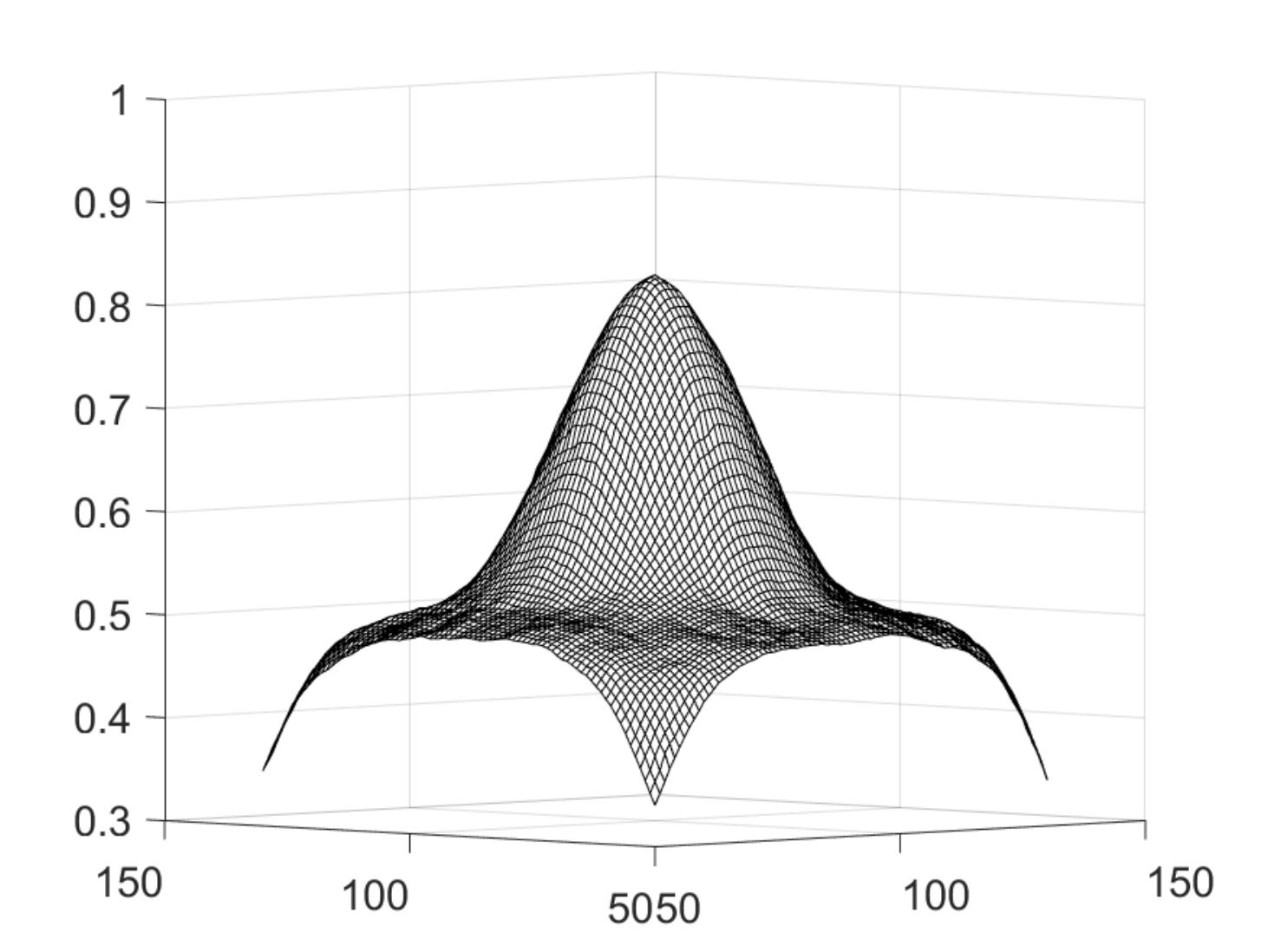}&
		\includegraphics[trim={0.8cm 0 0.3cm 0},clip,width=0.21\textwidth]{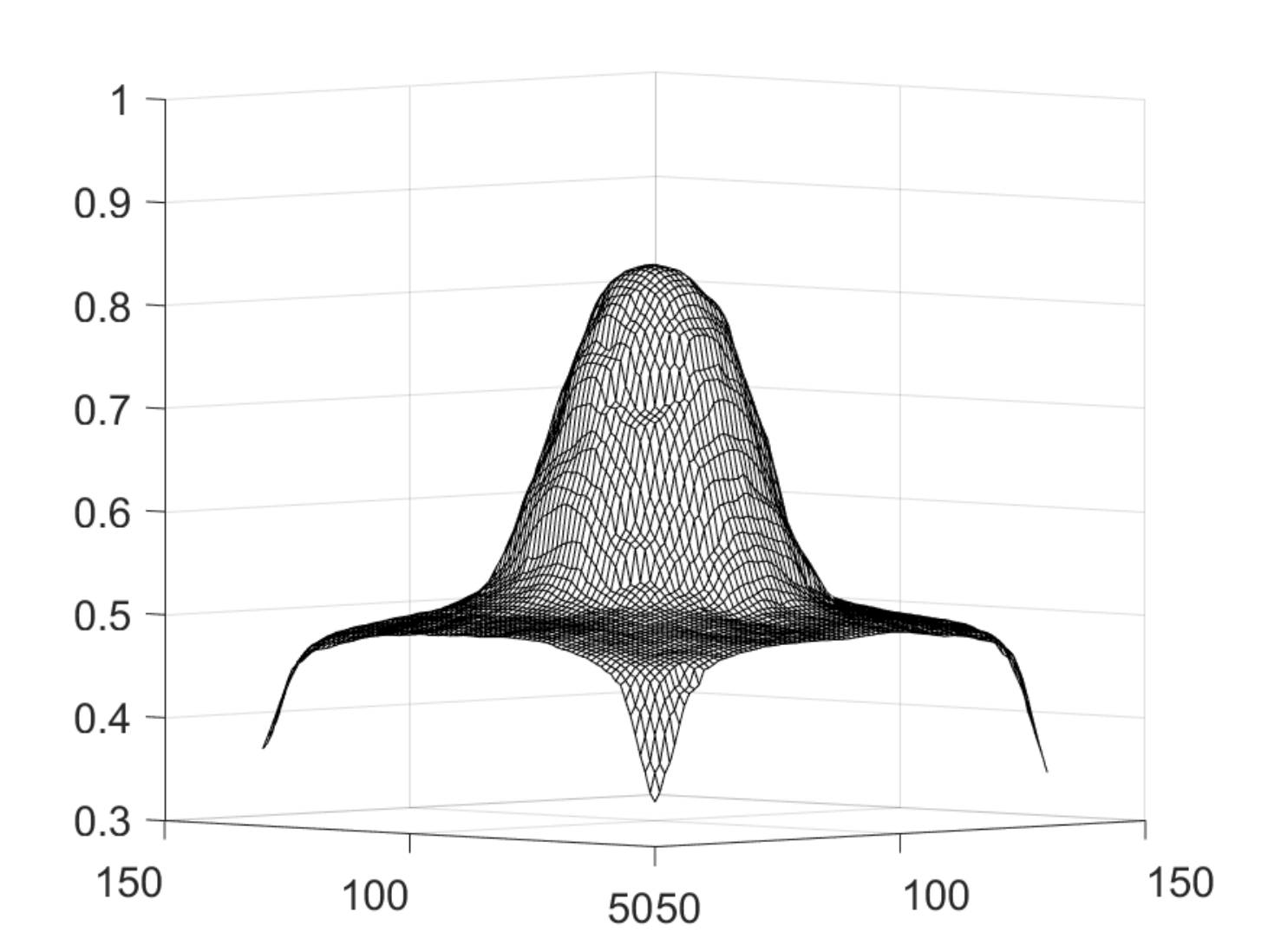}&
		\includegraphics[trim={0.8cm 0 0.3cm 0},clip,width=0.21\textwidth]{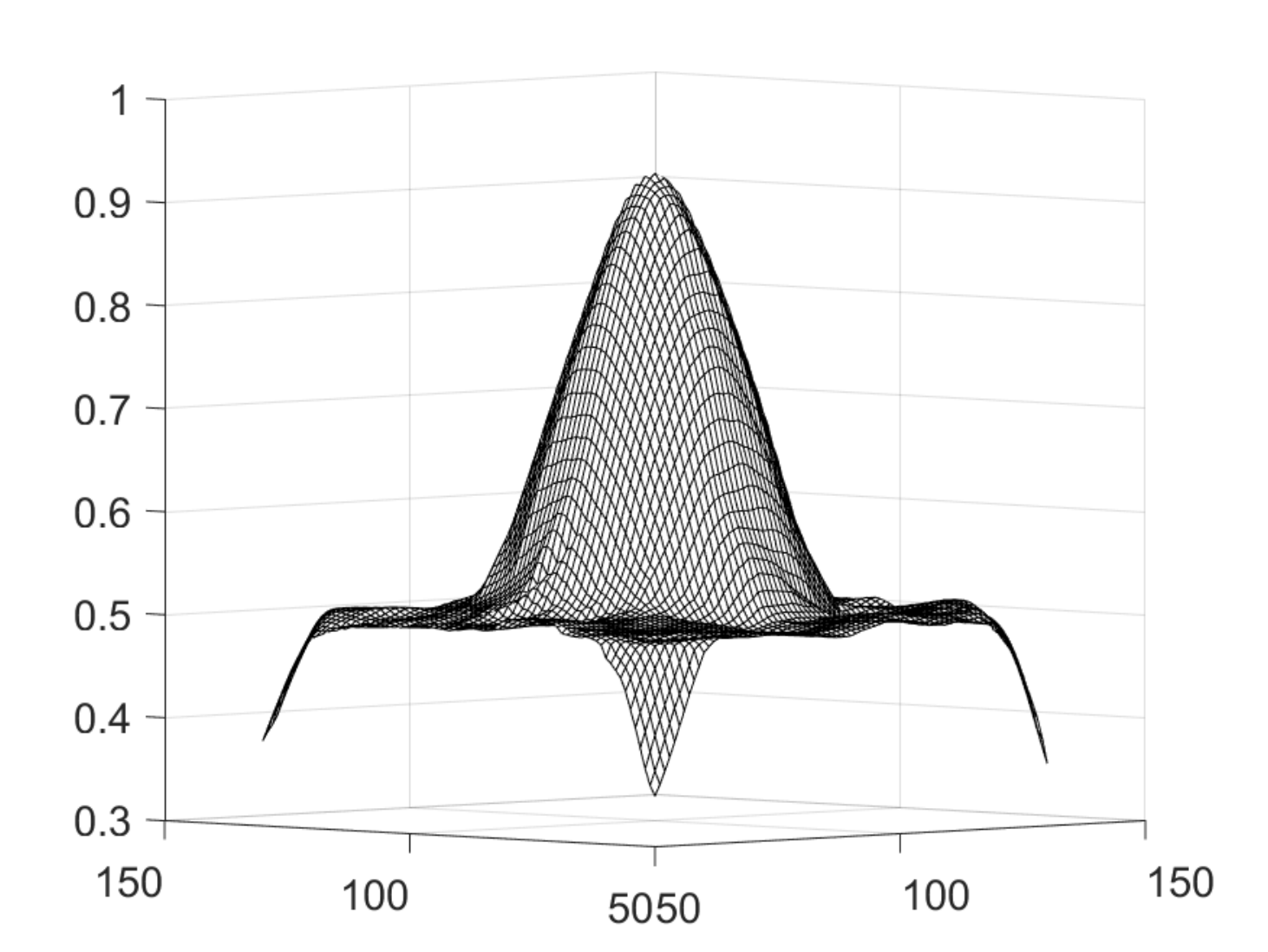}\\
		\includegraphics[trim={0.8cm 0 0.3cm 0},clip,width=0.21\textwidth]{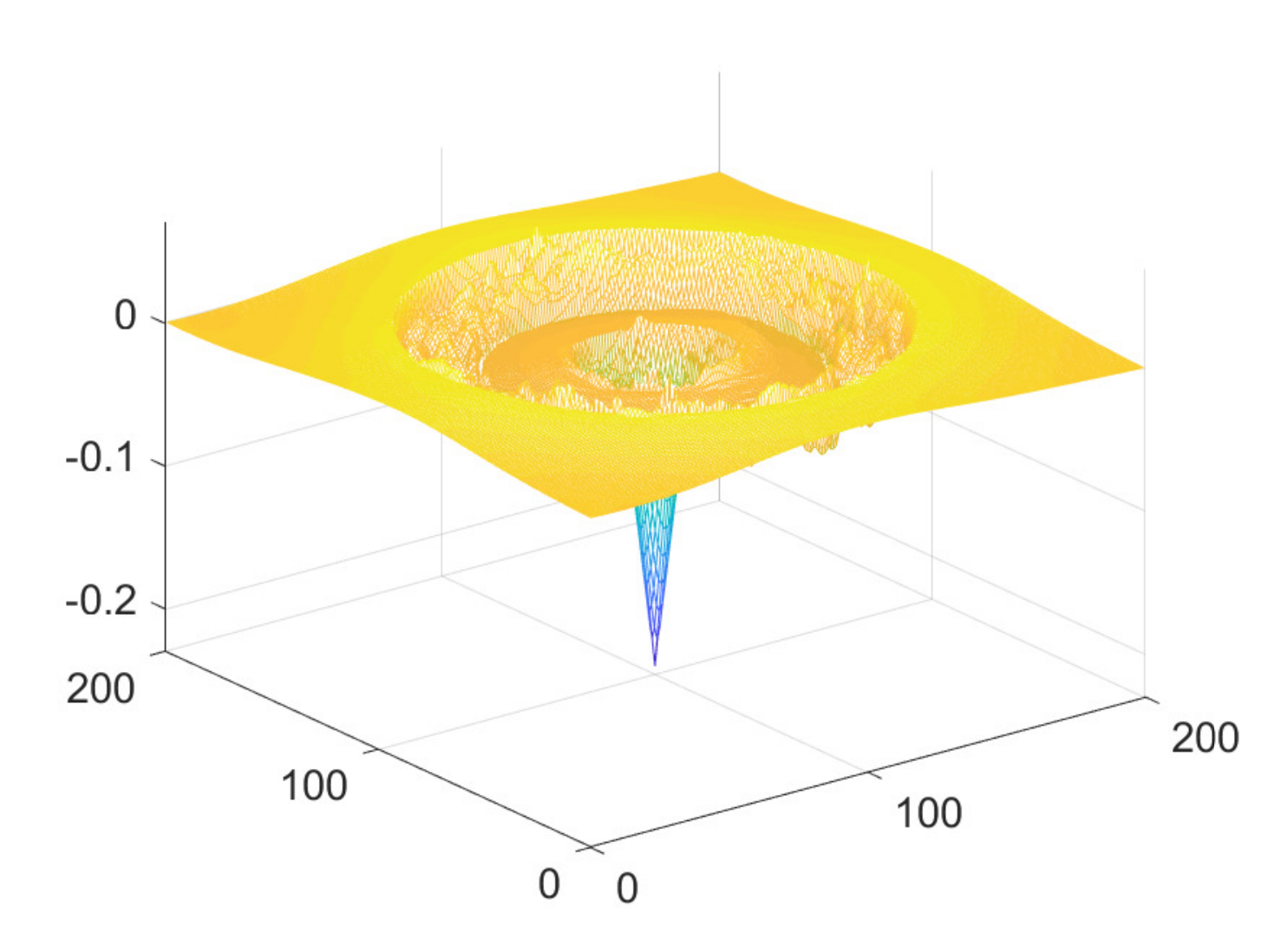}&
		\includegraphics[trim={0.8cm 0 0.3cm 0},clip,width=0.21\textwidth]{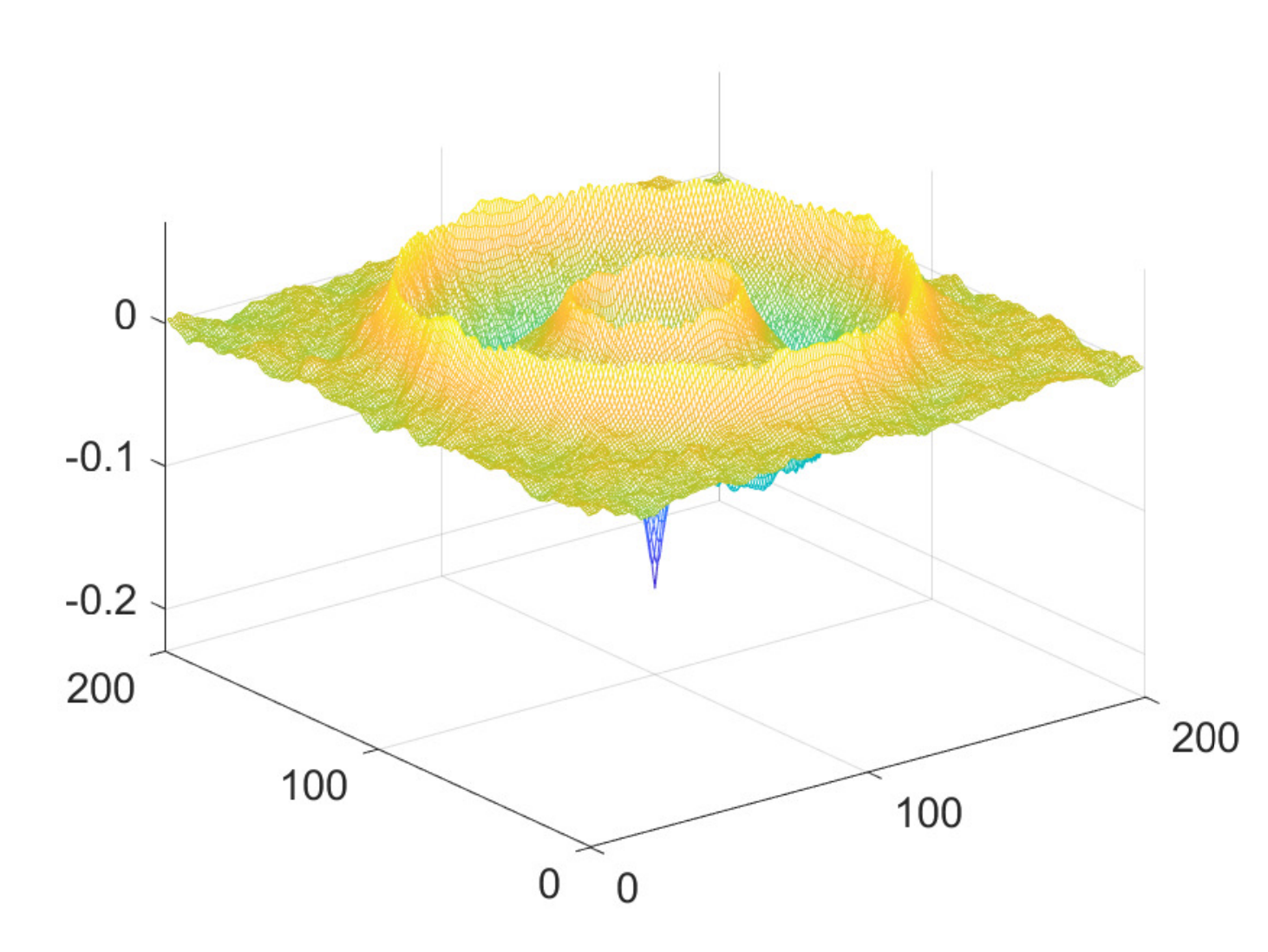}&
		\includegraphics[trim={0.8cm 0 0.3cm 0},clip,width=0.21\textwidth]{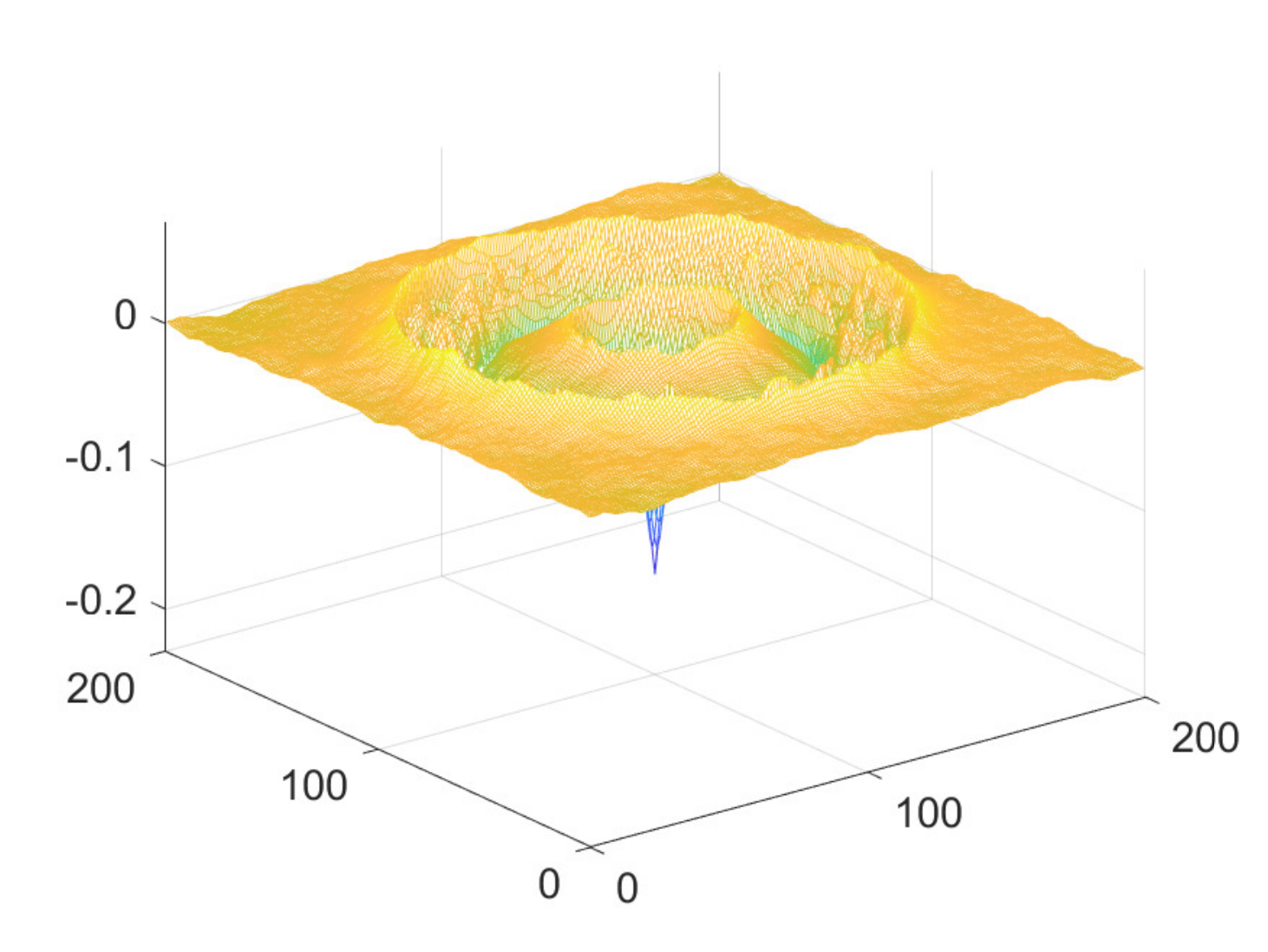}&
		\includegraphics[trim={0.8cm 0 0.3cm 0},clip,width=0.21\textwidth]{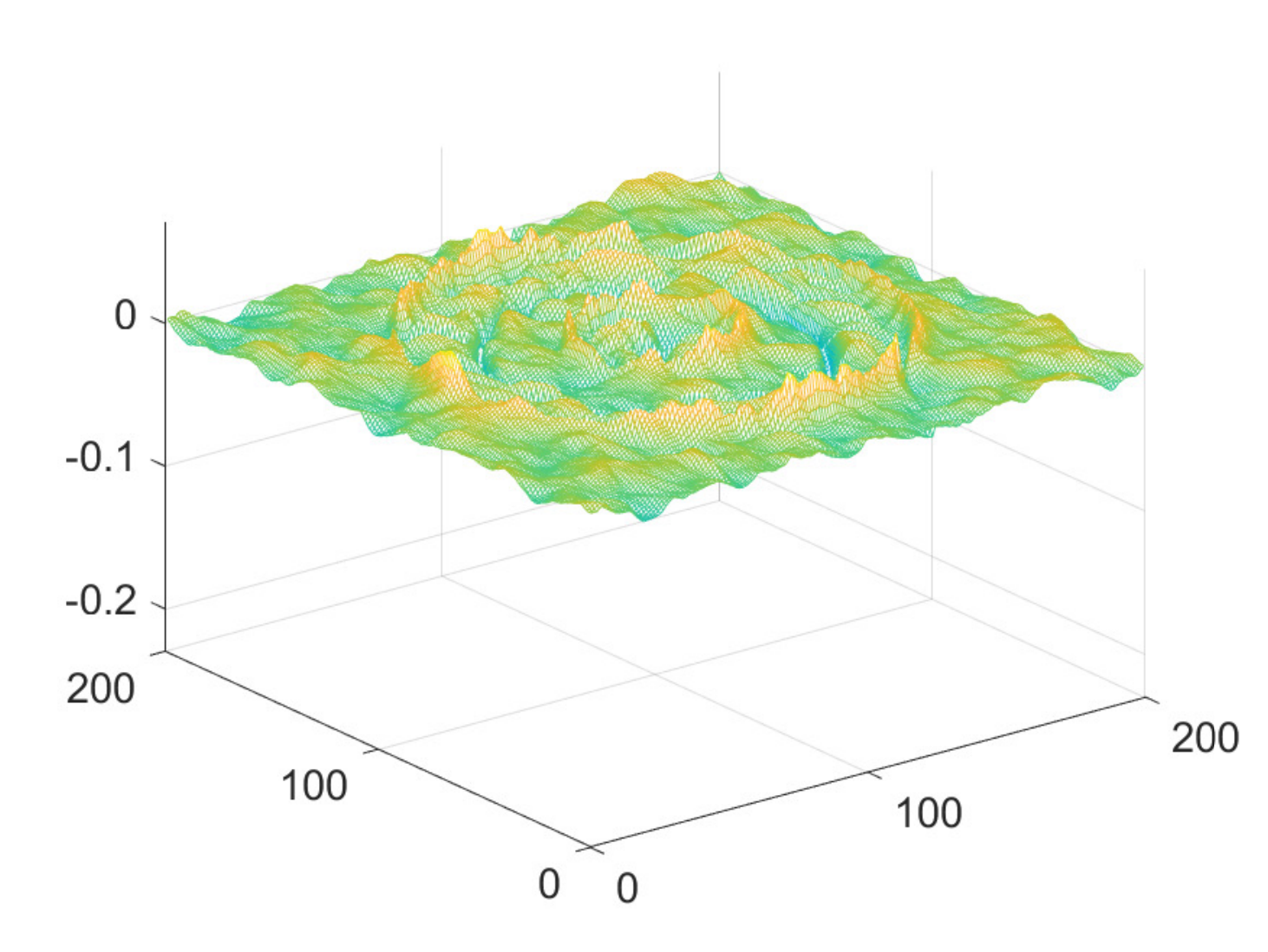}\\

	\end{tabular}
	\caption{(Comparison with other models on surface smoothing.) Comparison of the proposed model with the TV model and Euler's elastica model on smoothing a piecewise developable surface. (a) The clean surface. (c) The noisy surface with $\sigma=0.005$. (b) and (d) The central region of (a) and (c), respectively. (e) Results by the TV model with $\eta=0.5$ (the coefficient of the total variation term). (f) Results by Euler's elastica model with $a=b=0.4$. (g) Results by the proposed model with $\alpha=0.3,\beta=1$. (h) Results by the proposed model with $\alpha=5\times 10^{-5}$ and $\beta=10^3$. The second row shows the smoothed surfaces. The third row shows the plot of the central region. The forth row shows the graph of $u-f^*$. \vshrink%Staircase effect is observed in the result of the TV model. Euler's elastica model has a strong smoothing effect. The proposed model gives the best result with the least mismatch.
	}
	\label{fig.developable}
\end{figure}
\section{Numerical experiments}\label{sec.experiments}
We demonstrate the effectiveness of the proposed method through several experiments on surface smoothing and image denoising. All experiments are implemented in
MATLAB(R2018b) on a laptop of 8GB RAM and Intel Core i7-4270HQ CPU: 2.60 GHz. In our experiments, $\gamma=1$ and $h=1$ are used. For the scheme (\ref{eq.split.1.dis})-(\ref{eq.split.5.dis}), we adopt the initial condition (\ref{eq.initial1}) and stopping criterion on the relative error
%\begin{align}
$ \|u^{n+1}-u^n\|_{2}/\|u^{n+1}\|_{2}\leq tol$
%\end{align}
for some small $tol>0$. In this paper, without specification, $tol=10^{-5}$ is used, and the fixed point method (\ref{eq.p1.fix1})-(\ref{eq.p1.fix3}) is used to compute $\bp^{n+1/4}$.  When computing $\bp^{n+1/4}$ and $\bH^{n+1/4}$, we set $\xi_1=\xi_2=10^{-5}, \rho_1=\rho_2=0.8$ for (\ref{eq.p1.fix1})-(\ref{eq.p1.fix3}) and (\ref{eq.alterH.1})-(\ref{eq.alterH.end}). This article considers Gaussian noise whose magnitude is controlled by its variance, denoted by $\sigma$. Our code is available at the homepage of the first author\footnote{\url{https://www.math.hkbu.edu.hk/~haoliu/code.html}}.

\begin{remark}
	Although there are several parameters in the proposed method, these parameters can be adjusted easily and the performance of the method is not sensitive to their values. Specifically, $\xi_1$ are $\xi_2$ are stopping criteria of the iterative methods computing $\bp^{1+1/4}$ and $\bH^{1+1/4}$, respectively. $\rho_1$ and $\rho_2$ are parameters controlling the evolution speed of $\bq$ and $\bG$ when computing $\bp^{1+1/4}$ and $\bH^{1+1/4}$. Parameter $\gamma$ controls the evolution speed of $\bp$. The proposed method converges as long as these parameters are small enough.
\end{remark}

{
	\begin{remark}
		As discussed in Section \ref{sec.p2} and \ref{sec.H2}, the subiterations (\ref{eq.p1.fix1})--(\ref{eq.p1.fix3}) and (\ref{eq.alterH.1})--(\ref{eq.alterH.end}) are expected to fast converge  with initial guess $\bp^n$ and $\bH^n$. In all our experiments with the choice of parameters mentioned above, in each outer iteration, most of the subiterations only require less than 10 iterations to satisfy the stopping criterion.
	\end{remark}
}
\subsection{Surface smoothing}
The first problem we use to demonstrate the effectiveness of the proposed algorithm is surface smoothing. We consider the clean surfaces defined on a $200\times 200$ grid shown in Figure \ref{fig.surf}(a). The noisy surfaces are constructed by adding Gaussian noise with $\sigma=10^{-4}$ and are shown in (b). In our algorithm, we set $\alpha=1, \beta=0.1,\tau=0.01$ and $tol=10^{-5}$. We present the smoothed surfaces in Figure \ref{fig.surf}(c). The difference between the smoothed surfaces $u$ and the clean surfaces $f^*$ are shown in Figure \ref{fig.surf.err}(a). The smoothed surfaces are close to the clean surfaces with small mistaches. To demonstrate the efficiency of the proposed method, we present the histories of the energy and relative error with respect to the number of iterations in Figure \ref{fig.surf.err}(b) and (c), respectively.  For both examples, the energy achieves its minimum with about 40 iterations. Sublinear convergence is observed for the relative error.

We next compare the proposed model with the TV model \cite{rudin1992nonlinear,chambolle2011first} and Euler's elastica model \cite{deng2019new} for the smoothing of a developable surface shown in Figure \ref{fig.developable}(a). The noisy surface is shown in Figure \ref{fig.developable}(c). The plot of the central region of the clean and noisy surfaces are shown in (b) and (d), respectively.  In this set of experiments, we run the algorithm of each model until converge. The smoothed surfaces by the TV model with $\eta=0.5$ (the coefficient of the total variation term), Euler's elastica model with $a=b=0.4$  and the proposed model with $\alpha=0.3,\beta=1$ are shown in (e)-(g), respectively. Since the surface is developable, Gaussian curvature is a perfect regularizer.  Figure \ref{fig.developable}(h) presents the results by the proposed  model with $\alpha=5\times 10^{-5}$ and $\beta=10^3$. Under this setting, the Gaussian curvature dominates the proposed functional in (\ref{eq.model}). For better visualization of the difference, the graph of the central region of the smoothed surfaces and the difference $u-f^*$ are presented in the third and forth row, respectively. %Here $u$ denotes the smoothed surface and $f^*$ denotes the clean surface.
In this comparison, staircase effects are observed in the result by the TV model: the peak in the smoothed surface is flattened. While the peak is kept in the result of Euler's elastica model, it is smoothed a lot. The central flat region in this result is also smoothed and no longer flat. By the proposed model, the flat region is retained and the peak is recovered well. As shown in the forth row, the proposed model gives results with the smallest mismatch.
To quantify the difference $u-f^*$, we report the  errors $\|u-f^*\|_1$ and $\|u-f^*\|_{\infty}$ in Table \ref{tab.developable}. Results by the propose model give smaller errors than those of the other two models.  Under the choice of parameters in (h), the Gaussian curvature dominates the functional (\ref{eq.model}). The surface $f^*$ in this experiment is expected to be close to the global minimum of the functional since $f^*$ is piecewise developable. From Table \ref{tab.developable}, the result of (h)  has a very small $L^{\infty}$ error, i.e., it is very close to $f^*$. Therefore the proposed method provides a result that is close to the global minimum.

\begin{table}[th!]
	\centering
	\begin{tabular}{c|c|c|c|c}
		\hline
		Results in Figure \ref{fig.developable} & (e)&(f) & (g)& (h)\\
		\hline
		$\|u-f^*\|_1$ & 482.08 & 565.97 &345.35 &206.04\\
		\hline
		$\|u-f^*\|_{\infty}$ &0.2244&0.1701 &0.1600& 0.0717\\
		\hline
	\end{tabular}
	\caption{(Comparison with other models on surface smoothing.) Comparison of the errors $\|u-f^*\|_1$ and $\|u-f^*\|_{\infty}$ of results in Figure \ref{fig.developable}. (e) Result by the TV model. (f) Result by Euler's elastica model. (g)-(h) Results by the proposed model. \vshrink}
	\label{tab.developable}
\end{table}

\begin{figure}[t!]
	\begin{tabular}{ccc}
		\includegraphics[width=0.28\textwidth]{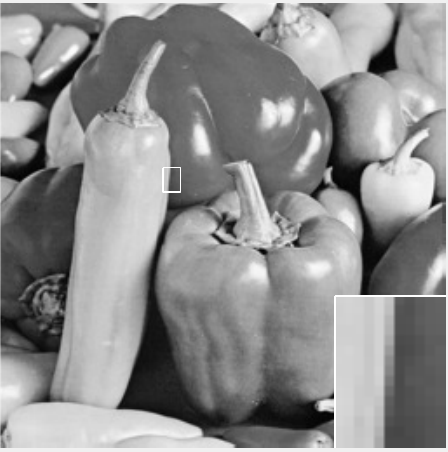}&
		\includegraphics[width=0.28\textwidth]{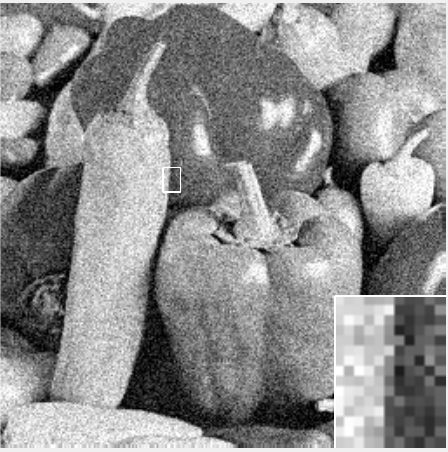}&
		\includegraphics[width=0.28\textwidth]{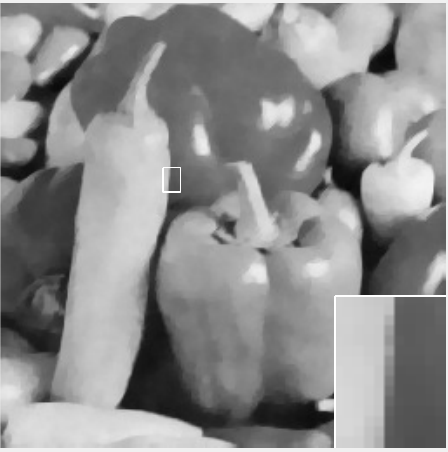}\\
		\includegraphics[width=0.28\textwidth]{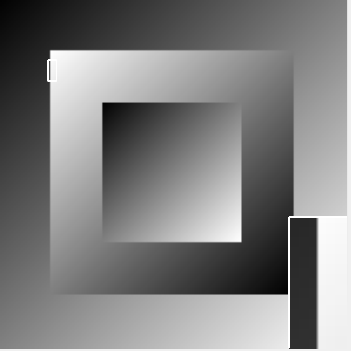}&
		\includegraphics[width=0.28\textwidth]{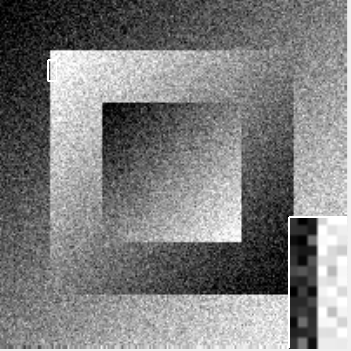}&
		\includegraphics[width=0.28\textwidth]{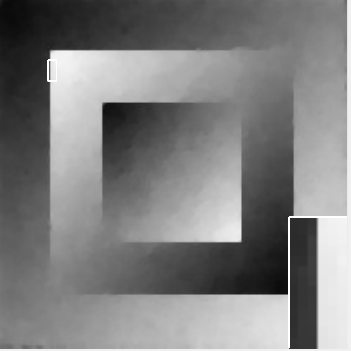}
	\end{tabular}
	\caption{(Gaussian noise with $\sigma=0.01$.) Denoised images by the proposed model with $\alpha=0.2,\beta=0.6$. First column: Clean images. Second column: Noisy images. Third column: Denoised images. \vshrink}
	\label{fig.G001}
\end{figure}

\begin{figure}[t!]
	\begin{tabular}{cccc}
		(a) & (b) & (c) &(d)\\
		\includegraphics[width=0.21\textwidth]{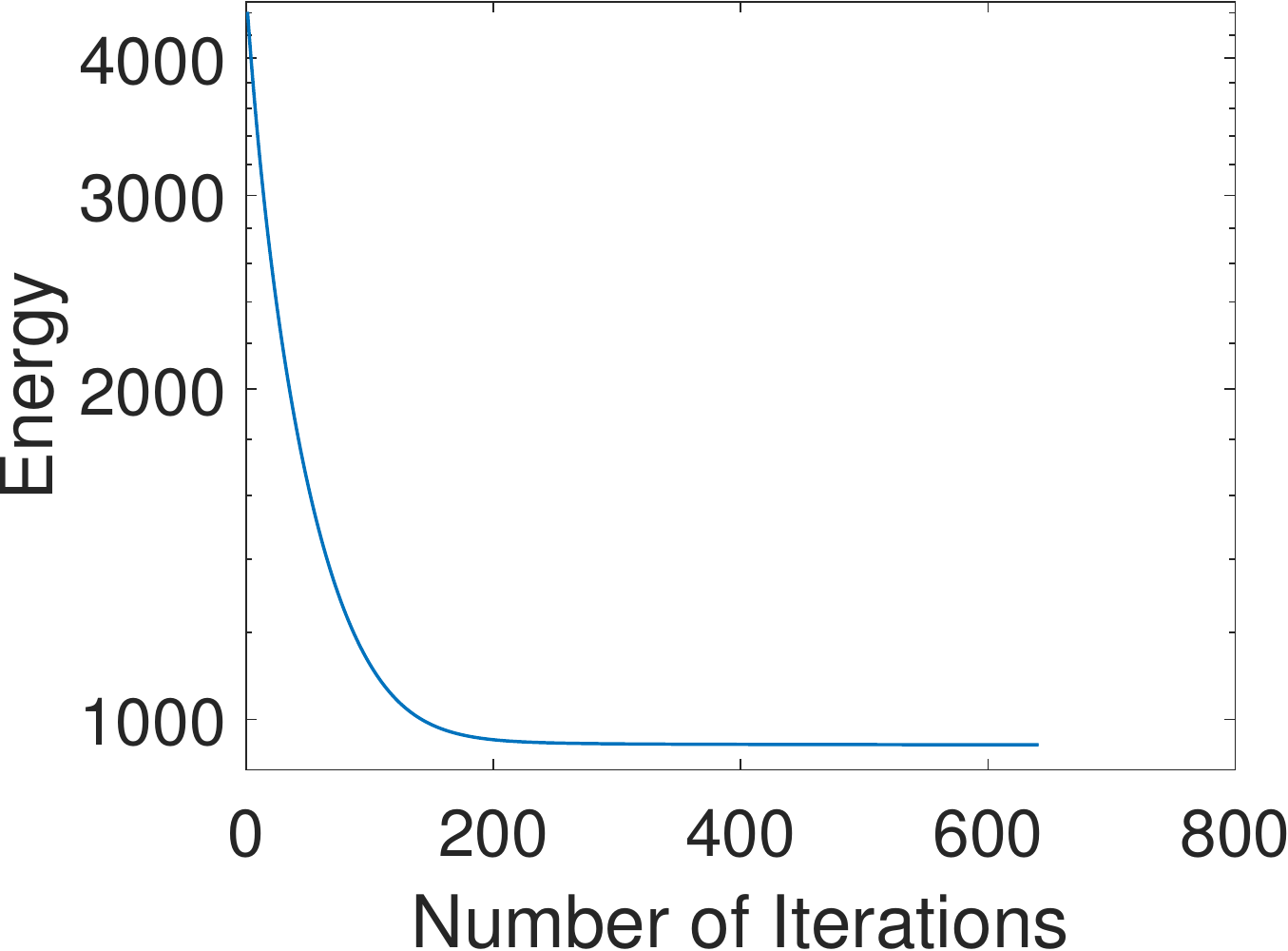}&
		\includegraphics[width=0.21\textwidth]{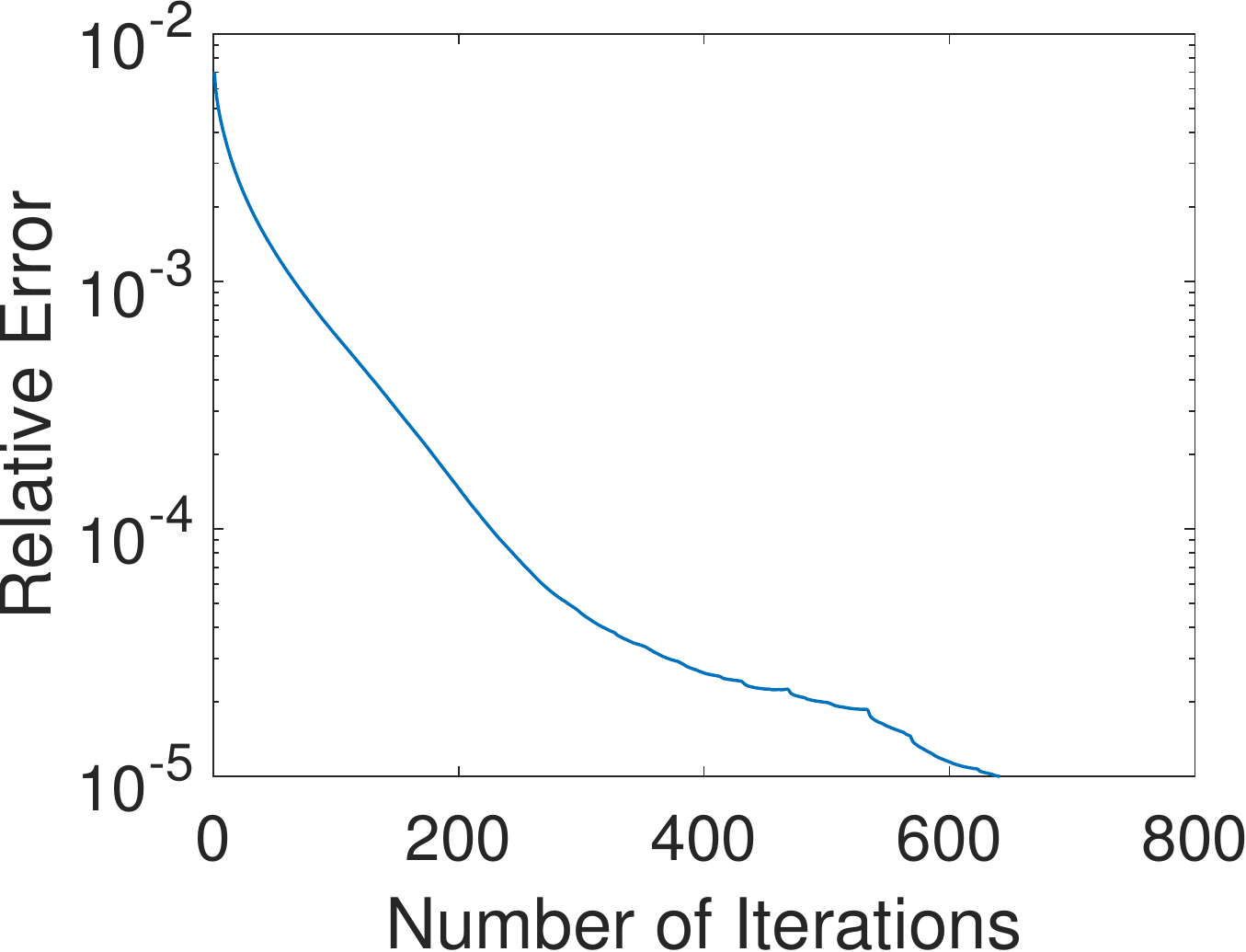}&
		\includegraphics[width=0.21\textwidth]{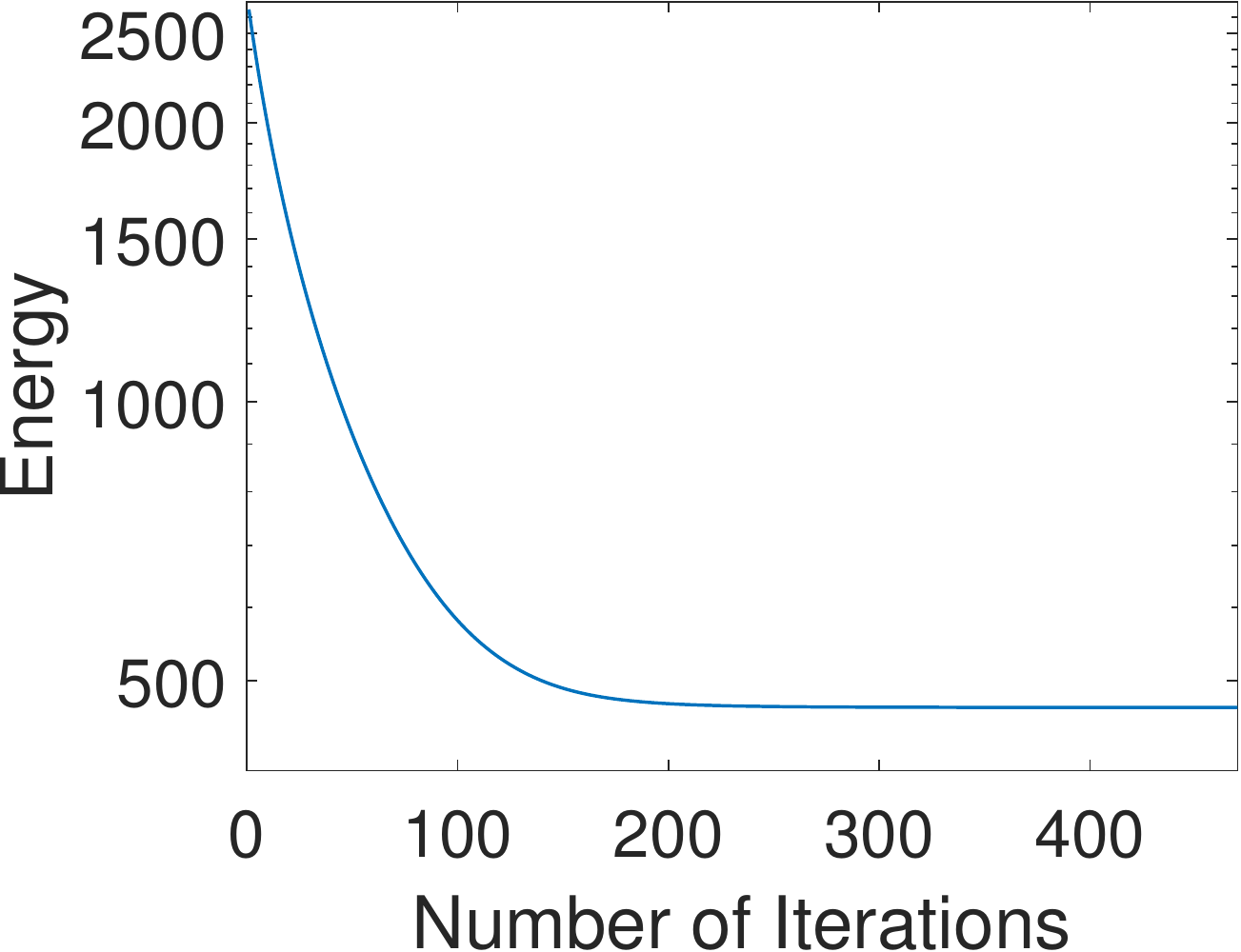}&
		\includegraphics[width=0.21\textwidth]{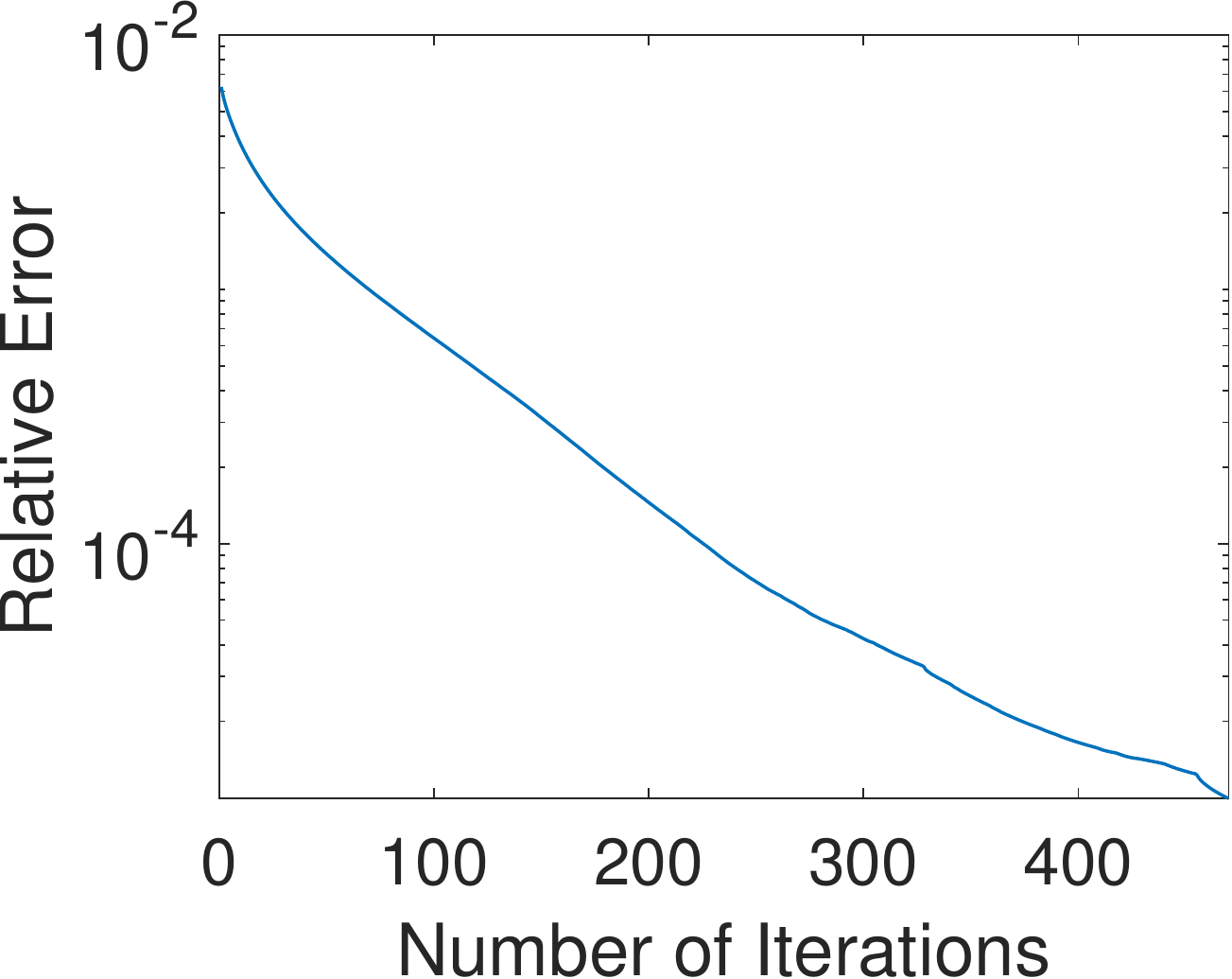}
	\end{tabular}
	\caption{(Gaussian noise with $\sigma=0.01$.) Histories of ((a) and (c)) the  energy and ((b) and (d)) the relative error of results in Figure \ref{fig.G001}. Here (a)-(b) and (c)-(d) correspond to the results in the first row and second row of Figure \ref{fig.G001}, respectively. \vshrink}
	\label{fig.G001.E}
\end{figure}

\begin{table}[th!]
	\centering
	\begin{tabular}{c|c|c|c|c||c}
		\hline
		& $\tau=0.5$& $\tau=0.1$& $\tau=0.05$ & $\tau=0.01$ & CPU time per iter. \\
		\hline
		Newton& 3.16 & 2.41 &2.07 & 2 & $1.25\times 10^{-2}$\\
		\hline
		Fixed point& 6.42 & 5.12 &5.08 & 5 &$7.01\times 10^{-3}$ \\
		\hline
	\end{tabular}
	\caption{ (Comparison of the efficiency of Newton's method and the fixed point method when computing $\bp^{n+1/4}$.) We take the image in the second row of Figure \ref{fig.G001} as an example. Column 2--5 show the averaged number of iterations used in Newton's method (\ref{eq.p1.newton}) and the fixed point method (\ref{eq.p1.fix1})--(\ref{eq.p1.fix3}) when computing $\bp^{n+1/4}$ per outer iteration. Column 6 shows the CPU time per iteration in Newton's method and the fixed point method. \vshrink}
	\label{tab.newtonfix}
\end{table}

\begin{table}[th!]
	\centering
	\begin{tabular}{c|c|c|c||c|c}
		\hline
		Image size $p$& Num. of Iter.& Total& Order & $\bp^{n+1/4}$ & $\bH^{n+1/4}$\\
		\hline
		50 & 505 & 1.32 &-- & 0.29 & 0.42 \\
		\hline
		100 & 559 & 3.96 & 1.58 & 1.17 & 0.99\\
		\hline
		150 & 630 & 8.00 & 1.73 & 2.53 & 2.20\\
		\hline
		200 & 498 & 9.67 & 0.66 & 3.53 & 2.35\\
		\hline
		250 & 477 & 13.64 & 1.54 &4.40 & 3.75\\
		\hline
		300 & 479 & 20.67 & 2.28& 6.99 & 5.31\\
		\hline
	\end{tabular}
	\caption{(Computational complexity with respect to image size.) Number of iterations and CPU time in seconds required to satisfy the stopping criterion with image size $p\times p$ for $p=50,100,150,200,250,300$. We take the image in the second row of Figure \ref{fig.G001} as an example. Column 1: Image size $p$. Column 2: Number of iterations. Column 3: Total CPU time. Column 4: Power order of total CPU time in terms of $p$. Column 5: CPU time used to compute the subiteration (\ref{eq.p1.fix1})--(\ref{eq.p1.fix3}) for $\bp^{n+1/4}$. Column 6: CPU time used to compute the subiteration (\ref{eq.alterH.1})--(\ref{eq.alterH.end}) for $\bH^{n+1/4}$.\vshrink}
	\label{tab.size}
\end{table}

\begin{figure}[th!]
	\begin{tabular}{cccc}
		(a) & (b) & (c) & (d)\\
		\includegraphics[width=0.21\textwidth]{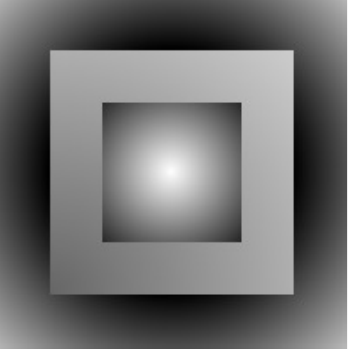}&
		\includegraphics[trim={0.8cm 0 0.3cm 0},clip,width=0.21\textwidth]{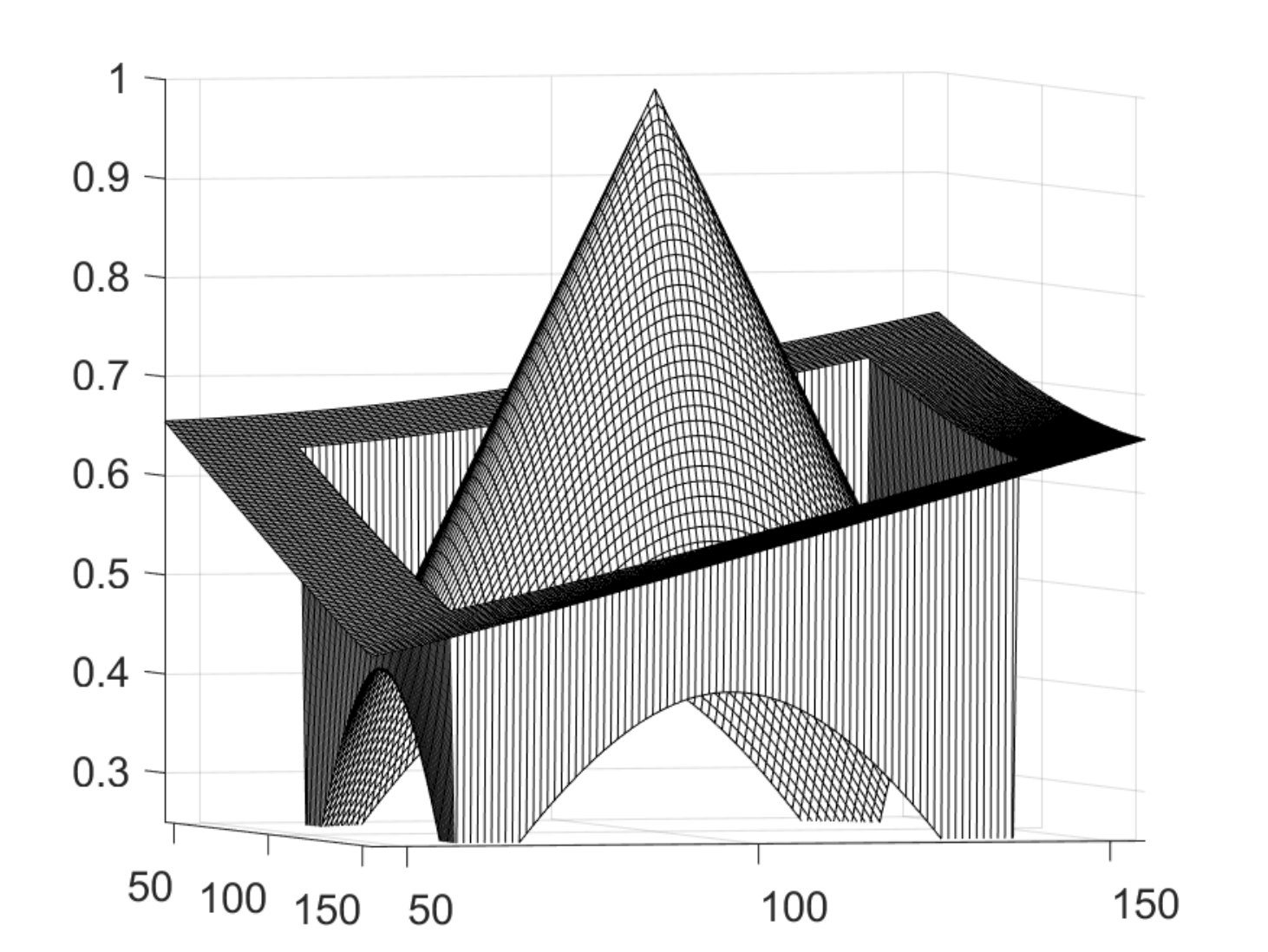}&
		\includegraphics[width=0.21\textwidth]{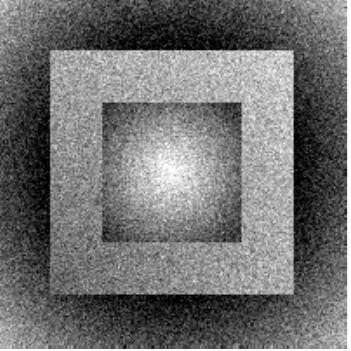}&
		\includegraphics[trim={0.8cm 0 0.3cm 0}, clip,width=0.21\textwidth]{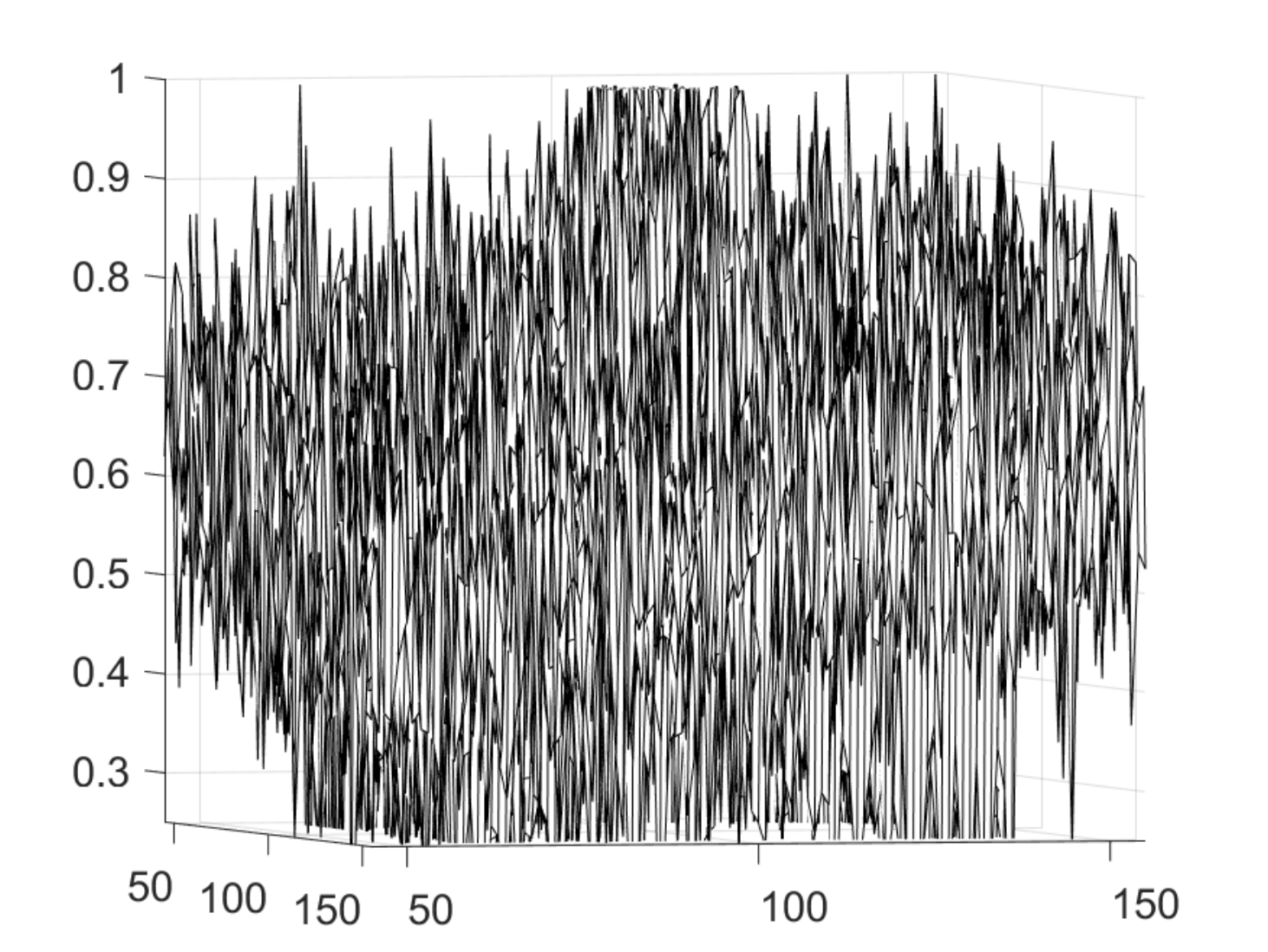}\\	
	\end{tabular}
	\begin{tabular}{ccc}
		(e) & (f) & (g) \\
		\includegraphics[width=0.3\textwidth]{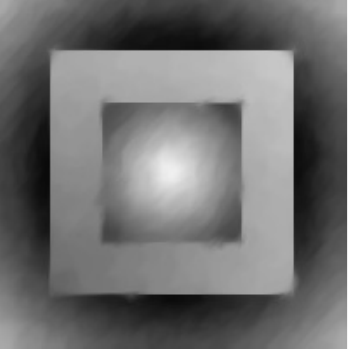}&
		\includegraphics[width=0.3\textwidth]{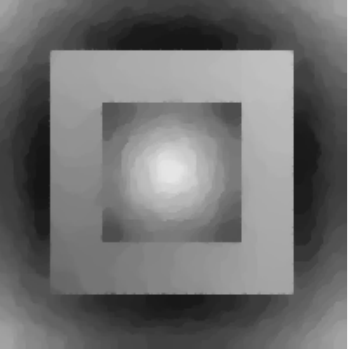}&
		\includegraphics[width=0.3\textwidth]{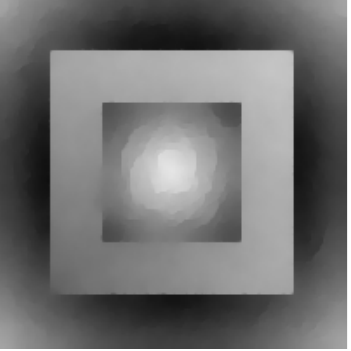}\\
		\includegraphics[width=0.3\textwidth]{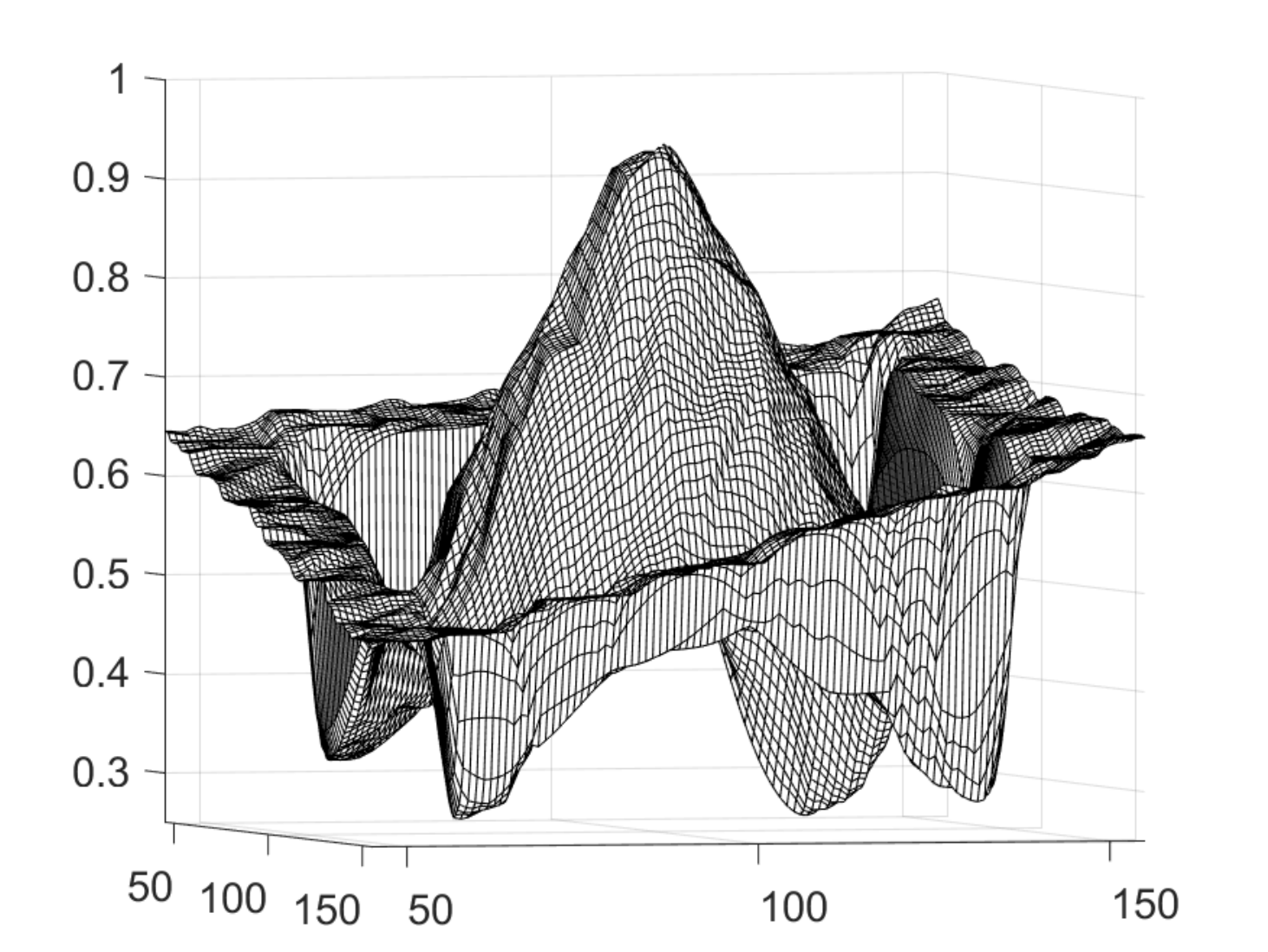}&
		\includegraphics[width=0.3\textwidth]{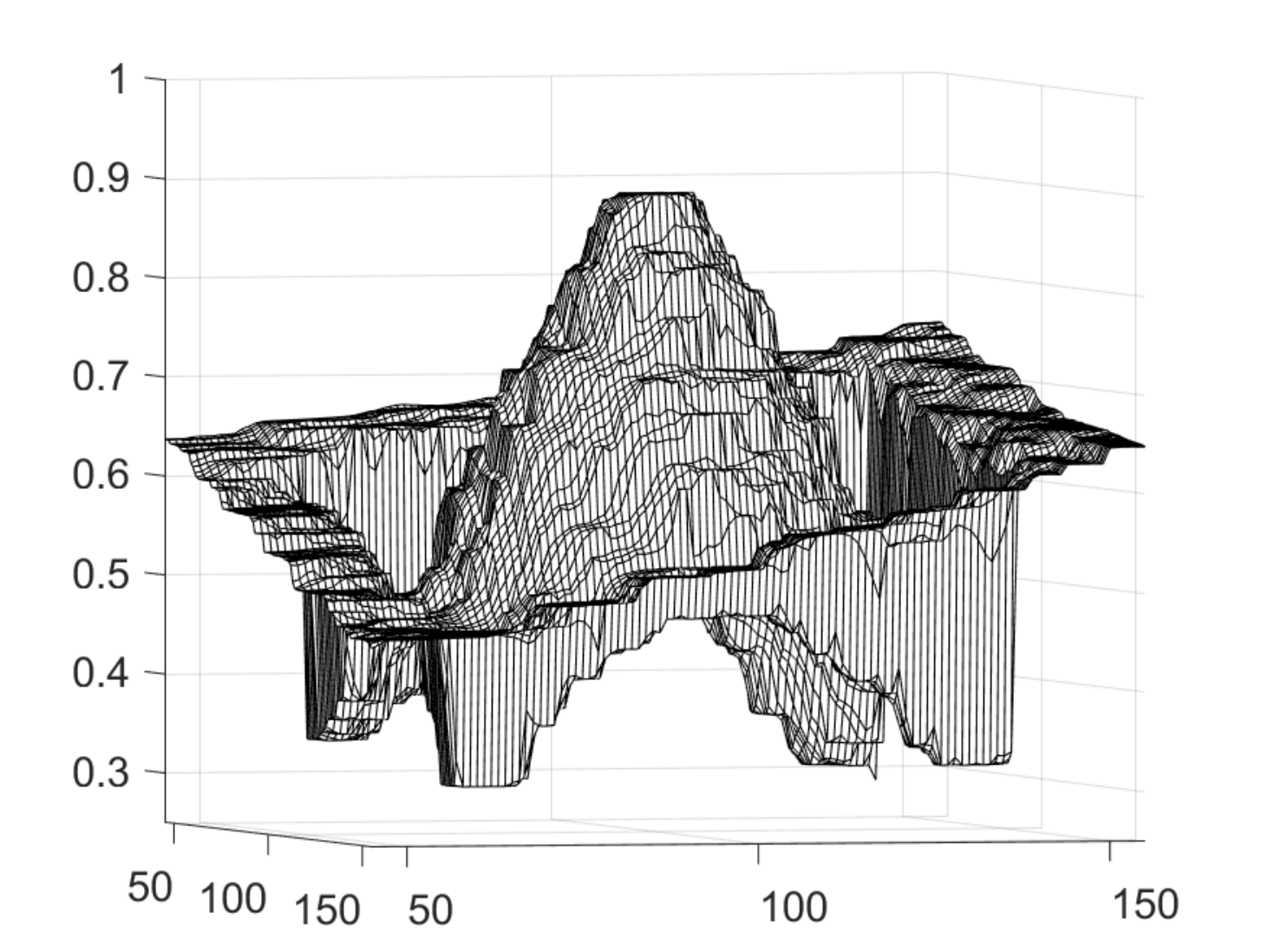}&
		\includegraphics[width=0.3\textwidth]{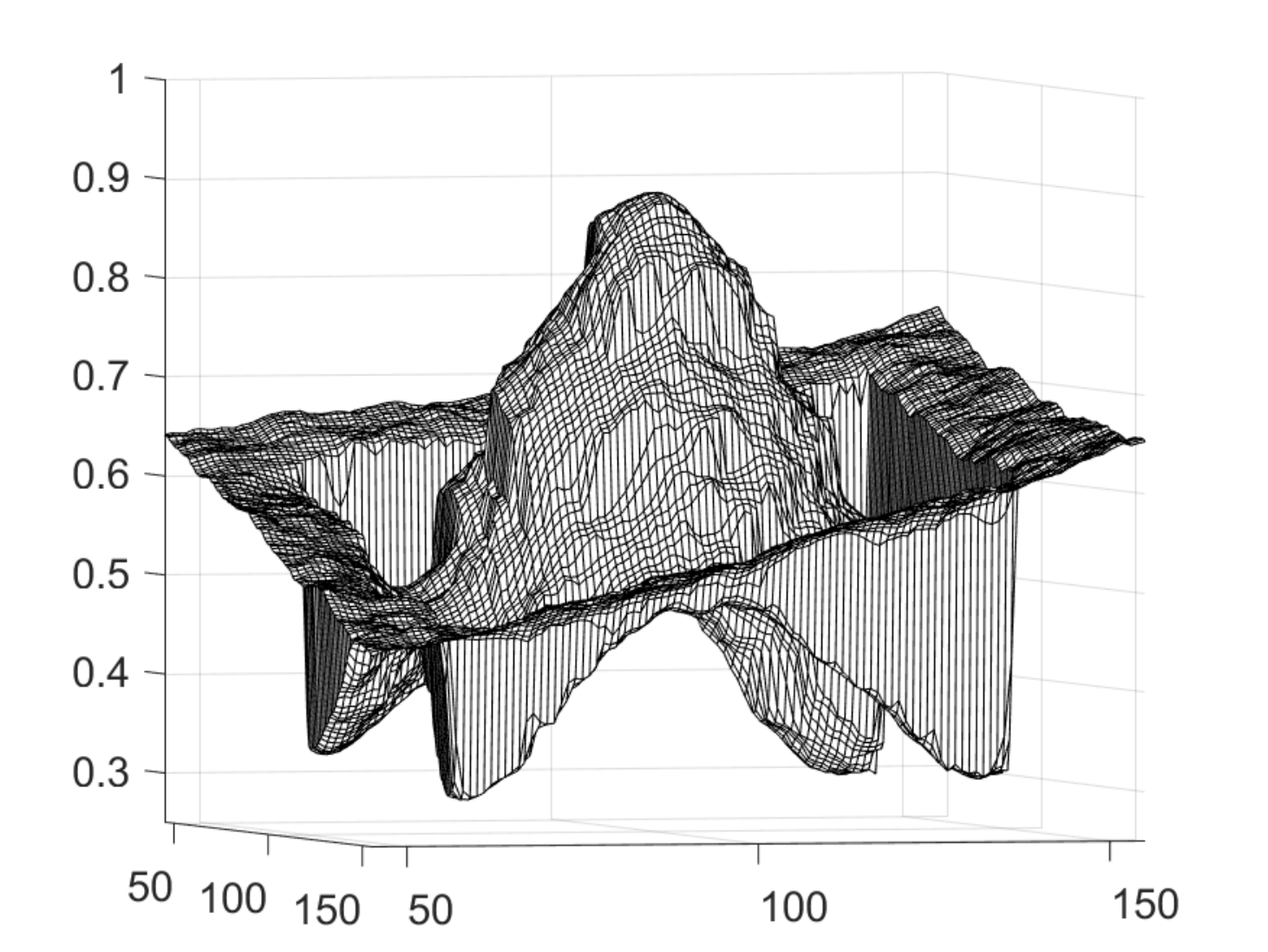}
	\end{tabular}
	\caption{(Recovering cone-shape objects.) Comparison of the proposed model with the TV model and Euler's elastica model on denoising an image whose graph contains a cone-shape object. (a) The clean image. (c) The noisy image with Gaussian noise and $\sigma=0.01$. (b) and (d) Surface plot of the central region of (a) and (c), respectively. The second and third row show the denoised images and the surface plot of their central regions by (e) the proposed model with $\alpha=0.002, \beta=40$, (f) the TV model with $\eta=0.2$, (g) Euler's elastica model with $a=b=0.15$.\vshrink
	}
	\label{fig.im.developable}
\end{figure}

\begin{figure}[th!]
	\begin{tabular}{cccc}
		(a) & (b) & (c) & (d)\\
		\includegraphics[width=0.21\textwidth]{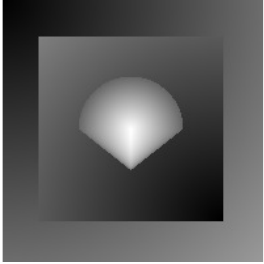}&
		\includegraphics[trim={0.8cm 0 0.3cm 0},clip,width=0.21\textwidth]{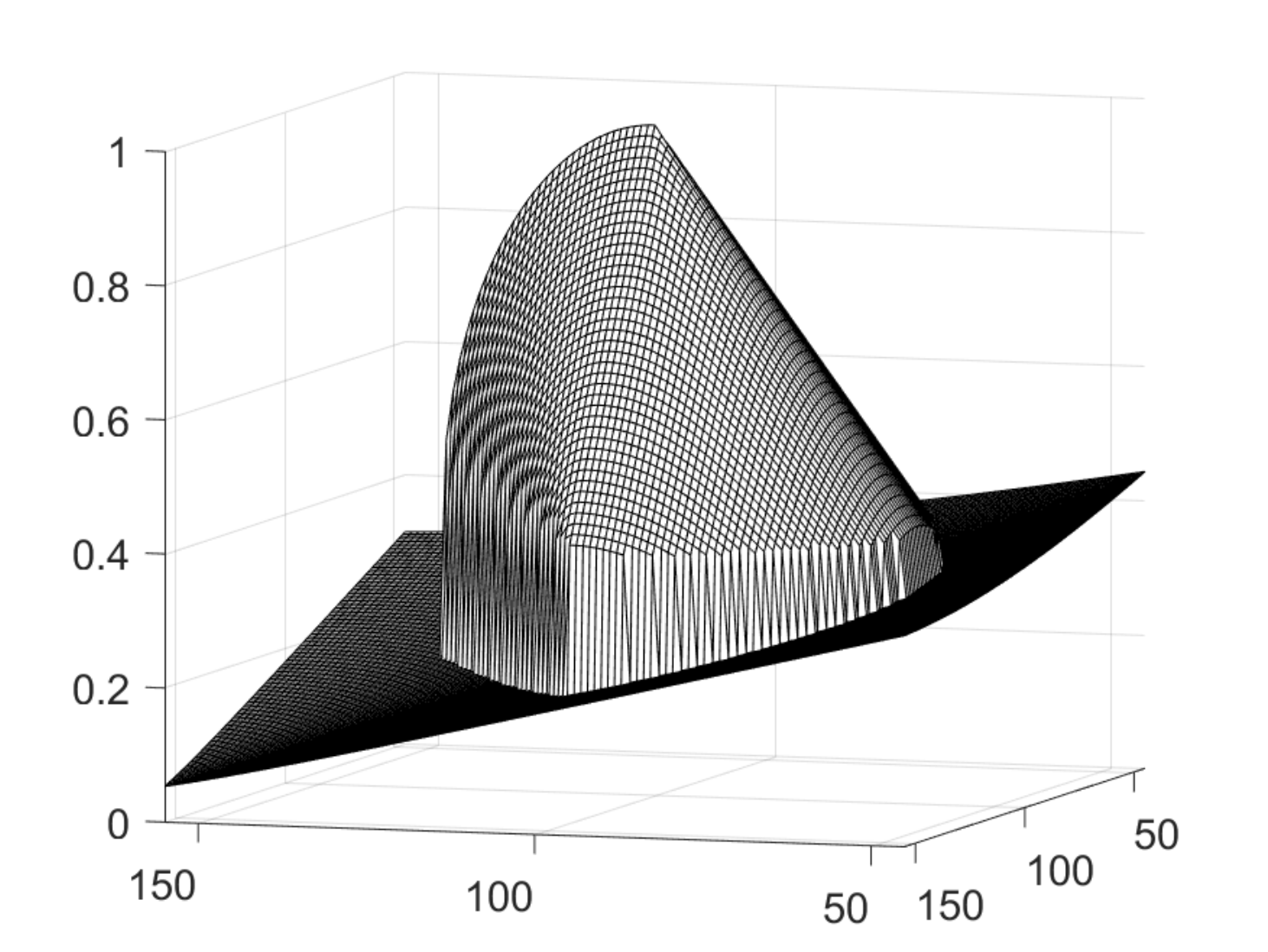}&
		\includegraphics[width=0.21\textwidth]{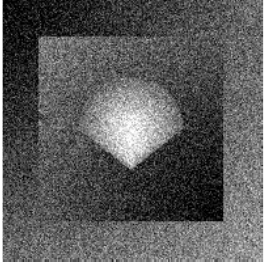}&
		\includegraphics[trim={0.8cm 0 0.3cm 0},clip,width=0.21\textwidth]{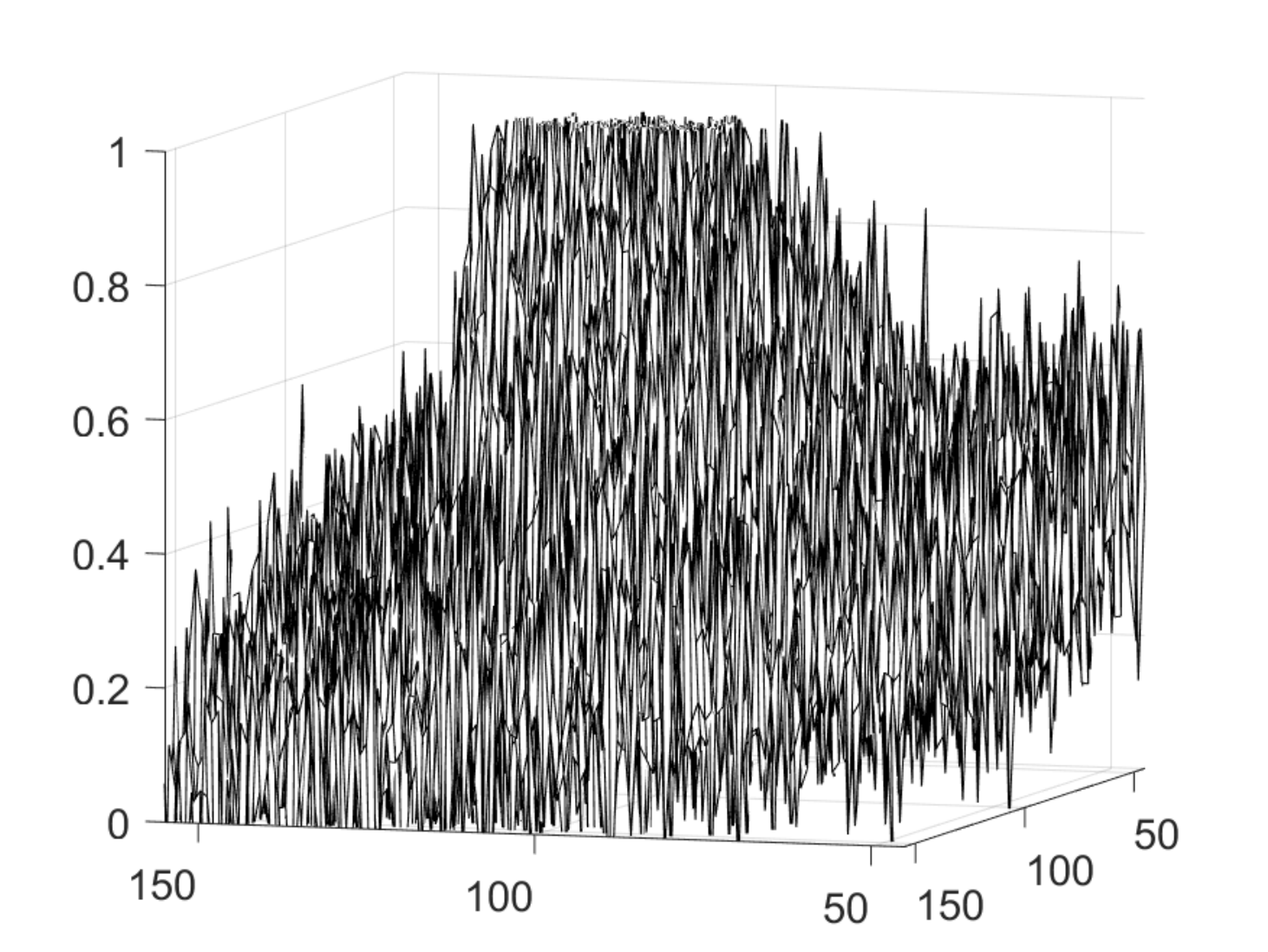}\\	
	\end{tabular}
	\begin{tabular}{ccc}
		(e) & (f) & (g) \\
		\includegraphics[width=0.3\textwidth]{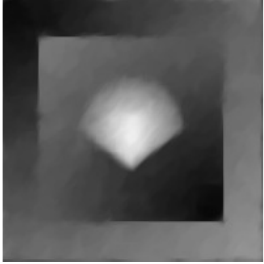}&
		\includegraphics[width=0.3\textwidth]{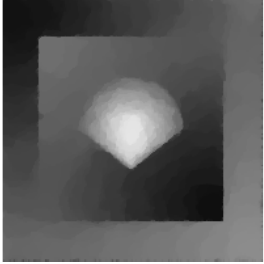}&
		\includegraphics[width=0.3\textwidth]{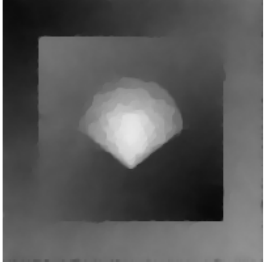}\\
		\includegraphics[width=0.3\textwidth]{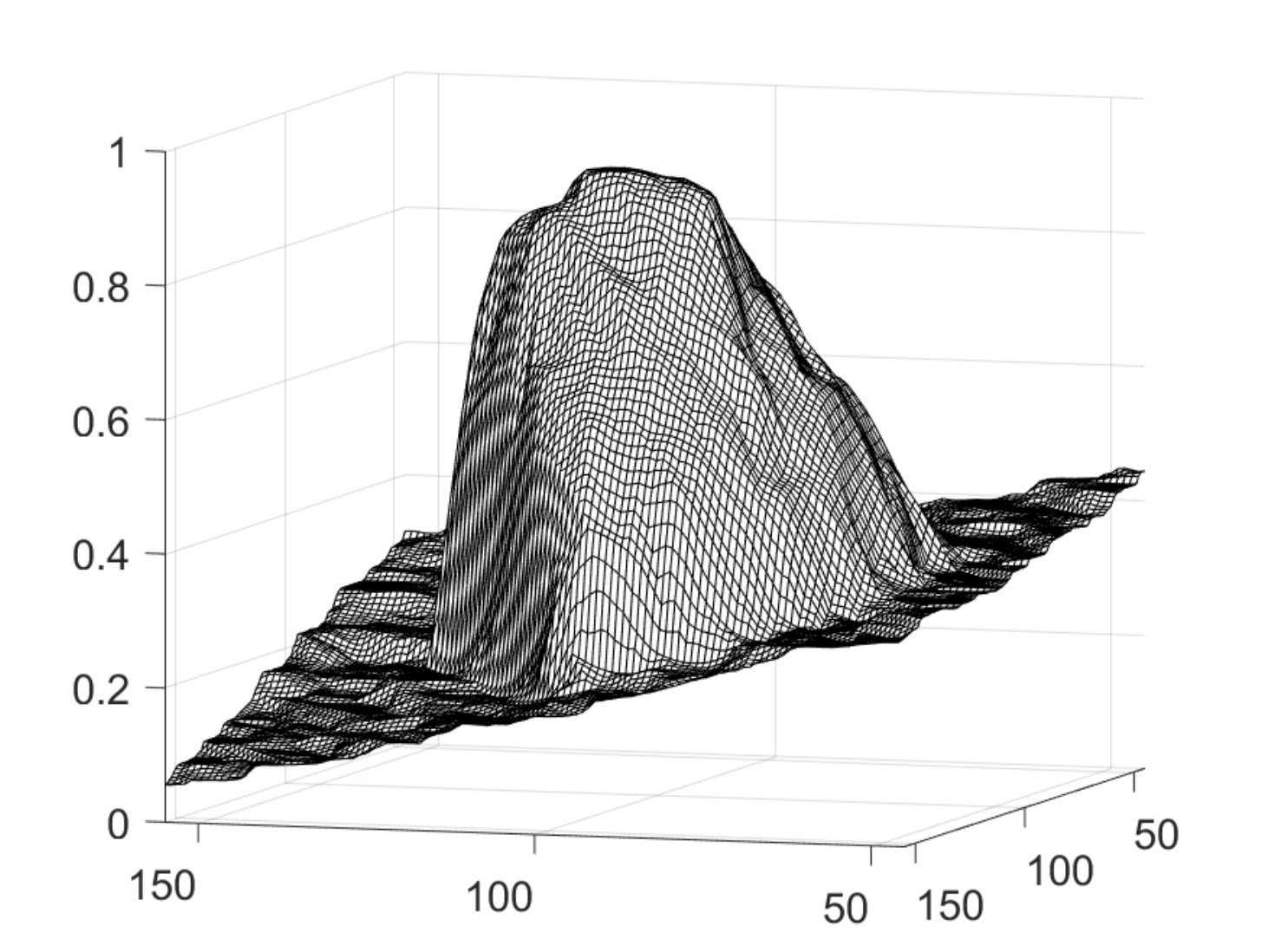}&
		\includegraphics[width=0.3\textwidth]{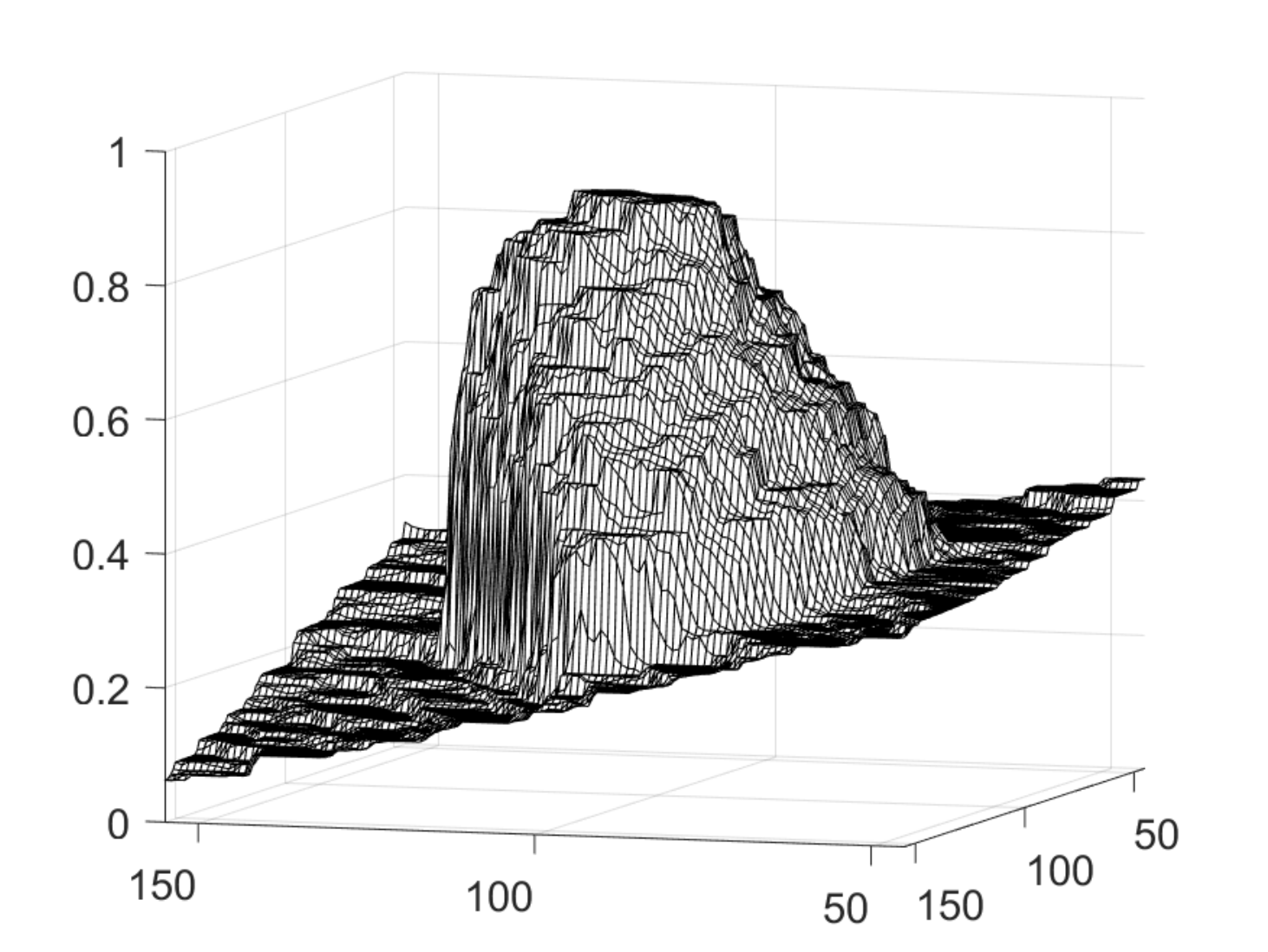}&
		\includegraphics[width=0.3\textwidth]{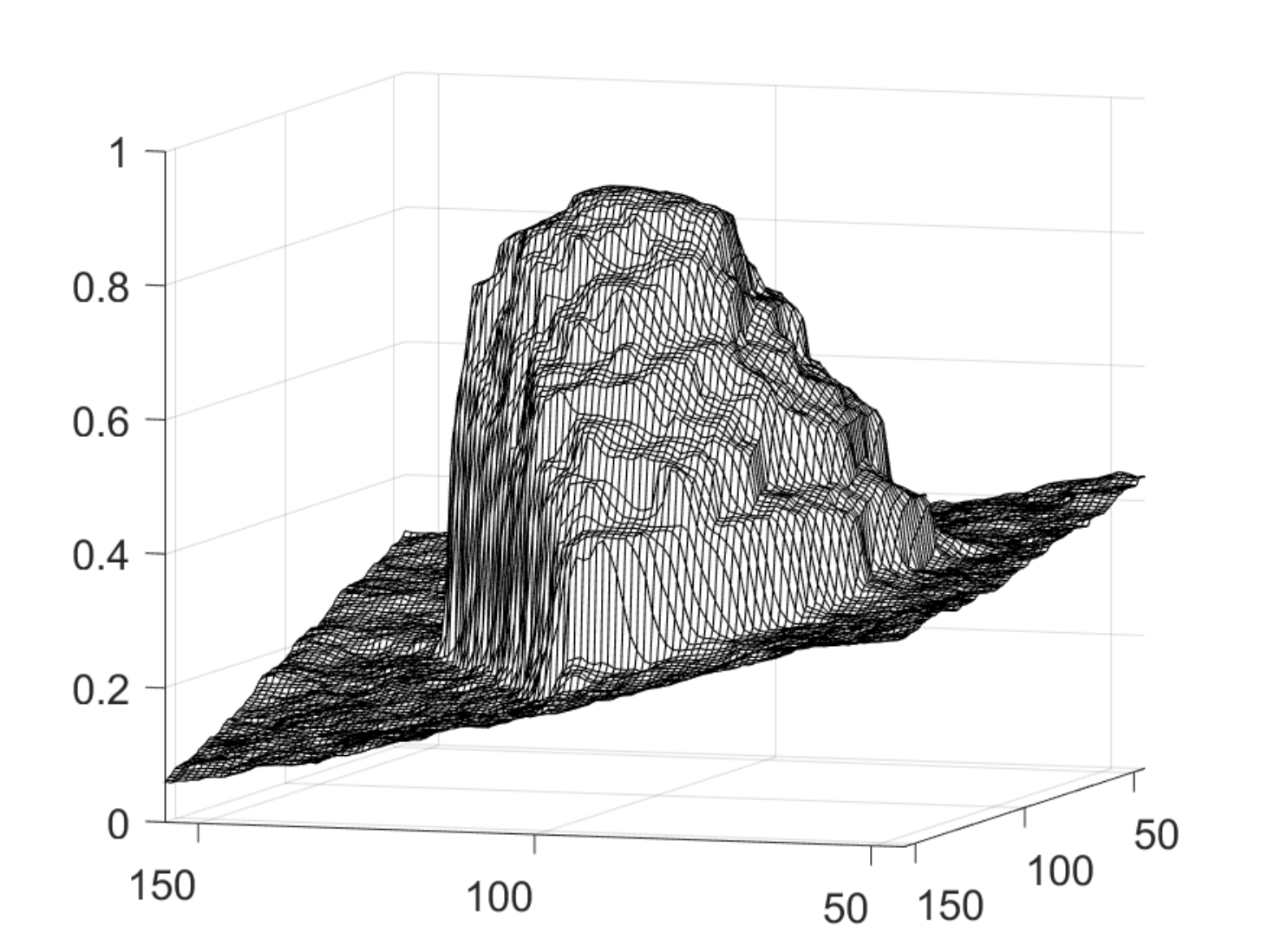}
	\end{tabular}
	\caption{(Recovering cone-shape objects.) Comparison of the proposed model with the TV model and Euler's elastica model on denoising an image whose graph contains a cone-shape object. (a) The clean image. (c) The noisy image with Gaussian noise and $\sigma=0.01$. (b) and (d) Surface plot of the central region of (a) and (c), respectively. The second and third row show the denoised images and the surface plot of their central regions by (e) the proposed model with $\alpha=0.002, \beta=40$, (f) the TV model with $\eta=0.2$, (g) Euler's elastica model with $a=b=0.15$.\vshrink
	}
	\label{fig.im.developable1}
\end{figure}
%\begin{figure}[th!]
%	\begin{tabular}{cccc}
%		(a) & (b) & (c) & (d)\\
%		\includegraphics[width=0.22\textwidth]{figures/TS2GCSharp}&
%		\includegraphics[trim={0.8cm 0 0.3cm 0},clip,width=0.22\textwidth]{figures/TS2GCSharp_surf}&
%		\includegraphics[width=0.22\textwidth]{figures/TS2GCSharp_SNP015}&
%		%\includegraphics[trim={0.8cm 0 0.3cm 0},clip,width=0.22\textwidth]{figures/TS2GCCon1_SD001_surf}
%		\\	
%	\end{tabular}
%	\begin{tabular}{ccc}
%		(e) & (f) & (g) \\
%		\includegraphics[width=0.3\textwidth]{figures/TS2GCSharp_SNP015_GC_eta100}&
%		\includegraphics[width=0.3\textwidth]{figures/TS2GCSharp_SNP015_ROF_eta04}&
%		\includegraphics[width=0.3\textwidth]{figures/TS2GCSharp_SNP015_EE_a04_b04}\\
%		\includegraphics[width=0.3\textwidth]{figures/TS2GCSharp_SNP015_GC_eta100_surf}&
%		\includegraphics[width=0.3\textwidth]{figures/TS2GCSharp_SNP015_ROF_eta04_surf}&
%		\includegraphics[width=0.3\textwidth]{figures/TS2GCSharp_SNP015_EE_a04_b04_surf}
%	\end{tabular}
%	\caption{(Developable image denoising salt and pepper noise.) Comparison of the proposed model with the TV model and Euler's elastica model on denoising a developable image. (a) The clean surface. (b) Surface plot of the central region of (a). (c) The noisy surface. (d) Surface plot of the central region of (c). The second and third row show the denoised images and the surface plot of their central region by (e) the proposed model , (f) the TV model , (g) Euler's elastica model.
%	}
%	\label{fig.im.developable}
%\end{figure}

\begin{figure}[th!]
	\begin{tabular}{ccc}
		(a) & (b) & \\
		\includegraphics[width=0.28\textwidth]{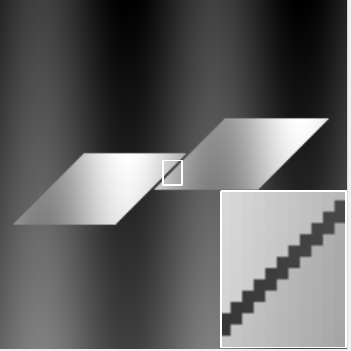}&
		\includegraphics[width=0.28\textwidth]{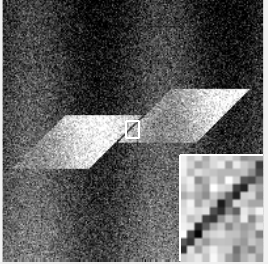}&\\
		(c) & (d) &(e)\\
		\includegraphics[width=0.28\textwidth]{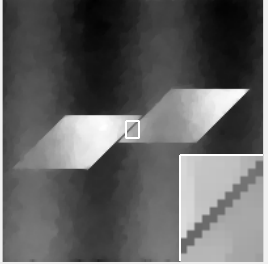}&
		\includegraphics[width=0.28\textwidth]{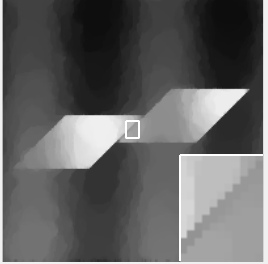}&
		\includegraphics[width=0.28\textwidth]{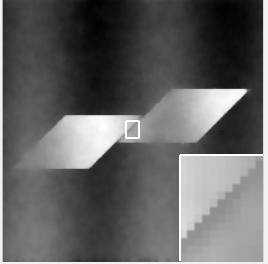}\\
		\includegraphics[width=0.28\textwidth,trim={2.5cm 1.5cm 1.5cm 3cm},clip]{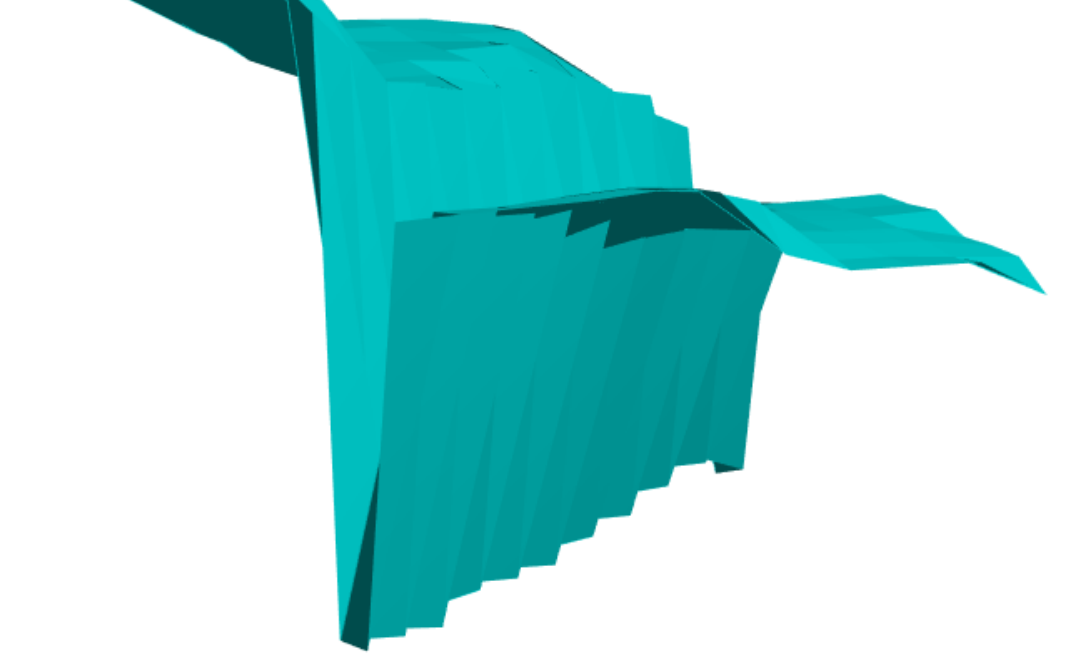}&
		\includegraphics[width=0.28\textwidth,trim={2.5cm 1.5cm 1.5cm 3cm},clip]{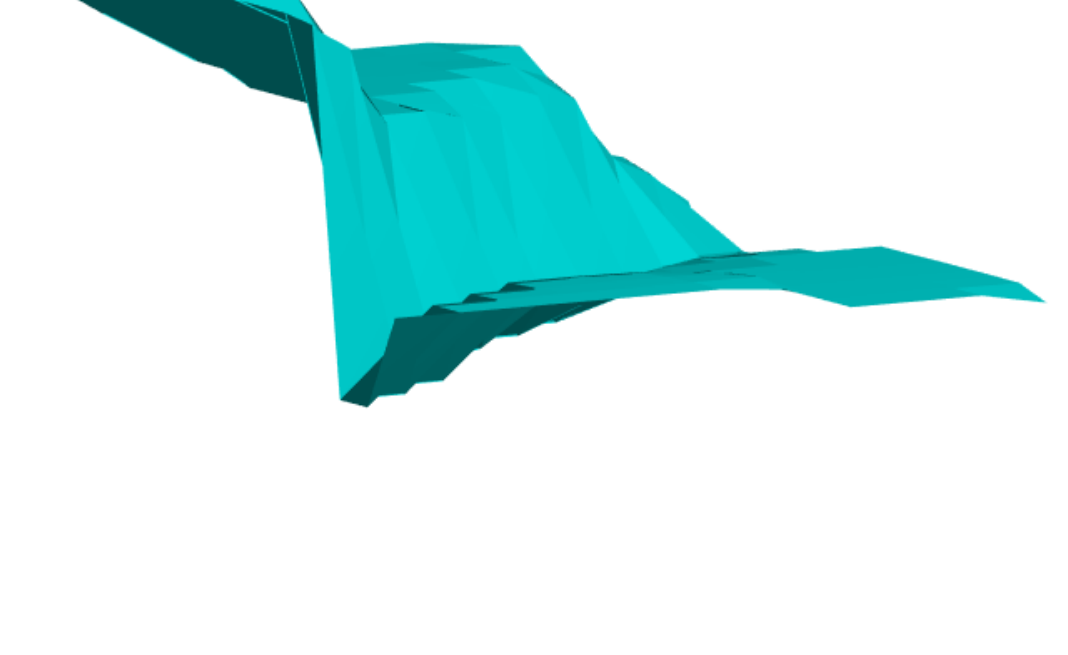}&
		\includegraphics[width=0.28\textwidth,trim={2.5cm 1.5cm 1.5cm 3cm},clip]{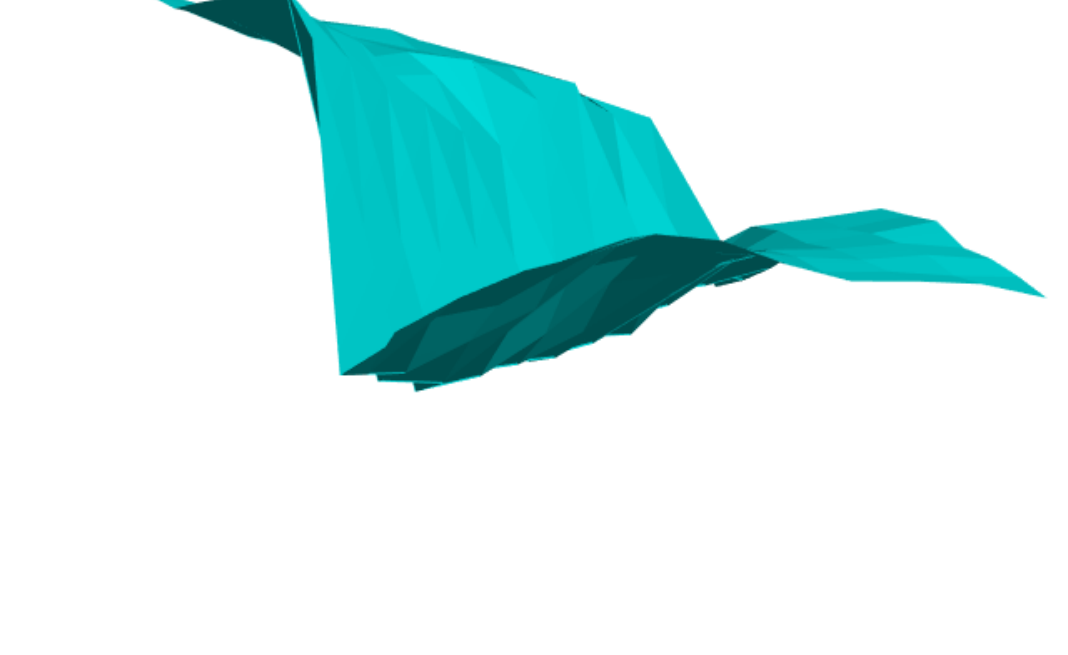}
	\end{tabular}
	\caption{(Recovering thin textures.) Comparison of the proposed model with the TV and Euler's elastica models. (a) Clean image. (b) Noisy image with Gaussian noise and $\sigma=0.015$. The second and third row show the denoised images and the surface plots of the zoomed regions by (c) the proposed model with $\alpha=0.1,\beta=1.3$, (d) the TV model with $\eta=0.25$, and (e) Euler's elastica model with $a=b=0.13$. \vshrink
	}
	\label{fig.GSP}
\end{figure}

\begin{figure}[th!]
	\begin{tabular}{ccc}
		(a) & (b) & \\
		\includegraphics[width=0.28\textwidth]{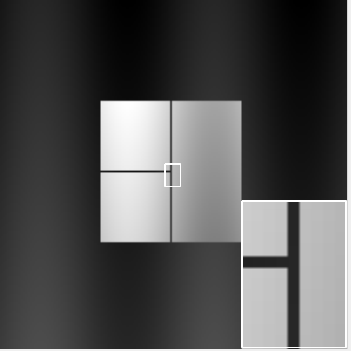}&
		\includegraphics[width=0.28\textwidth]{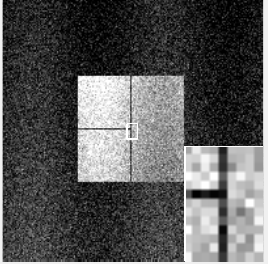}&\\
		(c) & (d) &(e)\\
		\includegraphics[width=0.28\textwidth]{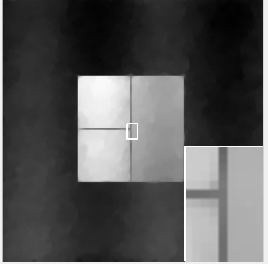}&
		\includegraphics[width=0.28\textwidth]{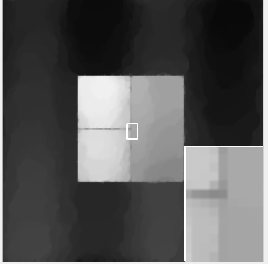}&
		\includegraphics[width=0.28\textwidth]{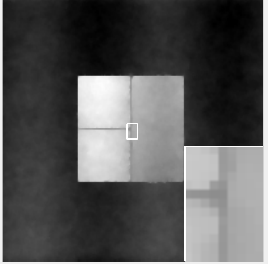}\\
		\includegraphics[width=0.28\textwidth,trim={2.5cm 1.5cm 1.5cm 2cm},clip]{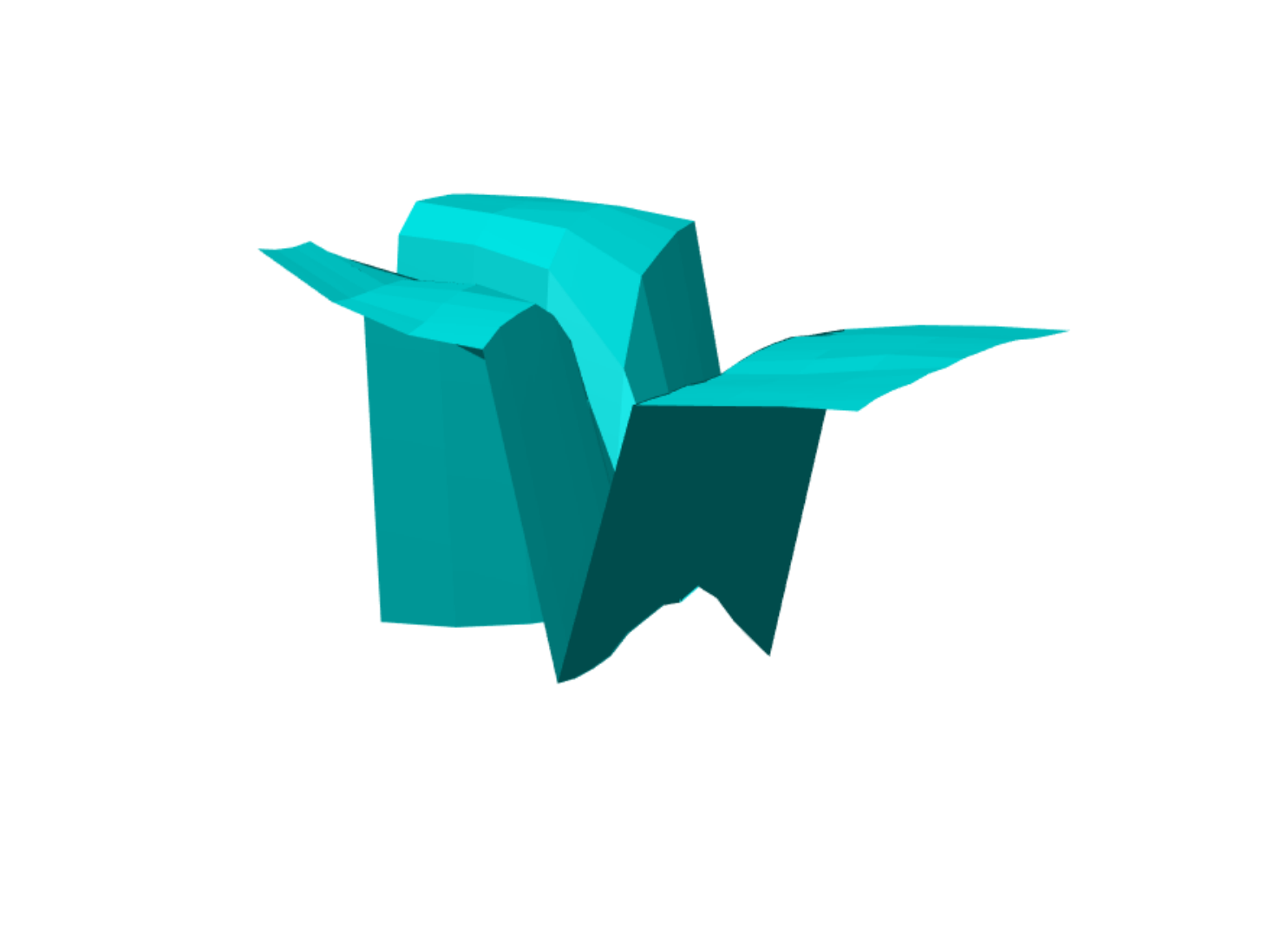}&
		\includegraphics[width=0.28\textwidth,trim={2.5cm 1.5cm 1.5cm 2cm},clip]{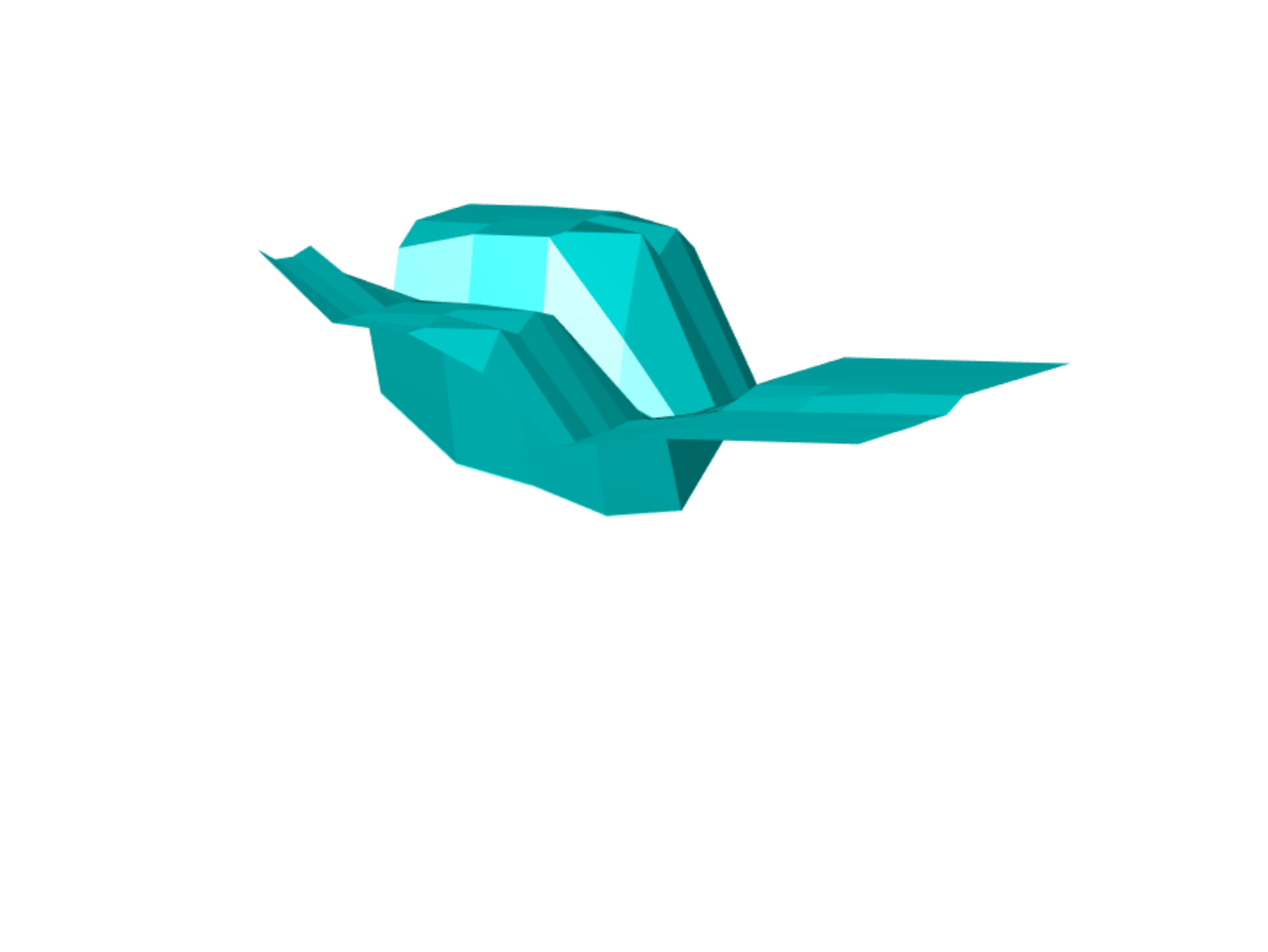}&
		\includegraphics[width=0.28\textwidth,trim={2.5cm 1.5cm 1.5cm 2cm},clip]{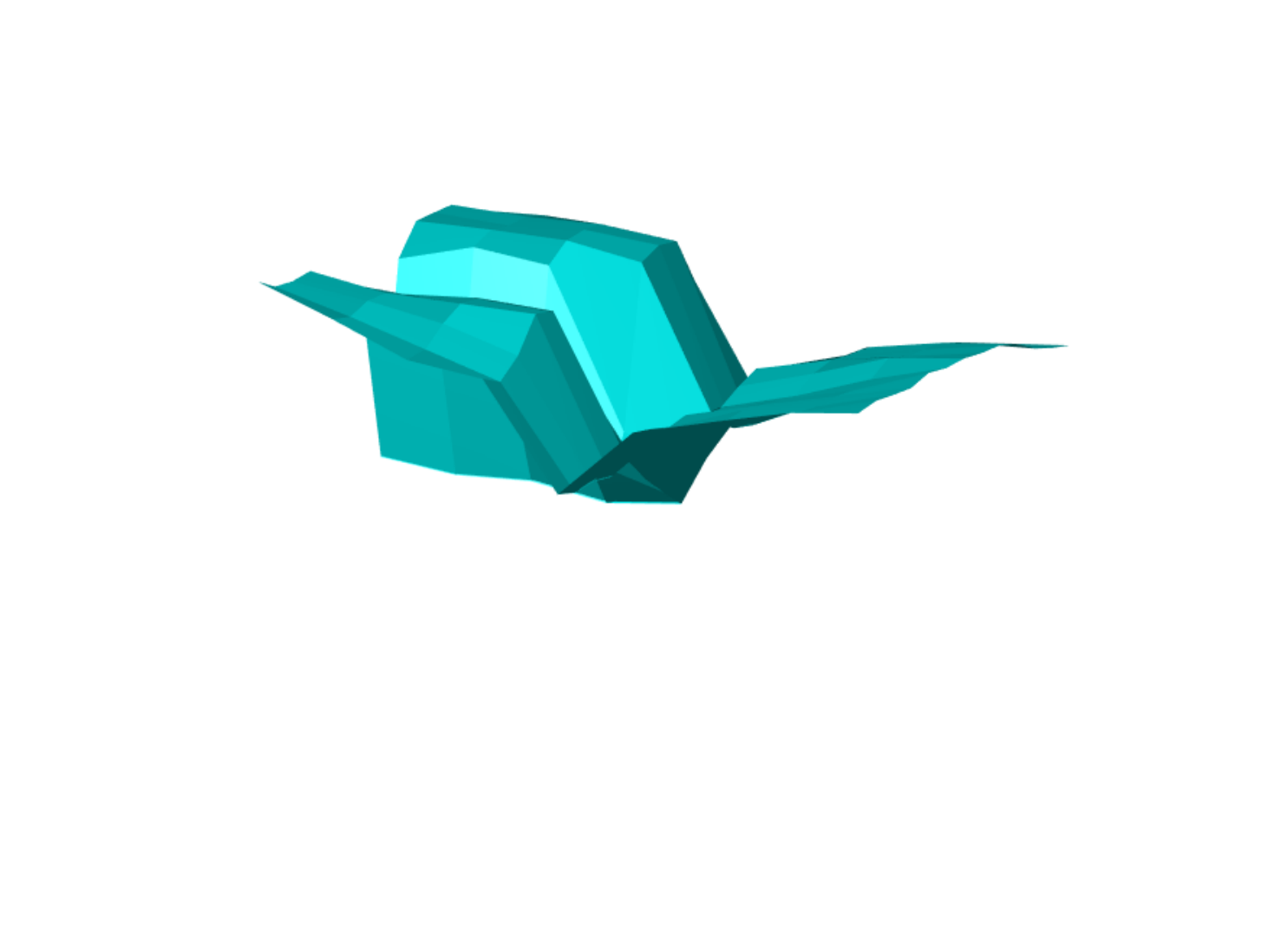}
	\end{tabular}
	\caption{(Recovering thin textures.) Comparison of the proposed model with the TV and Euler's elastica models. (a) Clean image. (b) Noisy image with Gaussian noise and $\sigma=0.015$. The second and third row show the denoised images and the surface plots of the zoomed regions by (c) the proposed model with $\alpha=0.1,\beta=1.3$, (d) the TV model with $\eta=0.25$, and (e) Euler's elastica model with $a=b=0.13$. \vshrink
	}
	\label{fig.GSSC}
\end{figure}

\begin{figure}[t!]
	\centering
	\begin{tabular}{cccc}
		(a) & (b) & (c)& (d)\\
		\includegraphics[width=0.21\textwidth]{figures/PP_SD001}&
		\includegraphics[width=0.21\textwidth]{figures/PP_SD001_GCsurf_eta06}&
		\includegraphics[width=0.21\textwidth]{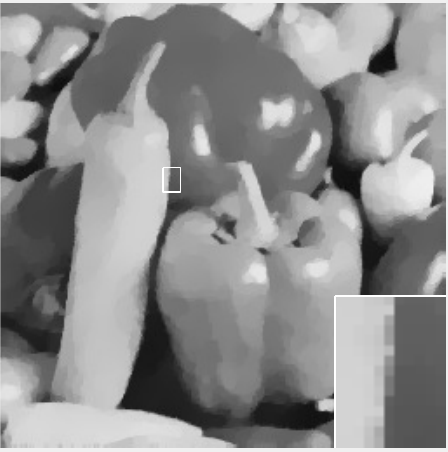}&
		\includegraphics[width=0.21\textwidth]{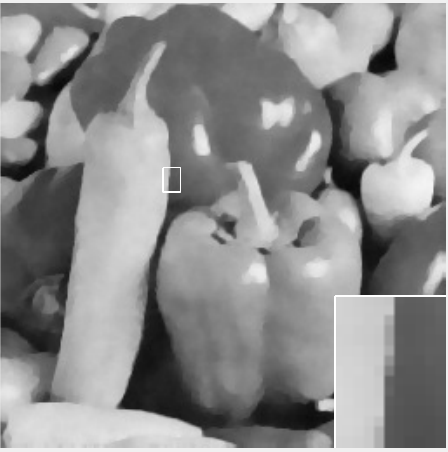}\\
		\includegraphics[width=0.21\textwidth]{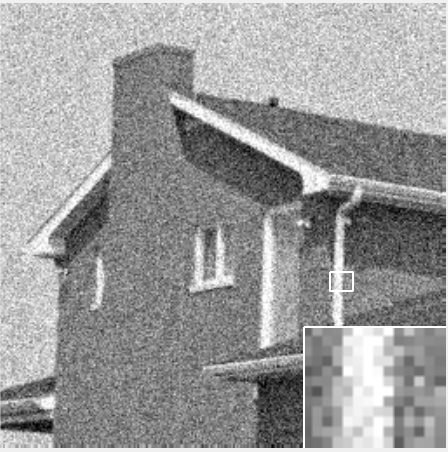}&
		\includegraphics[width=0.21\textwidth]{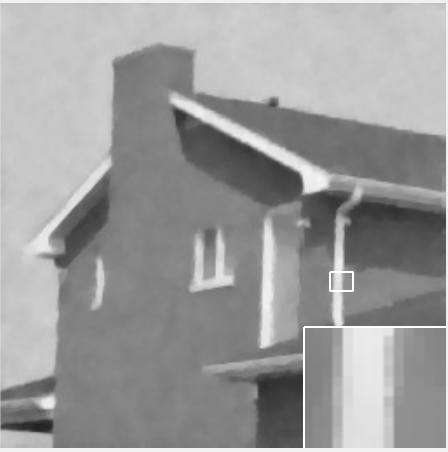}&
		\includegraphics[width=0.21\textwidth]{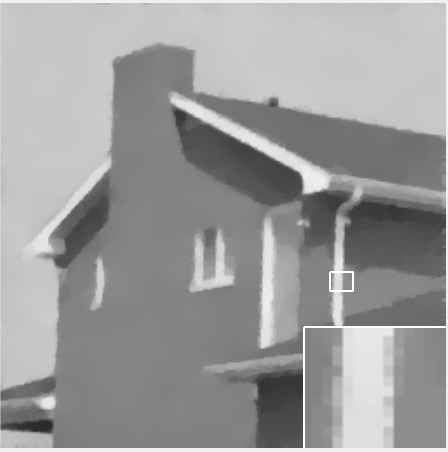} &
		\includegraphics[width=0.21\textwidth]{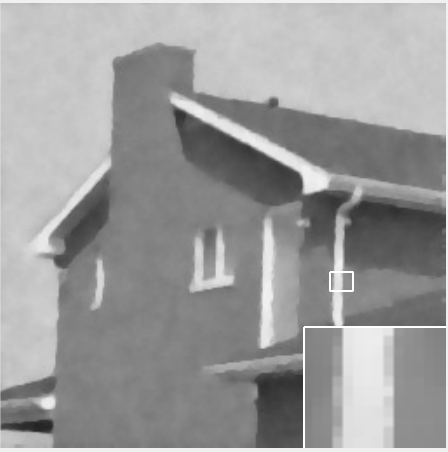}
	\end{tabular}
	\caption{(Natural image denoising) Comparison of the proposed model with the TV and Euler's elastica models. (a) Noisy images with Gaussian noise and $\sigma=0.01$. (b) Denoised images by the proposed model with $\alpha=0.2,\beta=0.6$. (c) Denoised images by the TV model with $\eta=0.15$. (d) Denoised images by Euler's elastica model with $a=0.1,b=0.1$. \vshrink}
	\label{fig.natural}
\end{figure}

\subsection{Image denoising}
We then test the proposed model on image denoising. In all of the experiments, images with pixel value varying from 0 to 1 are used. In the rest of this section, without specification, $\tau=0.05$ is used.

%In the first set of experiments, Gaussian noise with variance $\sigma=0.005$ is added to the clean images. The clean images and noisy images are shown in the first and second column of Figure \ref{fig.G0005}, respectively. The denoised images by the proposed model with $\alpha=0.2,\beta=0.4$ are shown in the third column. The proposed model smooths the noisy images while keeping sharp edges. The histories of the energy and the relative error $\|u^{n+1}-u^n\|_{2}/\|u^{n+1}\|_{2}$ in these examples are shown in Figure \ref{fig.G0005.E}. For all examples, the energy achieves its minimum within 200 iterations. Sublinear convergence is observed for the relative error.

In the first set of experiments, Gaussian noise with variance $\sigma=0.01$ is added to the clean images. The clean images and noisy images are shown in the first and second column of Figure \ref{fig.G001}, respectively. The denoised images by the proposed model with $\alpha=0.2,\beta=0.6$ are shown in the third column. The proposed model smooths the noisy images while keeping sharp edges. The histories of the energy and the relative error $\|u^{n+1}-u^n\|_{2}/\|u^{n+1}\|_{2}$ of these examples are shown in Figure \ref{fig.G001.E}. For both examples, the energy achieves its minimum within 200 iterations. Sublinear convergence is observed for the relative error.

We then take the image in the second row of Figure \ref{fig.G001} as an example and compare the efficiency of Newton's method (\ref{eq.p1.newton}) and the fixed point method (\ref{eq.p1.fix1})--(\ref{eq.p1.fix3}) when computing $\bp^{n+1/4}$. We set $\rho=1$ in Newton's method and $\rho_1=0.8$ in the fixed point method. For various time steps, we present the average number of iteration used in Newton's method and the fixed point method per outer iteration in Table \ref{tab.newtonfix} Column 2-5. As we expected, smaller time step makes $\bp^n$ a better initial guess of $\bp^{n+1/4}$ so that less iterations are needed for both subiterations to converge. Since the computation complexity in the fixed point method is lower than that in Newton's method, each iteration of the fixed point method uses less CPU time than that of Newton's method, as shown in Table \ref{tab.newtonfix} Column 6.

We next study the computational cost of the proposed algorithm with respect to the dimension of images. We use the image in the second row of Figure \ref{fig.G001} as an example. In this test, we generate clean images with size $p\times p$ for $p=50,100,150,200,250,300$. Then  Gaussian noise with $\sigma=0.01$ is added to these images. In our experiments, we set $\alpha=0.2,\beta=0.6$. The number of iterations and CPU time used to satisfy the stopping criterion is summarized in Table \ref{tab.size}. In this experiment, the total number of iteration is not sensitive to the image size: all experiments used about 500 iterations to satisfy the stopping criterion. The total CPU time scales quadraticly with the image size. Consider that the total dimension of an image with size $p\times p$ is $p^2$, the computational cost of the proposed algorithm grows linear with the total dimension of the image. To demonstrate the efficiency of subiterations (\ref{eq.p1.fix1})--(\ref{eq.p1.fix3}) and (\ref{eq.alterH.1})--(\ref{eq.alterH.end}), we present the CPU time used by each subiteration in Column 5 and 6, respectively. In general, the sum of the CPU time used by both subiterations take up no more than $60\%$ of the total CPU time.

We then compare the proposed model with the TV and Euler's elastica model on denoising images whose graph contains cone-shape objects and are piecewise developable. For the first example, the clean image is shown in Figure \ref{fig.im.developable}(a). The graph of the central region of the image is shown in (b). The noisy image is generated by adding Gaussian noise with $\sigma=0.01$, which is shown in Figure \ref{fig.im.developable}(c) and (d). By the proposed model, the TV model and Euler's elastica model, the denoised images are shown in (e)-(g), respectively. Since the graph of the clean image is developable, we use $\alpha=2\times10^{-3},\beta=40$ in the proposed model such that the functional is dominated by the Gaussian curvature term. We use $\eta=0.2$ in the TV model and $a=b=0.15$ in Euler's elastica model. To better compare the details, the surface plot of the central region of each denoised image is presented under it. In the result of the TV model, staircase effects are observed and the peak is flattened. Compared to the TV model, Euler's elastica model has a stronger smoothing effect, while whose result has some oscillations in the denoised cone. The proposed model gives the best results which recovers a smooth surface of the cone while preserving the peak. Our second example is shown in Figure \ref{fig.im.developable1}, in which the noisy image contains heavy Gaussian noise with $\sigma=0.015$. We use $\alpha=2\times10^{-3},\beta=40$  in the proposed model, $\eta=0.2$ in the TV model and $a=b=0.15$ in Euler's elastica model. In the denoised images, staircase effects and patterned artifacts are observed in the results of the TV and Euler's elastica model. The proposed model provides smooth recovery of the central sphericon together with a better recovery of the peak.

We next demonstrate the advantage of the proposed model on preserving thin textures. We consider clean images as shown in Figure \ref{fig.GSP}(a) and Figure \ref{fig.GSSC}(a). Noisy images are generated by adding Gaussian noise with $\sigma=0.01$ in Figure \ref{fig.GSP}(b) and $\sigma=0.015$ in Figure \ref{fig.GSSC}(b). The denoised images (and the surface plot of the zoomed regions) by the proposed model, the TV model and Euler's elastica model are shown in (c)-(e) in both figures, respectively. In the results by the TV model and Euler's elastica model, the gaps are smoothed a lot. The proposed model provides the best results which preserve the thin gaps well.

We then compare these three models on two natural images: 'Peppers' and 'House'. Gaussian noise with $\sigma=0.01$ is added to these clean images. The noisy images and denoised images by the three models are shown in Figure \ref{fig.natural}. We use $\alpha=0.2,\beta=0.6$ in the proposed model, $\eta=0.15$ in the TV model and $a=b=0.1$ in Euler's elastica model.
These results are comparable while there are some oscillations around edges in the results by the TV model.
The comparison of the PSNR and SSIM \cite{wang2004image} values of all images in Figure \ref{fig.natural} are shown in Table \ref{tab.natural}. The proposed model provides results with the largest PSNR and SSIM values. To compare the efficiency, in Table \ref{tab.cpu}, we show the number of iterations and CPU time used to get results in Figure \ref{fig.natural}. Since the TV model is the simplest model, results by it have the least CPU time. Compared to the algorithm of Euler's elastica model in \cite{deng2019new}, the proposed algorithm needs approximately half of its number of iterations to meet the stopping criterion. Note that the proposed model is more complicated than Euler's elastica model due to the determination of the Hessian matrix. The proposed algorithm needs more time at each iteration. The overall CPU time of the proposed algorithm is comparable to that of the algorithm proposed in \cite{deng2019new}.

\begin{table}[t!]
	\centering
	(a)\\
	\begin{tabular}{c|c|c|c|c}
		\hline
		& Noisy &Proposed model & TV & Euler's elastica \\
		\hline
		%Cameraman & 20.02 & {\bf 25.76} & 25.72 & 25.55\\ \hline
		Peppers & 19.99 & {\bf 27.30} & 26.70 & 27.27\\
		\hline
		%		Lena & 19.99 & {\bf 27.16} & 26.57 & 27.09\\
		%		\hline
		House & 19.99 & {\bf 28.91} & 28.37 & 27.78\\
		\hline
		%		Plane & 20.01 & {\bf 27.76} & 27.20 & 27.19\\
		%		\hline
	\end{tabular}
	\vspace{0.2cm}\\
	(b)\\
	\begin{tabular}{c|c|c|c|c}
		\hline
		& Noisy &Proposed model & TV & Euler's elastica \\
		\hline
		%		Cameraman & 0.3286 & {\bf 0.7784} & 0.7754 & 0.7720\\
		%		\hline
		Peppers & 0.3763 & {\bf 0.8402} & 0.8198 & 0.8363\\
		\hline
		%		Lena & 0.3509 & {\bf 0.8089} & 0.7881 & 0.8032\\
		%		\hline
		House & 0.2876 & {\bf 0.8146} & 0.8059 & 0.8097\\
		\hline
		%		Plane & 0.2979 & {\bf 0.8375} & 0.8287 & 0.8250\\
		%		\hline
	\end{tabular}
	\caption{\label{tab.natural}(Gaussian noise with $\sigma=0.01$.) Comparison of (a) the  PSNR and (b) the SSIM value of images in Figure \ref{fig.natural}. The largest value for each image is marked in bold. \vshrink}
	
\end{table}

\begin{table}[t!]
	\centering
	\begin{tabular}{c|c|c|c}
		\hline
		&Proposed model & TV & Euler's elastica \\
		\hline
		%Cameraman $(256\times 256)$& 691 (45.19) & 567 (5.78)& 1222 (48.03)\\ \hline
		Peppers $(256\times 256)$ & 641 (44.39) & 771 (6.25) & 1395 (52.92)\\
		\hline
		%		Lena  $(256\times 256)$& 578 (39.97) & 820 (6.79) & 1459 (55.63)\\
		%		\hline
		House $(256\times 256)$ & 556 (38.99) & 702 (5.86) & 1028 (39.98)\\
		\hline
		%		Plane $(512\times 512)$ & 532 (226.76) & 462 (22.10) &  849 (221.93)\\
		%		\hline
	\end{tabular}
	\caption{\label{tab.cpu} (Gaussian noise with $\sigma=0.01$.) Comparison of the number of iterations (CPU time in seconds) used to get results in Figure \ref{fig.natural}. \vshrink}
	
\end{table}

\begin{figure}[t!]
	\centering
	\begin{tabular}{cccc}
		(a) & (b) & (c)& (d)\\
		\includegraphics[width=0.21\textwidth]{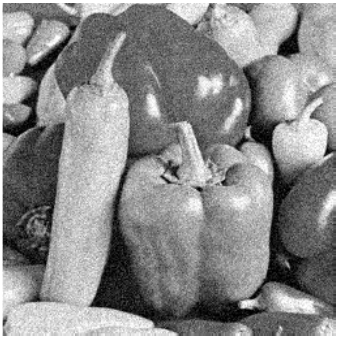}&
		\includegraphics[width=0.21\textwidth]{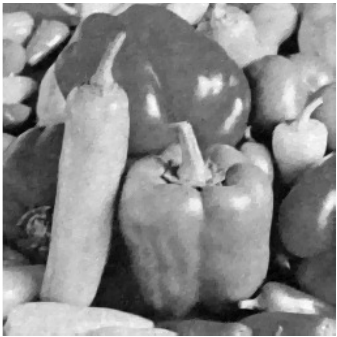}&
		\includegraphics[width=0.21\textwidth]{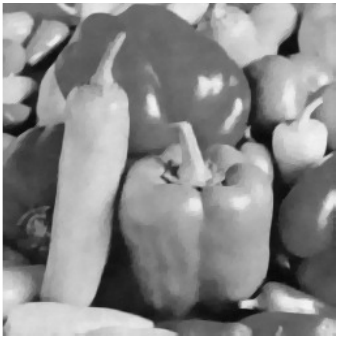} &
		\includegraphics[width=0.21\textwidth]{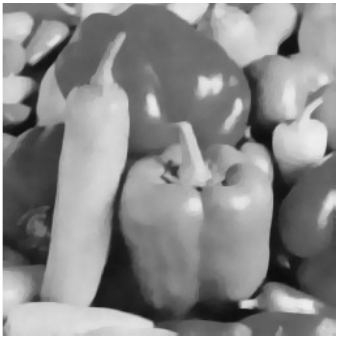}
	\end{tabular}
	\caption{(Effect of $\beta$.) (a) Noisy image with Gaussian noise and $\sigma=0.005$. (b) Denoised image with $\beta=0.2$. (c) Denoised image with $\beta=0.4$. (d) Denoised image with $\beta=0.6$. We fix $\alpha=0.2$. \vshrink}
	\label{fig.beta}
\end{figure}

\begin{figure}[t!]
	\begin{tabular}{cccc}
		(a) & (b) & (c)& (d)\\
		\includegraphics[width=0.21\textwidth]{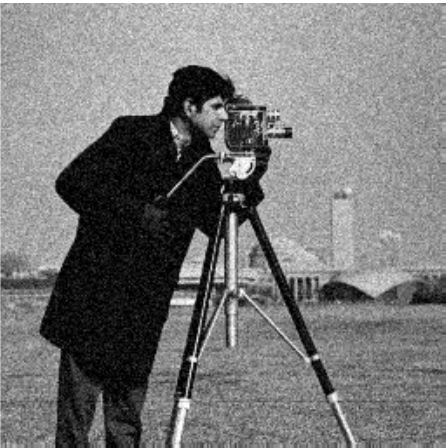}&
		\includegraphics[width=0.21\textwidth]{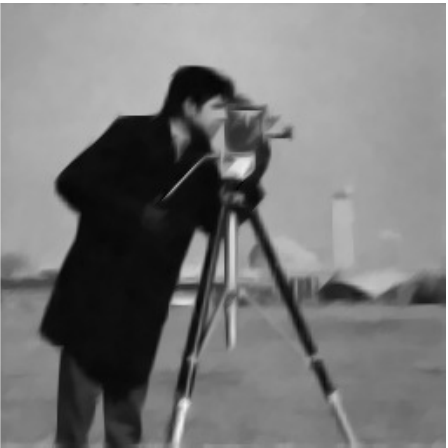}&
		\includegraphics[width=0.21\textwidth]{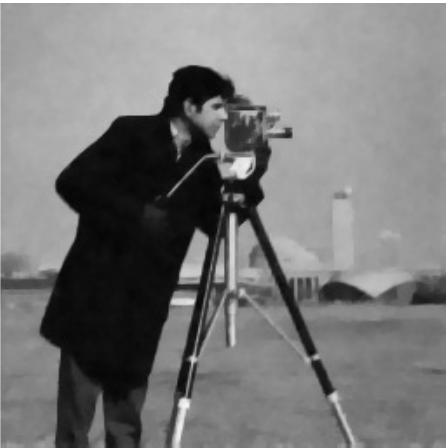} &
		\includegraphics[width=0.21\textwidth]{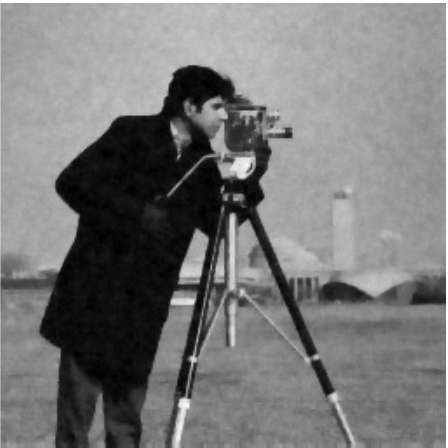}
	\end{tabular}
	\caption{(Effect of $\alpha$.) (a) Noisy image with Gaussian noise and $\sigma=0.005$. (b) Denoised image with $\alpha=0.01$. (c) Denoised image with $\alpha=0.2$. (d) Denoised image with $\alpha=0.8$. The Gaussian curvature term has a larger weight with a smaller $\alpha$. We fix $\alpha\beta=0.08$. \vshrink}
	\label{fig.alpha}
\end{figure}

\subsection{Effects of parameters}
We explore the effects of the parameters in the proposed model (\ref{eq.model}). In (\ref{eq.model}), $\beta$ controls the weight of the fidelity term. We expect larger $\beta$ makes the result smoother, which is verified by the following experiment. We add Gaussian noise with $\sigma=0.005$ to the clean image and fix $\alpha=0.2$. The noisy image and denoised images with $\beta=0.2,0.4$ and $0.6$ are shown in Figure \ref{fig.beta}.

A more interesting study is the effects of $\alpha$, which balances the weight between the first order term (the TV term) and the second order term (Gaussian curvature term). In this experiment, the clean image is perturbed by Gaussian noise with $\sigma=0.005$. We test $\alpha$ among $0.01,0.2$ and $0.8$. If we fix $\beta=0.04$ and when $\alpha$ is too small (like 0.005), the regularization is not enough. To resolve this problem, we fix $\alpha\eta=0.008$. Under this setting, increasing $\alpha$ amounts to decreasing the weight of the Gaussian curvature term. The noisy and denoised images are shown in Figure \ref{fig.alpha}.
When $\alpha$ is too large (like 0.8), the regularization is dominated by the TV term. The regularization effect is not enough under this choice, as shown in Figure \ref{fig.alpha}(d). As we decrease $\alpha$, i.e., the weight of the Gaussian curvature term increases, the denoised image has a stronger smoothing effect while edges are kept well, as shown in Figure \ref{fig.alpha}(b) and (c). This experiment shows that Gaussian curvature smooths the flat region of an image while keeping sharp edges.
%When $\alpha$ is small, i.e., the regularization is dominated by the Gaussian curvature term, is the most interesting one. When $\alpha=0.05$, as shown in Figure \ref{fig.alpha}(a), there are some artifacts in the denoised image. It seems that the image is perturbed in a different way. Note that the Gaussian curvature is the product of principal curvatures. As long as one principal curvature is 0, the product is 0. Thus for any consecutive pixel in a line with the same value, the Gaussian curvature of the middle pixel is zero, no matter how different its value is from that of its other neighbors. This explains our observation of Figure \ref{fig.alpha}(a).

\section{Conclusion}\label{sec.conclusion}
We propose an efficient operator-splitting method to optimize a general Gaussian curvature model. The optimization problem is associated with an initial-value problem whose steady state solution solves the optimization problem. Such an initial-value problem is time-discretized by the operator-splitting method. In our splitting scheme, each sub-problem has either a closed-form solution or can be solved efficiently. The efficiency and performance of the proposed method is demonstrated on systematic numerical experiments on surface smoothing and image denoising. The proposed model has excellent performance in smoothing developable surfaces and images, and has advantages in recovering thin textures of images.%Our numerical experiments show that such the Gaussian curvature model smooths the flat region of an image while keeping sharp edges and thin textures.

\section*{Acknowledgement}
The authors would like to sincerely thank Prof. Ron Kimmel at Technion for invaluable discussions on geometric regularizers.

\bibliographystyle{abbrv}
\bibliography{ref}
\end{document}